%% file: ArXiv_Changes_main.tex
\documentclass[10pt]{article} 
\usepackage[preprint]{tmlr}

\input{math_commands.tex}

\usepackage{hyperref}
\usepackage{url}
\usepackage{amsthm}
\newcommand{\rmd}{\mathrm{d}}
\newtheorem{thm}{Theorem}
\newtheorem{assumption}{Assumption}
\newtheorem{remark}[thm]{Remark}
\newtheorem{prop}[thm]{Proposition}
\newtheorem{lem}[thm]{Lemma}
\newtheorem{cor}[thm]{Corollary}
\newcommand{\N}{\mathbb{N}}
\usepackage{caption}
\usepackage{float}
\usepackage{booktabs}
\usepackage{subcaption}
\usepackage{graphicx}
\usepackage[titletoc,toc,title]{appendix}
\newtheorem{intassumption}{Assumption}
\numberwithin{intassumption}{assumption}

\title{On diffusion-based generative models and their error bounds: The  log-concave case with full convergence estimates}

\author{\name Stefano Bruno \email sbruno@ed.ac.uk \\
      \addr School of Mathematics \\
      University of Edinburgh
      \AND
      \name Ying Zhang \email yingzhang@hkust-gz.edu.cn \\
      \addr Fintech Thrust \\ The Hong Kong University of Science and Technology
      (Guangzhou)
      \AND
      \name Dong-Young Lim \email dlim@unist.ac.kr \\
      \addr Department of Industrial Engineering \\ Artificial Intelligence Graduate School \\  Ulsan National Institute of Science and Technology
  \AND \name {\"O}mer Deniz Akyildiz \email deniz.akyildiz@imperial.ac.uk \\
  \addr  Department of Mathematics \\ Imperial College London 
  \AND 
  \name Sotirios Sabanis \email   s.sabanis@ed.ac.uk \\
  \addr  School of Mathematics \\
  University of Edinburgh \\ The Alan Turing Institute \\ National Technical University of Athens
}

\def\month{02}  
\def\year{2025} 
\def\openreview{\url{https://openreview.net/forum?id=zjxKrb4ehr}}

\begin{document}

\maketitle

\begin{abstract}
\noindent We provide full theoretical guarantees for the convergence behaviour of diffusion-based generative models under the assumption of strongly log-concave data distributions while our approximating class of functions used for score estimation is made of Lipschitz continuous functions avoiding any Lipschitzness assumption on the score function. We demonstrate via a motivating example, sampling from a Gaussian distribution with unknown mean, the powerfulness of our approach. In this case, explicit estimates are provided for the associated optimization problem, i.e. score approximation, while these are combined with the corresponding sampling estimates. As a result,  we obtain the best known upper bound estimates in terms of key quantities of interest, such as the dimension and rates of convergence, for the Wasserstein-2 distance between the data distribution (Gaussian with unknown mean)  and our sampling algorithm.
Beyond the motivating example and in order to allow for the use of a diverse range of stochastic optimizers, we present our results using an $L^2$-accurate score estimation assumption, which crucially is formed under an expectation with respect to the stochastic optimizer and our novel auxiliary process that uses only known information. This approach yields the best known convergence rate for our sampling algorithm.
\end{abstract}

\section{Introduction} \label{Introduction}

\noindent Diffusion-based generative models, also known as score-based generative models (SGMs) \citep{NEURIPS2019_3001ef25,song2021scorebased,sohl2015deep,ho2020denoising}, have become over the past few years one of the most popular approaches in generative modelling due to their empirical successes  in data generation tasks. These models have achieved state-of-the-art results in image generation \citep{dhariwal2021diffusion, rombach2022high}, audio generation \citep{kong2020diffwave} and inverse problems \citep{chung2022diffusion, song2022pseudoinverse, cardoso2023monte,boys2023tweedie} outperforming other generative models like generative adversarial networks (GANs) \citep{goodfellow2014generative}, variational autoencoders (VAEs) \citep{kingma2014auto}, normalizing flows \citep{rezende2015variational} and energy-based methods \citep{zhao2017energybased}.

SGMs generate approximate data samples from high-dimensional data distributions by combining two diffusion processes, a forward and a backward in time process. The former process  is used to iteratively and smoothly transform samples from the unknown data distribution into (Gaussian) noise, while the associated backward in time process reverses the noising procedure and generates new samples from the starting unknown data distribution. A key role in these models is played by the score function, i.e. the gradient of the log-density of the solution of the forward process, which appears in the drift of the stochastic differential equation (SDE) associated to the backward process. Since this quantity depends on the unknown data distribution, an estimator of the score has to be constructed during the noising step using score-matching techniques \citep{hyvarinen2005estimation,vincent2011connection}. These techniques have the advantage of not suffering from known problems of traditional pushforward generative models, such as mode collapse \citep{salmona2022can}.

The widespread applicability and success of SGMs have been accompanied by a growing interest in the theoretical understandings of these models, particularly in their convergence theory e.g. in \citet{block2020generative, de2021diffusion, debortoli2022convergence, lee2022convergence,yang2022convergence, kwon2022score, liu2022let, oko2023diffusion, lee2023convergence,chen2023improved,chen2023sampling, li2023towards,  pedrotti2023improved, conforti2023score, benton2023linear}, with further works appearing after the first version of our preprint e.g. in \citet{tang2024contractive, strasman2025an, mimikos-stamatopoulos2024scorebased}.  At its core, this new generative modeling approach combines optimisation and sampling procedures -- specifically, the approximation of the score and the creation of new samples, which make its theoretical analysis both an interesting and a rich challenge.


Some of the recent advances in  the study of the theoretical properties of SGMs concentrate around the sampling procedure by assuming suitable control for the score estimation procedure. For instance, the analysis in \citet{lee2022convergence,chen2023sampling} assumes that the score estimate is $L^2$-\textit{accurate}, meaning that the $L^2$ error between the score and its estimate is small, and provides estimates in total variation (TV) distance. Under the same assumption, the more recent contribution in \citet{conforti2023score} establishes non-asymptotic bounds in Kullback Leibler (KL) divergence by assuming finite relative Fisher of data distribution with respect to a Gaussian distribution.

The main drawback of the aforementioned $L^2$-accurate (and in some cases $L^{\infty}$-accurate) score estimation assumption is that the corresponding expectation is given with respect to density of the solution of the forward process, which depends on the unknown data distribution.

Our approach introduces a novel auxiliary process that relies solely on known information and uses the density of the solution of this process in the $L^{2}$-accurate score estimation.
To further highlight the powerfulness of our approach, we present a motivating example on the case of sampling from a Gaussian distribution with unknown mean. Full theoretical estimates for the convergence properties of the SGM are provided in Wasserstein-2 distance, while our choice of stochastic optimizer for the score approximation is a simple Langevin-based algorithm. Our estimates are explicit and contain the best known optimal dependencies in terms of dimension and rate of convergence (Theorem \ref{main_theorem_toy_example}). To the best of the authors' knowledge, these are the first such explicit results with transparent dependence on the parameters involved in the sampling and optimization combined procedures of the diffusion models. By connecting the diffusion models with the theoretical guarantees of machine learning optimizers via standard stochastic calculus tools, the results in Theorem \ref{main_theorem_toy_example}, together with the bounds achieved in the more general setting in Theorem \ref{main_theorem_general}, provide the theoretical justification for the empirical success of the diffusion models. 

We close this introductory section by highlighting some other, alternative approaches which were recently developed. One may consult \citet{yingxi2022convergence} for an approach based on the assumption that the score estimation error has a sub-Gaussian tail. This is a stronger assumption than $L^2$-\textit{accurate}. In \citet{chen23score}, non-asymptotic bounds in TV and Wassertein distance of order 2 are derived when the data distribution is supported on a low-dimensional linear subspace. Finally, convergence guarantees in TV were developed in \citet{chen2023probability} for the probability flow ordinary differential equation (ODE) implementation of SGMs under the $L^2$-accurate score estimate assumption.

\textit{Notation.} Let $(\Omega, \mathcal{F}, \mathbb{P})$  be a fixed
probability space. We denote by $\mathbb{E}[X]$ the expectation of a random variable $X$. For $ 1 \le p < \infty$, $L^p$ is used to denote the usual space of $p$-integrable real-valued random
variables. The $L^p$-integrability of a random variable $X$ is defined as   $\mathbb{E}[|X|^p] < \infty$. Fix an integer $M \ge 1$. For an $\mathbb{R}^{M }$-valued random variable $X$, its law on $\mathcal{B}(\mathbb{R}^{M} )$, i.e. the
Borel sigma-algebra of $\mathbb{R}^{M} $  is denoted by $\mathcal{L}(X)$.   Let $T>0$ denotes some time horizon. For a positive real number $a$, we denote its integer part by $\lfloor a \rfloor$. The Euclidean scalar product is denoted by $\langle \cdot, \cdot \rangle $, with $| \cdot |$ standing for the corresponding norm (where the dimension of the space may vary depending on the context). Let $\mathbb{R}_{>0}:=\{ x \in \mathbb{R} | \  x >0\}$. Let $f: \mathbb{R}^{M } \rightarrow \mathbb{R}$ be a continuously differentiable function. The gradient of $f$ is denote by $\nabla f$. For any integer $q \ge 1$, let $\mathcal{P}(\mathbb{R}^q)$ be the set of probability measures on $\mathcal{B}(\mathbb{R}^q)$. For $\mu$, $\nu \in \mathcal{P}(\mathbb{R}^M)$, let $\mathcal{C}(\mu, \nu)$ denote the set of probability measures $\zeta$ on $\mathcal{B}(\mathbb{R}^{2M })$ such that its respective marginals are $\mu$ and $\nu$. For any $\mu$ and $\nu \in \mathcal{P}(\mathbb{R}^{M })$, the Wasserstein distance of order 2 is defined as
\begin{equation*}
	W_2(\mu, \nu)= \left(  \inf_{\zeta \in \mathcal{C}(\mu, \nu)} \int_{\mathbb{R}^M}  \int_{\mathbb{R}^M}  |x - y|^2 \ \text{d} \zeta(x,y) \right)^{\frac{1}{2}}.
\end{equation*}

	\section{Technical Background} \label{Introduction_score_based_generative_models_background}
\noindent In this section, we provide a brief summary behind the construction of score-based generative models (SGMs) based on diffusion introduced in \citet{song2021scorebased}. The fundamental concept behind SGMs centers around the use of an ergodic (forward) diffusion process to diffuse the unknown data distribution $\pi_{\mathsf{D}} \in \mathcal{P}(\mathbb{R}^{M})$ to a known prior distribution and then learn a backward process to transform the prior to the target distribution $\pi_{\mathsf{D}}$ by estimating the score function of the forward process. In our analysis, we focus on the forward process $(X_t)_{t \in [0,T]}$  given by  an Ornstein-Uhlenbeck (OU) process
\begin{equation} \label{OU_process_introduction}
	\text{d} X_t = -  X_t \ \text{d}t + \sqrt{2  } \ \text{d}B_t, \quad X_0 \sim \pi_{\mathsf{D}},
\end{equation}
where $(B_t)_{t \in [0,T]}$ is an $M$-dimensional Brownian motion and we assume that  $\mathbb{E}[|X_0|^2] < \infty$. The process \ref{OU_process_introduction} is chosen to match the forward process in the original paper \citep{song2021scorebased}, which is also referred to as Variance Preserving Stochastic Differential Equation. The noising process \ref{OU_process_introduction} can also be represented as follows
\begin{equation}\label{eq:OUdistribtuion}
	X_t\overset{\text{a.s.}}{=} m_t X_0+\sigma_t Z_t, \quad m_t = e^{-  t}, \quad \sigma_t^2 = 1-e^{-2  t}, \quad Z_t \sim \mathcal{N}(0,I_{M}),
\end{equation}
where $\overset{\text{a.s.}}{=}$ denotes almost sure equality and $I_{M}$ denotes the identity matrix of dimension $M$. Under mild assumptions on the target data distribution $\pi_{\mathsf{D}}$ \citep{haussmann1986time,cattiaux2023time}, the backward process $(Y_t)_{t \in [0,T]} = (X_{T-t})_{t \in [0,T]}$ is given by
\begin{equation} \label{eq:backwardproc_real_initial_condition_introduction}
	\text{d} Y_t  =  (Y_t +2  \nabla \log p_{T-t}(Y_t)) \ \text{d} t +\sqrt{2  } \ \text{d} \bar{B}_t, \quad Y_0 \sim \mathcal{L}(X_T),
\end{equation}
where $\{p_t \}_{t \in [0,T]}$ is the family of densities of $\{ \mathcal{L}(X_t)  \}_{t \in (0,T]}$ with respect to the Lebesgue measure, $\bar{B}_t$ is an another Brownian motion independent of $B_t$ in \ref{OU_process_introduction} defined on $(\Omega, \mathcal{F}, \mathbb{P})$.
However, the sampling is done in practice from the invariant distribution of the forward process, which, in this case, is a standard Gaussian distribution. Therefore, the backward process  \ref{eq:backwardproc_real_initial_condition_introduction} becomes
\begin{equation} \label{Y_hat_auxiliary}
	\text{d} \widetilde{Y}_t =  (\widetilde{Y}_t + 2 \  \nabla \log p_{T-t} ( \widetilde{Y}_t ) ) \ \text{d} t + \sqrt{2  } \ \text{d} \bar{B}_t, \quad \widetilde{Y}_0 \sim \pi_{\infty} = \mathcal{N}(0,I_{M}).
\end{equation}
Here, we have $ \widetilde{Y}_0  \overset{\text{a.s.}}{=} Z_T $.
Since $\pi_{\mathsf{D}}$ is unknown, the score function $ \nabla \log p_t$  in \ref{eq:backwardproc_real_initial_condition_introduction}  cannot be computed exactly. To address this issue, an estimator $s(\cdot, \theta^{*}, \cdot)$  is \textit{learned} based on a family of functions $ s: [0,T] \times \mathbb{R}^d \times \mathbb{R}^M \rightarrow  \mathbb{R}^M $ parametrized in $\theta$, aiming at approximating the score of the ergodic (forward) diffusion process over a fixed time window $[0,T]$. In practice, the functions $s$ are neural networks and in particular cases like the motivating example in Section \ref{Motivating_example}, the functions $s$ can be wisely designed.
The optimal value  $\theta^*$ of the parameter $\theta$ is determined  by  optimizing the following score-matching objective
\begin{equation}\label{eq:obj_explicit_score_matching_general_no_weighting}
	\begin{split}
		\mathbb{R}^d \ni \theta \mapsto  \E \left[  \int_{0}^T  | \nabla \log p_{t}(X_t)  -  s(t, \theta , X_t) |^2 \  \rmd t \right].
	\end{split}
\end{equation}	
To account for numerical instability issues for training and sampling at $t=0$ as observed in practice in \citet[Appendix C]{song2021scorebased} and for the possibility that the integral of the score function in \ref{eq:obj_explicit_score_matching_general_no_weighting} may diverge when $t= 0$ (see Appendix \ref{full_proof_denoising_score_matching}), a discretised version of the score-matching optimization problem is usually considered
\begin{equation}\label{eq:obj_explicit_score_matching_general}
	\begin{split}
		\text{minimize} \quad \mathbb{R}^d \ni \theta \mapsto U(\theta)
		&:=    \int_{\epsilon}^T  \frac{\kappa(t)}{T-\epsilon}   \int_{\mathbb{R}^M} | \nabla \log p_{t}(x)  -  s(t, \theta , x) |^2 \ p_{t}(x) \ \rmd x \  \rmd t,
	\end{split}
\end{equation}
where $\epsilon>0$ and  $\kappa : [0,T] \rightarrow \mathbb{R}_{>0}$ is a weighting function. The score-matching objective $U$ in \ref{eq:obj_explicit_score_matching_general} can be rewritten via denoising score matching \citep{vincent2011connection} as follows
\begin{equation} \label{objective_with_denoising_score_matching_general}
	\begin{split}
		U(\theta)   & =  \mathbb{E}  [ \kappa(\tau) | \sigma_{\tau}^{-1}Z+ s(\tau, \theta , m_{\tau} X_0+\sigma_{\tau} Z) |^2 ] + C,
	\end{split}
\end{equation}
where the expectation is over $\tau \sim \text{Uniform}([\epsilon,T])$, $X_0 \sim \pi_{\mathsf{D}}$ and $Z \sim \mathcal{N}(0,I_{M})$, and where $C \in \mathbb{R}$ is a constant independent of $\theta$ (see Appendix \ref{full_proof_denoising_score_matching} for the derivation with the OU representation \ref{eq:OUdistribtuion_distribution}).
The stochastic gradient $H:\R^{d} \times \R^m \to \R^{d}$ of \ref{eq:obj_explicit_score_matching_general} deduced using \ref{objective_with_denoising_score_matching_general} is given by
\begin{equation} \label{stochastic_gradient_motivating_example_general}
	\begin{split}
		H(\theta, \mathbf{x})
		&=   2 \kappa(t) \sum_{i=1}^{M} \left(   \sigma_t^{-1}z^{(i)}+ s^{(i)}(t, \theta , m_t x_0+\sigma_t z)\right)\nabla_{\theta} s^{(i)}(t, \theta , m_t x_0+\sigma_t z),
	\end{split}
\end{equation}
where  $\mathbf{x} = (t, x_0, z) \in \mathbb{R}^m$ with $m=2 M+1$. As contribution to this analysis, we introduce an auxiliary process $ (Y_t^{\text{aux}})_{t \in [0,T]}$ containing the approximating function $s$ depending on the (random) estimator of $\theta^*$ denoted by $\hat{\theta} $, for $t \in [0,T]$,
\begin{equation}  \label{Y_auxiliary_theta_hat}
	\text{d} Y_t^{\text{aux}} =  (Y_t^{\text{aux}} + 2 \ s(T-t, \hat{\theta}, Y_t^{\text{aux}}) )  \ \text{d} t + \sqrt{2 } \ \text{d} \bar{B}_t, \quad Y_0^{\text{aux}} \sim \pi_{\infty} = \mathcal{N}(0,I_{M}).
\end{equation}
The process \ref{Y_auxiliary_theta_hat} will play an important role in the derivation of the nonasymptotic estimates in Wasserstein distance of order two between the target data distribution  and the generative distribution of the diffusion model. Indeed, it connects the backward process \ref{Y_hat_auxiliary} and the numerical scheme \ref{continuous_time_EM_version}, which facilitates the analysis of the convergence of the diffusion model (see Appendix \ref{appendix_proof_main_result_motivating_example} and Appendix \ref{appendix_proof_main_result_general_case} for more details). For this reason, we introduce a sequence of stepsizes $\{  \gamma_k\}_{k \in \{0,\dots, K \}}$ such that $\sum_{k=0}^K \gamma_k =T$. For any $k \in \{0, \dots, K  \}$, let $t_{k+1} = \sum_{j=0}^{k} \gamma_j$ with $t_0=0$ and $t_{K+1} = T$. Let $\gamma_k=\gamma \in (0,1)$ for each $k=0,\dots, K$.
The discrete process $(Y_k^{\text{EM}})_{k \in \{0, \dots, K +1 \} }$ of the Euler--Maruyama approximation of \ref{Y_auxiliary_theta_hat} is given, for any $k \in \{0, \dots, K \}$, as follows
\begin{equation} \label{eq:backwardprocemdisc_old}
	Y_{k+1}^{\text{EM}} =	Y_{k}^{\text{EM}} +  \gamma (Y_{k}^{\text{EM}} + 2  \ s(T-t_k, \hat{\theta} ,  Y_{k}^{\text{EM}})) + \sqrt{2  \gamma   } \ \bar{Z}_{k+1}, \quad Y_{0}^{\text{EM}} \sim  \pi_{\infty} = \mathcal{N}(0,I_{M}),
\end{equation}
where $\{\bar{Z}_k \}_{k \in \{0, \dots, K +1 \}}$ is a sequence of independent $M$-dimensional Gaussian random variables with zero mean and identity covariance matrix. We emphasize that the approximation \ref{eq:backwardprocemdisc_old} is the one chosen in the original paper \citet{song2021scorebased}. Finally, the continuous-time interpolation of \ref{eq:backwardprocemdisc_old},  for $t \in [0,T]$, is given by
\begin{equation} \label{continuous_time_EM_version}
	\rmd \widehat{Y}_t^{\text{EM}} =  (\widehat{Y}_{\lfloor t/\gamma \rfloor \gamma }^{\text{EM}} + 2 \ s(T-\lfloor t/\gamma \rfloor \gamma, \hat{\theta} , \widehat{Y}_{\lfloor t/\gamma \rfloor \gamma}^{\text{EM}})) \ \text{d}t + \sqrt{2 } \ \text{d}\bar{B}_t, \qquad \widehat{Y}_0^{\text{EM}} \sim \pi_{\infty} = \mathcal{N}(0,I_{M}),
\end{equation}
where $\mathcal{L}(\widehat{Y}_{k}^{\text{EM}}) = \mathcal{L}(Y^{\text{EM}}_k)$ at grid points for each $k \in\{0, \dots, K +1\}$.

\section{Main Results}

\noindent Before introducing the main assumptions of the paper, we present a motivating example.
\subsection{A Motivating Example:  Full Estimates for Multivariate Gaussian Initial Data with Unknown Mean} \label{Motivating_example}
\noindent In this section, we  consider the case where the data distribution follows a multivariate normal distribution with unknown mean and identity covariance, i.e., $X_0 \sim \pi_{\mathsf{D}}=\mathcal{N}(\mu,I_d)$ for some unknown $\mu \in \mathbb{R}^d$ with $M=d$. We show that, by using diffusion models, we are able to generate new data from an approximate distribution that is close to $\pi_{\mathsf{D}}$. More precisely, we provide a non-asymptotic convergence bound with explicit constants in Wasserstein-2 distance between the law of the diffusion model and $\pi_{\mathsf{D}}$, which can be made arbitrarily small by appropriately choosing key parameters on the upper bound. In this example, the estimation of the score function reduces to the estimation of the unknown mean by the methods of convex optimization with optimal dependence of the dimension combined with the most efficient sampling method for high-dimensional Gaussian data.

In this setting by using \ref{OU_process_introduction}, we can derive the score function given by
\begin{equation} \label{eqn:scoreexample}
	\nabla	\log  p_t(x)  = - x + m_t \mu,
\end{equation}
which can be approximated using
\begin{equation} \label{eqn:appscoreexample}
	s(t,\theta,x) = -  x+ m_t \theta, \quad (t,\theta,x) \in [0,T] \times \mathbb{R}^d \times \mathbb{R}^d.
\end{equation}
To obtain an approximated score, i.e., to obtain an optimal value of $\theta$ in \ref{eqn:appscoreexample}, we opt for a popular class of algorithms called stochastic gradient Langevin dynamics (SGLD) to solve the optimisation problem \ref{eq:obj_explicit_score_matching_general}. In addition, we choose the weighting function $\kappa(t)=\sigma_t^2$ as in \citet{NEURIPS2019_3001ef25} and we set  $\epsilon=0$ in \ref{eq:obj_explicit_score_matching_general}.
Using the approximating function \ref{eqn:appscoreexample} in \ref{stochastic_gradient_motivating_example_general}, we can obtain the following expression for the stochastic gradient
\begin{equation} \label{stochastic_gradient_motivating_example}
	\begin{split}
		H(\theta, \mathbf{x})
		& = 2 \sigma^2_t  \sum_{i=1}^d \left(   \sigma_t^{-1}z^{(i)}+(-m_t x_0^{(i)}-\sigma_t z^{(i)}+m_t\theta^{(i)})\right)m_te_i
		\\	&= 2 \sigma^2_t m_t \left(   \sigma_t^{-1}z -m_t x_0 -\sigma_t z+m_t\theta \right),
	\end{split}
\end{equation}
where  $\mathbf{x} = (t, x_0, z) \in \mathbb{R}^m$ with $m=2 d+1$ and $e_i$ denotes the unit vector with $i$-th entry being $1$. Fixing the so-called inverse temperature parameter $\beta>0$, the associated SGLD algorithm is given by
\begin{equation}\label{eqn:SGLDexample}
	\theta_0^{\lambda} := \theta_0, \quad \theta_{n+1}^{\lambda} = \theta_{n}^{\lambda} - \lambda H(\theta_n^{\lambda},\mathbf{X}_{n+1}) + \sqrt{2 \lambda /\beta} \ \xi_{n+1}, \quad n \in \mathbb{N}_0,
\end{equation}
where $\lambda>0$ is often called the stepsize or gain of the algorithm, $(\xi_n)_{n\in \N_0}$ is a sequence of standard Gaussian vectors and
\begin{equation} \label{x_bold_SGLD}
	(\mathbf{X}_n)_{n \in \mathbb{N}_0}  = (\tau_n, X_{0,n}, Z_n)_{n \in \mathbb{N}_0},
\end{equation}
is a sequence of i.i.d.\ random variables generated as follows. For each $n \in \mathbb{N}_0$, we sample $\tau_n$ from $ \text{Uniform}([0,T])$ such that $\mathcal{L}(\tau_n) = \mathcal{L}(\tau) $,  sample $X_{0,n}$ from $\pi_{\text{D}} = \mathcal{N}(\mu,\text{I}_{d})$ such that $\mathcal{L}(X_{0,n}) = \mathcal{L}(X_{0}) $, and sample $Z_n$ from $ \mathcal{N}(0,I_{d})$ such that $\mathcal{L}(Z_{n}) = \mathcal{L}(Z) $. 	
In addition, we consider the case where $\theta_0$, $(\xi_n)_{n \in \mathbb{N}_0}$, and $(\mathbf{X}_n)_{n \in \mathbb{N}_0} $ in \ref{eqn:SGLDexample} are independent and we have $	\mathbb{E}[H(\theta, \mathbf{X}_{n+1})] = \nabla U(\theta)$.

\noindent Throughout this section, we fix
\begin{align}
	0  < \lambda & \leq \min\{\E[\sigma_{\tau}^2 m_{\tau} ^2]/(4\E[\sigma_{\tau}^4 m_{\tau}^4]), 1/(2\E[\sigma_{\tau}^2 m_{\tau} ^2])\}. \label{lambda_value_SGLD}
\end{align}

\begin{table}[htb] 	\caption{Explicit expressions for the constants in  Theorem \ref{main_theorem_toy_example}.}
	\renewcommand{\arraystretch}{2}
	\centering

	\begin{sc}
		
		\scriptsize
		\begin{tabular}{@{}lllll@{}}
			\toprule
			
			\multicolumn{1}{c}{Constant} &  \multicolumn{4}{c}{Full Expression}   \\
			
			
			
			\toprule
			$C_{\mathsf{SGLD},1}$  & & & & $1/ \E[\sigma_{\tau}^2 m_{\tau}^2]$
			\\\hline
			$ C_{\mathsf{SGLD},2}$ & & & & $4  \E[ \sigma_{\tau}^4 m_{\tau}^2(\sigma_{\tau}^{-1}|Z|+ m_{\tau}  |X_0| +\sigma_{\tau} |Z|+m_{\tau}  |\theta^*| )^2 ]/ \E[\sigma_{\tau}^2 m_{\tau}^2]$
			\\\hline
			$T_{\delta}$ & & & & $2^{-1} \ln \left(4 \sqrt{2} \left(\sqrt{\E[ |X_0|^2 ]}+\sqrt{d} \right)/\delta\right)$ \\\hline
			$\beta_{\delta}$& & & & $144d(\sqrt{4/3}+ 2 \sqrt{33})^2/(\delta^2\E[\kappa(\tau)m_{\tau}^2])$ \\\hline
			$\lambda_{\delta}$& & & &
			$
			\begin{aligned}
				&\min\left\{\E[\sigma_{\tau}^2 m_{\tau} ^2]/(4\E[\sigma_{\tau}^4m_{\tau} ^4]),1/(2\E[\sigma_{\tau}^2 m_{\tau}^2]),\delta^2 \E[\sigma_{\tau}^2m_{\tau}^2] \right.\\
				& \left.\times(576(\sqrt{4/3}+ 2 \sqrt{33})^2\E[ \sigma_{\tau}^4m_{\tau}^2\left(\sigma_{\tau}^{-1}|Z|+ m_{\tau}  |X_0| +\sigma_{\tau} |Z|+m_{\tau}  |\theta^*| \right)^2 ])^{-1} \right\}
			\end{aligned}
			$  \\\hline
			$n_{\delta}$& & & & $(\lambda\E[\sigma_{\tau}^2 m_{\tau} ^2])^{-1}\ln\left(12 (\sqrt{4/3}+ 2 \sqrt{33}) \sqrt{\E\left[|\theta_0-\theta^* |^2\right]}/\delta\right)$ with fixed $\lambda \ (\leq \lambda_{\delta})$  \\\hline
			$\gamma_{\delta}$& & & &$\min\left\{\delta /(  4(18 d+132   |\theta^*|^2)^{1/2}), 1/2\right\}$ \\
			\bottomrule
		\end{tabular}
		
	\end{sc}
	\label{tab:convconst}
\end{table}

Theorem \ref{main_theorem_toy_example} states the non-asymptotic (upper) bounds between the generative distribution of the diffusion model $\mathcal{L}(\widehat{Y}_{K+1}^{\text{EM}})$ and the data distribution $ \pi_{\mathsf{D}}$. An overview of the proof can be found in Appendix \ref{appendix_proof_main_result_motivating_example}.
\begin{thm} \label{main_theorem_toy_example}
	Under the setting described in this section, then, for any $T>0$ and $\gamma \in (0,1/2]$,
	\begin{equation} \label{statement_theorem_example_case_inequality}
		\begin{split}
			&	W_2(\mathcal{L}(\widehat{Y}_{K+1}^{\text{EM}}),\pi_{\mathsf{D}})
			\\ & \leq  \sqrt{2} e^{- 2 T} (\sqrt{\E\left[ |X_0|^2 \right]}+\sqrt{  d} )  \\
			&\quad +(\sqrt{4/3}+2 \sqrt{33}) (e^{-n\lambda\E[\sigma_{\tau}^2 m_{\tau} ^2 ] }\sqrt{\E[|\theta_0-\theta^* |^2]}+ \sqrt{dC_{\mathsf{SGLD},1} /\beta}+ \sqrt{\lambda C_{\mathsf{SGLD},2}})\\
			&\quad +\gamma (\sqrt{18 d}+ \sqrt{132|\theta^*|^2} ),
		\end{split}
	\end{equation}
	where $C_{\mathsf{SGLD},1}$  and $C_{\mathsf{SGLD},2}$ are given explicitly in Table \ref{tab:convconst}. 	
	In addition,	the result in \ref{statement_theorem_example_case_inequality} implies that for any $\delta>0$, if we choose $T>T_{\delta}$, $\beta  \geq \beta_{\delta}$, $0< \lambda\leq \lambda_{\delta}$, $n\geq n_{\delta}$, and $0< \gamma < \gamma_{\delta}$, then
	\begin{equation*}
		W_2(\mathcal{L}(\widehat{Y}_{K+1}^{\text{EM}}),\pi_{\mathsf{D}}) <~\delta,
	\end{equation*}
	where $T_{\delta},\beta_{\delta},\lambda_{\delta},n_{\delta}$ and $\gamma_{\delta}$ are given explicitly in Table \ref{tab:convconst}.
\end{thm}

\begin{remark} The result in Theorem \ref{main_theorem_toy_example}  achieves the optimal rate of convergence of order one for Euler or Milstein schemes of SDEs with constant diffusion coefficients. Furthermore, one notes that the dependence of the dimension on the upper bound in \ref{statement_theorem_example_case_inequality} is in the form of $\sqrt{d}$. To the best of the authors' knowledge, the result in Theorem \ref{main_theorem_toy_example} is the first convergence bound where the parameters involved in the sampling and optimization steps of the diffusion models appear explicitly.  
	In the optimization procedure, we use SGLD algorithm \ref{eqn:SGLDexample} to solve the problem \ref{eq:obj_explicit_score_matching_general}.  Since the stochastic gradient $H$ in \ref{stochastic_gradient_motivating_example} is strongly convex by Proposition \ref{Lipschitz_stochastic_gradient_motivating_example}, it has been proved, for instance in \citet{convex}, that, for large enough $\beta>0$, the output of SGLD is an almost minimizer of \ref{eq:obj_explicit_score_matching_general} when $n$ is large. Thus, in the diffusion model, we set $\hat{\theta} = \theta_{n}^{\lambda}$ indicating $\nabla \log p_t(x) \approx s(t,\hat{\theta},x) = s(t,\theta_{n}^{\lambda},x)$ for large values of $n$ and for all $t$ and $x$. Crucially, this allows us to use the established convergence results for SGLD to deduce a sampling upper bound for $W_2(\mathcal{L}(\widehat{Y}_{K+1}^{\text{EM}}),\pi_{\mathsf{D}})$ with explicit constants  in \ref{statement_theorem_example_case_inequality}. Consequently, this bound can be controlled by any given precision level $\delta>0$ by appropriately choosing $T,\beta,\lambda,n$ and $\gamma$.
\end{remark}
This motivating example has focused on exploring the convergence of diffusion-based generative models using a Langevin-based algorithm, specifically SGLD, which is well-known for its theoretical guarantees in achieving global convergence.  However, in the general case discussed in Section \ref{section_general_case}, we do not prescribe a specific optimizer  to choose to minimise the distance between $\hat{\theta}$ and $\theta^*$.

\subsection{General Case} \label{section_general_case}
\noindent In this section, we derive the full non-asymptotic estimates in Wasserstein distance of order two between the target data distribution $\pi_{\mathsf{D}}$  and the generative distribution of the diffusion model under the assumptions stated below. As explained in Section \ref{Introduction_score_based_generative_models_background} (see also Appendix \ref{full_proof_denoising_score_matching}), it could be necessary in the general setting to restrict  $t \in [\epsilon,T]$ for $\epsilon \in (0,1)$ in \ref{eq:obj_explicit_score_matching_general}. Therefore, we truncate the integration in the backward diffusion at $ T- \epsilon$ and run the process $(Y_t)_{t \in [0,T - \epsilon]}$.

\subsubsection{Assumptions for the General Case}
\noindent	In the motivating example in Section \ref{Motivating_example}, we have chosen the SGLD algorithm to solve the optimisation problem \ref{eq:obj_explicit_score_matching_general}. Other algorithms, such as ADAM \citep{KingBa15}, TheoPouLa \citep{SabanisLimTheoPoula} and stochastic gradient descent \citep{jentzen2021strong}, can be chosen as long as they satisfy the following assumption. Fix $\epsilon>0$.
\begin{assumption} \label{general_assumption_algorithm}
	Let $\theta^*$ be a minimiser\footnote{The score-matching optimization problem \ref{eq:obj_explicit_score_matching_general} is not necessarily (strongly) convex.} of \ref{eq:obj_explicit_score_matching_general} and let  $\hat{\theta}$ be the (random) estimator of $\theta^*$ obtained through some approximation procedure such that $	\mathbb{E} [ | \hat{\theta} |^4] < \infty$. There exists $ \widetilde{\varepsilon}_{\text{AL}} >0$  such that
	\begin{equation*}
		\mathbb{E} [ | \hat{\theta} - \theta^{*}|^2 ] <  \widetilde{\varepsilon}_{\text{AL}}.
	\end{equation*}
\end{assumption}
\begin{remark} \label{Control_algorithm}
	As a consequence of Assumption \ref{general_assumption_algorithm}, we have $	\mathbb{E} [ | \hat{\theta} |^2 ] < 2  \widetilde{\varepsilon}_{\text{AL}} + 2 | \theta^{*}|^2$.
\end{remark}




We consider the following assumption on the data distribution.

\begin{assumption} \label{assumption_score_strong_monotonicity}
	The data distribution $\pi_{\mathsf{D}}$ has a finite second moment, is strongly log-concave,  and  \newline \qquad \qquad $\int_{\epsilon}^T   | \nabla \log p_t(0)|^2 \text{d} t < \infty$.
\end{assumption}
\begin{remark} \label{remark_contraction_score}
	As a consequence of Assumption \ref{assumption_score_strong_monotonicity} and the preservation of strong log-concavity under convolution, see, e.g., \citet[Proposition 3.7]{logconcavitybooksaumardwellner},  there exists $L_{\text{MO}} : [0,T] \rightarrow (0,\infty]$  such that for all $t \in [0,T]$ and  $x, \bar{x} \in \mathbb{R}^{\text{M}}$, we have
	\begin{equation}  \label{contraction_score_assumption}
		\langle  \nabla \log p_t(x)   -\nabla \log p_t(\bar{x}), x - \bar{x} \rangle  \le - L_{\text{MO}}(t)  | x - \bar{x}|^2.
	\end{equation}
	The function $L_{\text{MO}}(t)$ in \ref{contraction_score_assumption} has a lower bound for all $t \in [0,T]$, which we denote by $\widehat{L}_{\text{MO}}$. Moreover, Assumption \ref{assumption_score_strong_monotonicity} with the estimate \ref{contraction_score_assumption} implies that the processes in \ref{eq:backwardproc_real_initial_condition_introduction} and in  \ref{Y_hat_auxiliary} have a unique strong solution, see, e.g., \citet[Theorem 1]{krylov1991simple}.
\end{remark}

Next, we consider the following assumption on the approximating function $s$ which is used in Remark \ref{main_theorem_general_relaxed_assumption}.
\stepcounter{assumption}
\begin{intassumption} \label{Assumption_2_without_derivative}
	The function $s : [0,T] \times \mathbb{R}^d \times \mathbb{R}^M \rightarrow \mathbb{R}^M $ is continuously differentiable in $x \in \mathbb{R}^M$. Let $D_1 : \mathbb{R}^d \times \mathbb{R}^d   \rightarrow \mathbb{R}_{+}$, $D_2 :  [0,T]  \times  [0,T]  \rightarrow \mathbb{R}_{+}$ and $D_3 : [0,T] \times [0,T] \rightarrow \mathbb{R}_{+}$ be such that $\int_{\epsilon}^T  \int_{\epsilon}^T D_2(t,\bar{t}) \ \text{d}t \ \text{d}\bar{t} < \infty $ and $\int_{\epsilon}^T  \int_{\epsilon}^T  D_3(t, \bar{t}) \ \text{d}t \ \text{d}\bar{t} < \infty $. For $\alpha \in \left[\frac{1}{2},1 \right]$ and for all $t, \bar{t} \in [0,T]$, $x, \bar{x} \in \mathbb{R}^{\text{M}}$, and $\theta, \bar{\theta} \in \mathbb{R}^d$, we have that
	\begin{equation*}
		\begin{split}
			| s(t, \theta,x) - s(\bar{t}, \bar{\theta},\bar{x}) | & \le  D_1(\theta, \bar{\theta})  |t - \bar{t}|^{\alpha} + D_2(t, \bar{t})  | \theta - \bar{\theta} |
			+ D_3(t,\bar{t})  |x -\bar{x} |,
		\end{split}		
	\end{equation*}
	where $D_1$, $D_2$ and $D_3$ have the following growth in each variable: i.e. there exist $\mathsf{K}_1$, $\mathsf{K}_2$, and $\mathsf{K}_3>0$ such that
	for each $t,\bar{t} \in [0,T]$ and $ \theta, \bar{\theta} \in \mathbb{R}^d$,
	\begin{equation*}
		\begin{split}
			|D_1(\theta, \bar{\theta}) |  & \le \mathsf{K}_1 (1+ | \theta | + | \bar{\theta} | ),
			\qquad 	|D_2(t, \bar{t}) | 	 \le \mathsf{K}_2 (1+ | t |^{\alpha} + | \bar{t}|^{\alpha}),
			\\ 	|D_3(t,\bar{t}) | & \le \mathsf{K}_3 (1+ | t |^{\alpha} + | \bar{t}|^{\alpha}).
		\end{split}
	\end{equation*}
\end{intassumption}
By adding a further condition on the gradient of $s$ in  Assumption \ref{Assumption_2_without_derivative} we are able to achieve the optimal rate of convergence in Theorem \ref{main_theorem_general} below.
\begin{intassumption} \label{Assumption_2}
Let $s$ be as in Assumption \ref{Assumption_2_without_derivative} and there exists $\mathsf{K}_4>0$ such that, for all $x, \bar{x} \in \mathbb{R}^M$ and  for any $k=1, \dots M$,
	\begin{equation*} 
		| \nabla_x s^{(k)}(t,\theta,x) - \nabla_{\bar{x}} s^{(k)} (t,\theta,\bar{x})  | \le  \mathsf{K}_4 (1+2 |t|^{\alpha}) | x -\bar{x}|.
	\end{equation*}
\end{intassumption}

\begin{remark} \label{remark_growth_estimate_neural_network}
	Let $ \mathsf{K}_{\text{Total}} := \mathsf{K}_1+\mathsf{K}_2+\mathsf{K}_3+  |  s(0, 0,0) |>0$. Using Assumption \ref{Assumption_2}, we have
	\begin{equation*}
		\begin{split}
			|s(t, \theta,x)| & \le   \mathsf{K}_{\text{Total}} (1+ | t |^{\alpha} ) (1+ | \theta | + |x| ).
		\end{split}
	\end{equation*}
\end{remark}
\noindent We postpone the proof of Remark \ref{remark_growth_estimate_neural_network} to Appendix \ref{proofs_appendix_preliminary_results_general_case_statements}.
\begin{remark} Assumption~\ref{Assumption_2_without_derivative} and \ref{Assumption_2}  impose Lipschitz continuity on a family of approximating functions $s(\cdot, \cdot, \cdot)$ with respect to the input variables, $t$ and $x$, as well as the parameters $\theta$. It is well-known that the continuity properties of neural networks with respect to $t$, and $x$ are largely determined by the activation function at the last layer. For instance, neural networks with activation functions such as hyberbolic tangent and sigmoid functions at the last layer satisfy the Lipschitz continuity with respect to $t$ and $x$ \citep{virmaux2018lipschitz, fazlyab2019efficient}.
\end{remark}

For the following assumption on the score approximation, we let  $\hat{\theta}$ be as in Assumption \ref{general_assumption_algorithm} and we let $ (Y_t^{\text{aux}})_{t \in [0,T]}$ be the auxiliary process defined in \ref{Y_auxiliary_theta_hat}.
\begin{assumption} \label{assumption_equivalence_global_minimiser_epsilon}
	There exists $\varepsilon_{\text{SN}} >0$ such that
	\begin{equation}\label{score_error_auxiliary}
		\mathbb{E}  \int_0^{T-\epsilon}   | \nabla \log p_{T - r}(Y_r^{\text{aux}}) - s(T-r, \hat{\theta} , Y_r^{\text{aux}})  |^2  \  \text{d} r <  \varepsilon_{\text{SN}} .
	\end{equation}
\end{assumption}

\begin{remark}
	We highlight that the expectation in Assumption \ref{assumption_equivalence_global_minimiser_epsilon} is taken over the auxiliary process \ref{Y_auxiliary_theta_hat}. This density is known since the approximating function $s$ and the estimator $\hat{\theta}$ are known. To the best of authors' knowledge, this is a novelty with respect to previous works \citep{debortoli2022convergence,chen2023sampling,lee2022convergence, lee2023convergence,chen2023improved,conforti2023score,benton2023linear} which consider the unknown density of the forward process (or its numerical discretization).
\end{remark}	

\begin{remark} \label{Remark_8_Assumption_4}
	Assumption \ref{assumption_equivalence_global_minimiser_epsilon} is satisfied, along with Assumption \ref{Assumption_2}, for data distributions satisfying Assumption \ref{assumption_score_strong_monotonicity}, beyond the multivariate Gaussian case discussed in the motivating example. Indeed,
    		\begin{equation} \label{score_approximation_different_scenario}
			\begin{split}
				& \mathbb{E}  \int_0^{T-\epsilon}   | \nabla \log p_{T - r}(Y_r^{\text{aux}}) - s(T-r, \hat{\theta} , Y_r^{\text{aux}})  |^2  \  \text{d} r
				\\ & \le 2 \ \mathbb{E}  \int_0^{T-\epsilon}   | \nabla \log p_{T - r}(Y_r^{\text{aux}}) - s(T-r, \theta^* , Y_r^{\text{aux}})  |^2  \  \text{d} r
				\\ & \quad + 2 \  \mathbb{E}  \int_0^{T-\epsilon}   |  s(T-r, \theta^* , Y_r^{\text{aux}}) -  s(T-r, \hat{\theta} , Y_r^{\text{aux}})   |^2 \  \text{d} r.
                \end{split}
		\end{equation}
If $\theta^*$ is such that $s(t,\theta^*,x)=\nabla \log p_{t}(x)$ as in the motivating example in Section \ref{Motivating_example}, then the first term on the right-hand side of \ref{score_approximation_different_scenario} vanishes.  Otherwise, we expect that the first term on the right-hand side above to be small. Indeed, by the definition of strong log-concavity, see e.g. \citet[Definition 2.9]{logconcavitybooksaumardwellner}, we have
\begin{equation}  \label{score_strongly_log}
	\begin{split}
		\nabla \log p_t(x) & = \nabla \log (g(x)) + \nabla \log (	\phi_t(x) ),
	\end{split}
\end{equation}
 where $g$ is some log-concave function  and $\phi_t$ is a multivariate Gaussian density. If $g$ is a multivariate logistic distribution\footnote{The multivariate logistic distribution is an example of elliptical distribution  widely used in portfolio risk management \citep{xiao2015black,owen1983class}.} \citep{malik1973multivariate}, then its score function is given by
\begin{equation} \label{score_function_logistic}
	\begin{split}
	\frac{\partial}{\partial x_k}	\log(	g(x)) & =  - 1 - (M+1) \frac{- \exp(-x_k)}{ 1+ \sum_{k=1}^M \exp(-x_k) } , \qquad \quad -\infty < x_k < \infty,  \qquad k=1,\dots, M,
	\end{split}
\end{equation}
while the Hessian of $\log(	g(x)) $ is bounded (see Appendix \ref{Additional_discussions_about_Assumption_4_Appendix} for more details).
Therefore, the score function of the multivariate logistic distribution    \ref{score_function_logistic} is Lipschitz, and, as a consequence, $\nabla \log p_t(x)$ in \ref{score_strongly_log} is still Lipschitz. Thus, we expect to have a good control on the first term on the right-hand side of \ref{score_approximation_different_scenario} since the function $s$ satisfying Assumption \ref{Assumption_2} approximates a Lipschitz score function.  In general, there exists a function
\begin{equation} \label{function_c_remark_8}
		c(t,x):=	\nabla \log p_{t}(x) - s(t, \theta^* , x), \qquad t >0, \ x \in \mathbb{R}^M,
		\end{equation}
      and therefore one has to define a log-concave function $g$ such that   the first term on the right-hand side of \ref{score_approximation_different_scenario} is small. Clearly, this is a problem specific challenge and it may not always have a good solution. The second term on the right-hand side of \ref{score_approximation_different_scenario} is controlled by Assumption \ref{Assumption_2}
        and  Assumption \ref{general_assumption_algorithm} (see Appendix \ref{Additional_discussions_about_Assumption_4_Appendix} for more details).

\end{remark}

\begin{remark}
	We conduct a numerical experiment to show the convergence of diffusion-based generative models under Assumption \ref{general_assumption_algorithm}, \ref{assumption_score_strong_monotonicity}, \ref{Assumption_2} and \ref{assumption_equivalence_global_minimiser_epsilon}. We consider the case where  $X_0 \sim \pi_{\mathsf{D}}=\mathcal N (\mu, I_d)$ with $\mu = (-1.2347, -0.89244)$ and $d=2$. We use SGLD algorithm \ref{eqn:SGLDexample} to solve the optimization problem \ref{objective_with_denoising_score_matching_general}, with $\kappa(t) = \sigma_t^2$, $\epsilon=0$, $T=2$, $\lambda = 5\times 10^{-5}$, $\beta=10^{12}$, and $n=4\times 10^{4}$. At each iteration, $128$ mini-batch samples are used to estimate the stochastic gradient. Then, we generate samples using the Euler-Maruyama approximation \ref{eq:backwardprocemdisc_old} with $\gamma = 10^{-3}$ and $s$ given in \ref{eqn:appscoreexample}.
	At each iteration $n$, we evaluate the quality of $100,000$ generated samples using the Wasserstein distance of order two. In addition, we compute the $L^2$ error between the score function and the approximated function $s$ with $\theta_n^\lambda$, using the auxiliary process as in Assumption \ref{assumption_equivalence_global_minimiser_epsilon}. Figure~\ref{fig:toy} (a) demonstrates the error in Assumption \ref{assumption_equivalence_global_minimiser_epsilon} vanishes as the generative model converges, where the degree of convergence is measured by $W_2(\mathcal{L}(Y_K^{\text{EM}}),\pi_{\mathsf{D}})$. This empirical observation justifies Assumption \ref{assumption_equivalence_global_minimiser_epsilon}. Moreover, we explore the relationship between the quality of generated samples and the error using \ref{objective_with_denoising_score_matching_general} via denoising score matching in Figure~\ref{fig:toy} (b).
	
	
	
	\begin{figure}
		\centering
		\begin{subfigure}[b]{0.47\textwidth}
			\includegraphics[width=\textwidth]{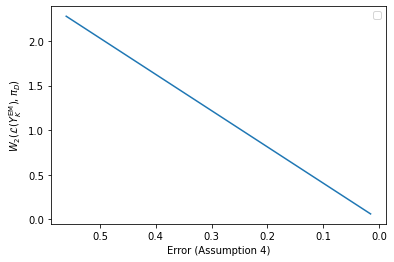}
			\caption{Error with Assumption \ref{assumption_equivalence_global_minimiser_epsilon}.}
		\end{subfigure}
		\begin{subfigure}[b]{0.47\textwidth}
			\includegraphics[width=\textwidth]{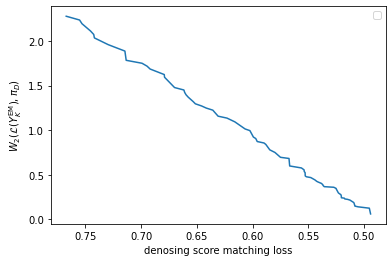}
			\caption{Error using $U(\theta)$ in \ref{objective_with_denoising_score_matching_general}.}
		\end{subfigure}
		\captionsetup{format=hang}
		\caption{The quality of generated samples with respect to (a) the error with Assumption \ref{assumption_equivalence_global_minimiser_epsilon} and (b) the error obtained through denoising score matching using $U(\theta)$ in \ref{objective_with_denoising_score_matching_general}.}
		\label{fig:toy}
	\end{figure}
\end{remark}


\subsubsection{Full Estimates for the General Case}

\noindent	The main result under the general setting is stated as follows. An overview of the proof can be found in Appendix \ref{appendix_proof_main_result_general_case}.
\begin{thm} \label{main_theorem_general}
	Let Assumptions \ref{general_assumption_algorithm}, \ref{assumption_score_strong_monotonicity},  \ref{Assumption_2} and  \ref{assumption_equivalence_global_minimiser_epsilon} hold.
	Then, there exist constants $C_1$, $C_2$, $C_3$ and $C_4 >0$ such that	for any $T>0$ and $\gamma, \epsilon \in (0,1)$, 
	\begin{equation} \label{final_bound_first_inequality_statement_theorem}
		\begin{split}
			W_2(\mathcal{L}(Y_K^{\text{EM}}),\pi_{\mathsf{D}})  & \le C_1 \sqrt{  \epsilon} + C_2 e^{      - 2 \widehat{	L}_{\text{MO}} (T-\epsilon) -\epsilon} +  C_3(T,\epsilon)  \sqrt{\varepsilon_{\text{SN}}} + C_4(T,\epsilon)  \gamma^{\alpha},
		\end{split}
	\end{equation}
	where $C_1$, $C_2$, $C_3$ and $C_4$ are given explicitly in Table \ref{tab:convconst_general} (Appendix \ref{Table_of_constants_appendix}).
	In addition, the  result in \ref{final_bound_first_inequality_statement_theorem} implies that for any $\delta>0$, if we choose $0 \le \epsilon<\epsilon_{\delta}$, $T>T_{\delta}$, $ 0 < \varepsilon_{\text{SN}} < \varepsilon_{\text{SN}, \delta}$ and $0< \gamma < \gamma_{\delta}$ with  $\epsilon_{\delta}$, $T_{\delta}$, $\varepsilon_{\text{SN}, \delta}$, and $\gamma_{\delta}$ given in Table \ref{tab:convconst_general}, then
	\begin{equation*}
		W_2(\mathcal{L}(Y_{K}^{\text{EM}}),\pi_{\mathsf{D}}) <~\delta.
	\end{equation*}		
\end{thm}
\begin{remark}
	The error bounds in  \ref{final_bound_first_inequality_statement_theorem} are not as good as the ones provided in Theorem \ref{main_theorem_toy_example} due to the general form of the approximating function $s$.  We emphasize that the optimal rate of convergence of order $\alpha \in \left[\frac{1}{2},1 \right]$ for the Euler or Milstein scheme of SDEs with constant diffusion coefficients is achieved using the Lipschitz continuity on the derivative of $s$ in Assumption \ref{Assumption_2}. In the explicit expression of $C_4$ in Table \ref{tab:convconst_general}, the dependence of the dimension is $O(M)$ due to numerical techniques from \citet{kumar2019milstein} used in the proof of Theorem \ref{main_theorem_general} to achieve the optimal rate of convergence.  
\end{remark}

\begin{remark}  \label{main_theorem_general_relaxed_assumption}
   If we replace Assumption \ref{Assumption_2} with Assumption \ref{Assumption_2_without_derivative} in Theorem \ref{main_theorem_general}, then the bound \ref{final_bound_first_inequality_statement_theorem} becomes 
    	\begin{equation*} 
		\begin{split}
			W_2(\mathcal{L}(Y_K^{\text{EM}}),\pi_{\mathsf{D}})  & \le C_1 \sqrt{  \epsilon} + C_2 e^{      - 2 \widehat{	L}_{\text{MO}} (T-\epsilon) -\epsilon} +  C_3(T,\epsilon)  \sqrt{\varepsilon_{\text{SN}}} + \widetilde{C}_4(T,\epsilon)  \gamma^{1/2},
					\end{split}
		\end{equation*}
        where $C_1$, $C_2$, $C_3$ are the same as in Theorem \ref{main_theorem_general} and $\widetilde{C}_4(T,\epsilon)$ contains a better dependence on the dimension, namely $O(\sqrt{M})$, than the one achieved in Theorem \ref{main_theorem_general}, and it is explicitly provided in Table \ref{tab:convconst_general} (Appendix \ref{Table_of_constants_appendix}). Although this relaxation achieves the same dependence on the data dimension as in the motivating example in Theorem \ref{main_theorem_toy_example}, it leads to a worse rate of convergence of order $1/2$.
\end{remark}

\section{Related Work and Comparison}\label{sec:related_work}
\noindent We describe assumptions and results of various existing works in our notation and framework to facilitate the comparisons with our results which are provided in 2-Wasserstein distance. Beyond being theoretical relevant, the choice of the use of this metric is motivated by its applications in generative modeling. A popular performance metric currently used to examine the quality of the images produced by generative models is the Fr{\'e}chet Inception Distance (FID) which was originally introduced in \citet{heusel2017gans}. This metric measures the Fr{\'e}chet distance between the distribution of generated samples and the distribution of real samples, assuming Gaussian distributions, which is equivalent to the Wasserstein distance of order two. Our results provided under this metric are therefore also relevant for practical applications.

We can classify the previous approaches based on their assumptions on two key quantities: (i) score approximation error and (ii) assumptions on the data distribution. Based on these approaches, we give a brief overview of some of the most relevant recent contributions in the field.

\subsection{Score Approximation Assumptions} 
\noindent The early work of \citet{de2021diffusion} requires an $L^{\infty}$-bound on the score approximation, which  is in contrast with the $L^2$ nature of the score-matching optimization problem \ref{eq:obj_explicit_score_matching_general}. Most of the recent analysis of score-based generative models, see, e.g., \citet{chen2023sampling,lee2022convergence,lee2023convergence,chen2023improved,conforti2023score,benton2023linear}, have considered assumptions (in $L^2$) on the absolute error of the following type, i.e. for any $k \in \{0, \dots, N-1\}$, there exists $\underline{\varepsilon}>0$
\begin{equation} \label{general_assumptions_related_works}
	\mathbb{E} \left[ |\nabla \log p_{t_k}(X_{t_k}) - s(t_k, \hat{\theta}, X_{t_k}) |^2\right] \leq \underline{\varepsilon},
\end{equation}
where the expectation is taken with respect to the unknown $\{p_t \}_{t \in [0,T]}$ and $\hat{\theta}$ is a deterministic quantity. In \citet{chen2023improved,conforti2023score,benton2023linear}, the assumption \ref{general_assumptions_related_works} is written in integral (averaged) form. In \citet{lee2023convergence}, the bound in \ref{general_assumptions_related_works} is not uniform over $t$, i.e. $\underline{\varepsilon}_t:= \frac{\underline{\varepsilon}}{\sigma_t^4}$ and this allows the score function to grow as  $\frac{1}{\sigma_t^2}$ as $t\rightarrow  0$. The authors in \citet{conforti2023score} use the relative $L^2$-score approximation error such as
\begin{equation} \label{general_assumptions_related_works_relative_score_conforti}
	\mathbb{E} \left[ |2 \nabla \log {\tilde{p}}_{t_k}(X_{t_k}) - \tilde{s}(t_k, \hat{\theta}, X_{t_k}) |^2\right] \leq \underline{\varepsilon} \  \mathbb{E} \left[ | 2 \nabla \log {\tilde{p}}_{t_k}(X_{t_k})   |^2\right],
\end{equation}
where the expectation is with respect to the unknown density of the law of $X_t$ against the Gaussian distribution $\pi_{\infty}(x)$, i.e. $\tilde{p}_t(x): = p_t(x) / \pi_{\infty}(x)$, and $\tilde{s}(t,\theta,x)$ in \ref{general_assumptions_related_works_relative_score_conforti} is the approximating function for the score function of $\tilde{p}_t$.
A pointwise assumption in \citet{debortoli2022convergence}, given by
\begin{equation} \label{eq:de_bortoli_assumption_score_initial}
	| \nabla \log p_t(x) - s(t, \hat{\theta}, x) | \leq C(1 + |x |)/\sigma_t^2,
\end{equation}
for $C \ge 0$, is considered under the manifold (compact) setting. The assumption \ref{eq:de_bortoli_assumption_score_initial} takes into account the explosive behaviour of the score function as $t \rightarrow 0$.  In an attempt to obtain a weaker control than \ref{eq:de_bortoli_assumption_score_initial}, the assumption \citet[A5]{debortoli2022convergence} is used, namely
\begin{align}\label{eq:de_bortoli_assumption_score}
	\mathbb{E}\left[ | \nabla \log p_{T-{t_k}}(Y^\text{EI}_{k}) - s(T-{t_k}, \hat{\theta}, Y^\text{EI}_{k})  |^2\right] \leq C^2 \mathbb{E}\left[\left(1 + |Y^\text{EI}_{k} |^2\right) \right]/\sigma_{T-{t_k}}^4.
\end{align}
We note that, unlike  \ref{general_assumptions_related_works} and \ref{general_assumptions_related_works_relative_score_conforti}, the expectation in \ref{eq:de_bortoli_assumption_score} is taken with respect to the algorithm $Y_{k}^\text{EI}$ given by the exponential integrator (EI) discretization scheme\footnote{This analysis can be extended to the Euler-Maruyama numerical scheme in $Y_{k}^\text{EM}$ \ref{eq:backwardprocemdisc_old}.}, which has known density.
However,  the bounds of \citet[Theorem H.1]{debortoli2022convergence} in Wasserstein distance of order one derived under this assumption scale exponentially in problem parameters such as the diameter of the manifold $\mathfrak{M}$ and the inverse of the early stopping time parameter $\epsilon$, i.e. $\exp(O(\text{diam}(\mathfrak{M})^2/\epsilon))$. Furthermore, an assumption similar to \ref{eq:de_bortoli_assumption_score} is considered in a very recent and concurrent result\footnote{The concurrent paper \citet{gao2023wasserstein} appeared few days earlier when the first draft of this work was made available online.} \citet{gao2023wasserstein}:
\begin{equation*}
	\sup_{k=1,\dots, K} \left( \E \left[ | \nabla \log p_{T-(k-1)\eta}(Y^\text{EI}_{k}) - s(T-(k-1)\eta , \hat{\theta}, Y^\text{EI}_{k}) |^2 \right] \right)^{1/2} \le \underline{\varepsilon},
\end{equation*}
where $\eta>0$ is the stepsize and $T=K \eta$ with $K \in \mathbb{N}$.

We emphasize that the existing results in the literature do not take the expectation with respect to the stochastic optimizer $\hat{\theta}$. This, together with the use of our novel auxiliary process \ref{Y_auxiliary_theta_hat} that uses only known information, allows us to deduce state-of-the-art  bounds in the Wasserstein distance of order two in the following sense. The bounds scale  polynomially in the data dimension $M$, i.e. $O(\sqrt{M})$
as shown in  Theorem \ref{main_theorem_toy_example} and Remark \ref{main_theorem_general_relaxed_assumption}, and achieve the optimal rate of convergence: of order one in Theorem \ref{main_theorem_toy_example} and of order $\alpha \in \left[\frac{1}{2},1 \right]$ in Theorem \ref{main_theorem_general}.



\subsection{Assumptions on the Data Distribution} \label{Assumption_data_distribution_subsection}


\noindent The vast majority of the results available in the literature are in KL divergence and TV distance. For two general data distributions $\mu$ and $\nu$, there is no known relationship between their KL divergence and their $W_2$. However, for strongly log-concave data distributions, a bound on the Wasserstein distance of order two in terms of KL divergence follows from an extension of Talagrand's inequality \citet[Corollary 7.2]{gozlan2010transport}. 

Some convergence results in different metrics can be deduced under the following types of assumptions. Convergence bounds in Wasserstein distance of order one with exponential complexity has been obtained in  \citet{debortoli2022convergence} under the so-called manifold hypothesis, namely assuming that the target distribution is supported on a lower-dimensional manifold or is given by some empirical distribution. Moreover, the results in TV distance in \citet{lee2022convergence} and in KL divergence  \citet{yingxi2022convergence} assumed that the data distribution satisfies a logarithmic Sobolev inequality and the score function is Lipschitz resulting in convergence bounds characterized by polynomial complexity. By replacing the requirement that the data distribution satisfies a functional inequality with the assumption that $\pi_{\text{D}}$ has finite KL divergence with respect to the standard Gaussian and by assuming that the score function for the forward process is Lipschitz, the authors in \citet{chen2023sampling} managed to derive bounds in TV distance which scale polynomially in all the problem parameters. By requiring only the Lipschitzness of the score at the initial time instead of the whole trajectory, the authors in \citet[Theorem 2.5]{chen2023improved} managed to show, using an exponentially decreasing then linear step size, convergence bounds in KL divergence with quadratic dimensional dependence and logarithmic complexity in the Lipschitz constant.  In the work by \citet{benton2023linear}, the authors provide convergence bounds in KL divergence that are linear in the data dimension, up to logarithmic factors, using early stopping and assuming finite second moments of the data distribution. A careful examination of \citet[Proof of Theorem 1 and Corollary 1]{benton2023linear} reveals that the authors still require the uniqueness of solutions for the backward SDE \ref{eq:backwardproc_real_initial_condition_introduction}. For instance, either measurability and boundedness of the drift \citep{stroock1997multidimensional} or integrability of the drift \citep{rockner2023sdes} should be imposed if the uniqueness  of weak solutions is considered. Therefore, additional conditions on the score function, depending on the type of uniqueness of solutions considered, are still needed in their bounds in Theorem 1 and Corollary 1. 

Assuming the finiteness of the second moment of the data distribution and using an EI discretization scheme with constant and exponentially decreasing step sizes, the authors in \citet[Corollary 2.4]{conforti2023score} derive a KL divergence bound with early stopping, which scales linearly in the data dimension up to logarithmic factors. In terms of dependence of the data dimension, this bound is slightly worse than the bound in $W_2$ provided in Remark \ref{main_theorem_general_relaxed_assumption}, which does not include the logarithmic dependence of $M$. Bounds in KL without early stopping have been derived in \citet{conforti2023score} for data distributions with finite Fisher information with respect to the standard Gaussian distribution. The bounds in \citet[Theorem 2.1 and 2.2]{conforti2023score} scale linearly in the Fisher information when an EI discretization scheme with constant step size is used and logarithmically in the Fisher information when an exponential-then-constant step size \citet[Theorem 2.3]{conforti2023score} is chosen. We note that \citet[Theorem 2.1 and 2.2]{conforti2023score} cannot achieve the optimal rate of convergence of Theorem \ref{main_theorem_toy_example} in the motivating example. 

We summarise the results of \citet{chen2023improved, benton2023linear, conforti2023score} and compare them to ours in Table \ref{table_comparison_results_non_early_stopping} and Table \ref{table_comparison_results_early_stopping}, making the distinction based on whether the early stopping rule is applied. In addition, a careful examination of \citet[Proof of Theorem 2.2]{chen2023improved} reveals that the authors require the uniqueness of solutions for the backward SDE \ref{eq:backwardproc_real_initial_condition_introduction}. For instance, when strong solutions are considered, this implies that the score function should be (at least) monotone in the space variable, and a suitable integrability condition in $t$ (similar to our Assumption \ref{assumption_score_strong_monotonicity}) is still needed to guarantee uniqueness of the solution (see, e.g., \citet[Theorem 1]{krylov1991simple}). For weak solutions, we refer to the discussion of \citet[Proof of Theorem 1 and Corollary 1]{benton2023linear} above. 

We emphasize that in Theorem \ref{main_theorem_general}, we do not assume  the score function to be Lipschitz continuous with a uniformly bounded Lipschitz constant. This is particularly useful for future work in nonconvex settings, where the upper bound estimates will be independent of the (potentially large) Lipschitz constant of the score function, which could otherwise hide additional dimensional dependencies. The requirement $\int_{\epsilon}^T   | \nabla \log p_t(0)|^2 \text{d} t < \infty$ in Assumption \ref{assumption_score_strong_monotonicity} is weaker than the Lipschitz assumption on the score function, but it is still difficult to verify in practical applications.

As pointed out in \citet[Section 4]{chen2023sampling}, some type of log-concavity assumption on the data distribution is needed to derive polynomial convergence rates in 2-Wasserstein distance. This justifies the need for our Assumption \ref{assumption_score_strong_monotonicity}. The
concurrent result \citet{gao2023wasserstein} has a similar assumption. For completeness, we also mention that the results in the same metric, which have appeared after the first version of our preprint on arXiv (e.g., \citet{tang2024contractive, strasman2025an}), continue to assume strong log-concavity of the data distribution.


\begin{table}[H]
		\caption{Summary of previous bounds without early stopping and our result in Theorem \ref{main_theorem_toy_example}. Bounds expressed in terms of the number of steps required to guarantee an error of at most $\delta$ in the stated metric. The relative Fisher information of the target $\pi_{\text{D}}$ against standard Gaussian measure $\pi_{\infty}$ is denoted by $\text{FI}(\pi_{\text{D}}|\pi_{\infty})$. All the bounds assume that $\pi_{\text{D}}$ has finite second moments.}
	\begin{tabular}{ |p{2cm} | p{3cm}|p{3cm}| p{4.5cm}| p{2 cm} | } 
		\hline
		Optimization \quad problem solved  &	Regularity  \quad condition	 & Metric & Complexity & Reference  \\
		\hline 
		No & $\forall t$, $\nabla \log p_t$ $L$-Lipschitz & $\sqrt{\text{KL}(\mathcal{L}(\widehat{Y}_{K+1}^{\text{EI}}) || \pi_{\text{D}} )}$ & $\tilde{O}\left(\frac{ML^2}{\delta}  \right)$ & \citet[Theorem  2.1]{chen2023improved} 
		\\ 
		No &    \citet[H2]{conforti2023score}, i.e.	$\text{FI}(\pi_{\text{D}}|\pi_{\infty}) < \infty$   & $	\sqrt{\text{KL}(\mathcal{L}(\widehat{Y}_{K+1}^{\text{EI}})  || \pi_{\text{D}}  )}$  & $	\tilde{O}\left(\sqrt{\frac{M +\E[|X_0|^2]}{\delta} \log^2(\mathfrak{L})} 
		\right) $, with $ \mathfrak{L}: =M^{-1} \text{FI}(\pi_{\text{D}}|\pi_{\infty})$   & \citet[Theorem 2.3]{conforti2023score} 
		\\ 
		Yes &	$\pi_{\text{D}} \sim \mathcal{N}(\mu,I_M)$  & $W_2( \mathcal{L}(\widehat{Y}_{K+1}^{\text{EM}}),\pi_{\mathsf{D}})$  & $\tilde{O}(\frac{\sqrt{M}}{\delta}    ) $ &  Theorem \ref{main_theorem_toy_example}
		\\  
		\hline
	\end{tabular}
	\label{table_comparison_results_non_early_stopping}
\end{table}

\newpage


\begin{table}[H]
		\caption{Summary of previous bounds with early stopping and our results in Remark \ref{main_theorem_general_relaxed_assumption}.   The results in \citet[Theorem 2.2]{chen2023improved}, \citet[Theorem 1]{benton2023linear}, and \citet[Corollary 2.4]{conforti2023score}  are stated for the smoothed version of the data, denoted by $\pi^{\epsilon}_{\text{D}} $ and using  the score approximation assumption \ref{general_assumptions_related_works} with $\underline{\varepsilon}$. Therefore, an additional error should be added to their bounds, as the distance between $\pi^{\epsilon}_{\text{D}} $ and $\pi_{\text{D}} $ scales with $\sqrt{M}$ in $W_2$ (see Appendix \ref{appendix_proof_main_result_general_case}). } 
	\begin{tabular}{ |p{3cm} | p{2.8cm}|p{7cm}| p{2cm}| p{2 cm} | } 
		\hline
			Assumption on the data 	 & Metric & Error bound & Reference  \\
		\hline
		 Finite second moments of $\pi_{\text{D}}$ and \ref{general_assumptions_related_works}. See Section \ref{Assumption_data_distribution_subsection} for conditions on $\nabla \log p_t$ for the uniqueness of $Y_t$.   & $\sqrt{\text{KL}(\mathcal{L}(\widehat{Y}_{K}^{\text{EI}})  || \pi^{\epsilon}_{\text{D}}  )}$ & $ \sqrt{(\mathbb{E}[|X_0|^2] +M)e^{-T} }+ \sqrt{T \underline{\varepsilon}}  +M (T+\log \epsilon^{-1})/\sqrt{N} $,  (+ additional bounds between $\pi^{\epsilon}_{\text{D}}$ and  
    $\pi_{\text{D}}$)  & \citet[Theorem  2.2]{chen2023improved}  \\ 
		   Finite second moments of $\pi_{\text{D}}$ and \ref{general_assumptions_related_works}. See Section \ref{Assumption_data_distribution_subsection} for conditions on $\nabla \log p_t$ for the uniqueness of $Y_t$.   & $\sqrt{\text{KL}(\mathcal{L}(\widehat{Y}_{K}^{\text{EI}})  || \pi^{\epsilon}_{\text{D}}  )}$  & $ \sqrt{\underline{\varepsilon}} + \sqrt{\tilde{O}(M/N)} + \sqrt{M e^{-2 T}}$,
    (+ additional  bounds between $\pi^{\epsilon}_{\text{D}}$ and  
    $\pi_{\text{D}}$) & \citet[Theorem 1]{benton2023linear} 
    \\ Finite second moments of $\pi_{\text{D}}$ and \ref{general_assumptions_related_works}. &  $\sqrt{\text{KL}(\mathcal{L}(\widehat{Y}_{K}^{\text{EI}})  || \pi^{\epsilon}_{\text{D}}  )}$  &  $ \sqrt{(\mathbb{E}[|X_0|^2] +M)e^{-T} }+ \sqrt{T \underline{\varepsilon}} + [c[(M+ \mathbb{E}[|X_0|^2])(\log (M+\mathbb{E}[|X_0|^2]) + \log(\epsilon^{-1}) +1) ] ]^{1/2}$,  with $c,\epsilon \in (0,1/2]$ (+ additional  bounds between $\pi^{\epsilon}_{\text{D}}$ and  
    $\pi_{\text{D}}$) &  \citet[Corollary 2.4]{conforti2023score} 
		\\ 
			Assumption \ref{assumption_score_strong_monotonicity} and Assumption \ref{assumption_equivalence_global_minimiser_epsilon}.  & $	W_2( \mathcal{L}(\widehat{Y}_{K}^{\text{EM}}), \pi_{\mathsf{D}}) $   &  $  O(\sqrt{M}) \sqrt{  \epsilon} + O(\sqrt{M}) e^{      - 2 \widehat{	L}_{\text{MO}} (T-\epsilon) -\epsilon} +   O(e^{(1+  \zeta - 2 \widehat{ L}_{\text{MO}} )(T-\epsilon) } )  \sqrt{\varepsilon_{\text{SN}}} + O(\sqrt{M} e^{T^{2\alpha +1}} T^{2 \alpha +1} \widetilde{\varepsilon}^{1/2}_{\text{AL}})  \gamma^{1/2}
$ &   Remark \ref{main_theorem_general_relaxed_assumption}
        \\ 
		\hline
	\end{tabular}
	\label{table_comparison_results_early_stopping}
\end{table}

\subsubsection*{Acknowledgments}
This work has been supported by The Alan Turing Institute through the Theory and Methods Challenge Fortnights event “Accelerating generative models and nonconvex optimisation”, which took place at The Alan Turing Institute headquarters. This work was made possible by a Research-in-Groups programme funded by the International Centre for Mathematical Sciences, Edinburgh.
This work is supported by the Guangzhou-HKUST(GZ) Joint Funding Programs (No. 2024A03J0630 and No. 2025A03J3322). This work has received funding from the Ministry of Trade, Industry and Energy (MOTIE) and Korea Institute for Advancement of Technology (KIAT) through the International Cooperative R\&D program (No.P0025828). S.S. was supported by the Alan Turing Institute under the EPSRC grant EP/N510129/1.

\bibliography{main}
\bibliographystyle{tmlr}

\appendix
\section*{Appendix}

\section{Objective Function via Denoising Score Matching} \label{full_proof_denoising_score_matching}
\noindent In this section, we show that the objective function $U$ in \ref{eq:obj_explicit_score_matching_general} can be written into \ref{objective_with_denoising_score_matching_general} using denoising score matching \citep{vincent2011connection} and the OU representation
\begin{equation}\label{eq:OUdistribtuion_distribution}
	X_t\overset{\text{d}}{=} m_t X_0+\sigma_t Z, \quad m_t = e^{-  t}, \quad \sigma_t^2 = 1-e^{-2  t}, \quad Z \sim \mathcal{N}(0,I_{M}),
\end{equation}
where $\overset{\text{d}}{=}$ denotes equality in distribution. We start by noticing that
\begin{equation} \label{Inner_product_explicit_score_matching_denoising_score_matching}
	\begin{split}
		&  \int_{\mathbb{R}^M}  \langle \nabla \log p_{t}(x), s(t, \theta , x) \rangle \ p_t(x)  \ \text{d}x
		\\ & = \int_{\mathbb{R}^M}  \langle  \nabla  p_{t}(x), \ s(t, \theta , x) \rangle  \ \text{d}x
		\\ & = \int_{\mathbb{R}^M} \left \langle  \nabla_{x} \int_{\mathbb{R}^M} p_{0}(\tilde{x}) p_{t|0}(x | \tilde{x}) \ \text{d}\tilde{x}, \ s(t, \theta , x) \right \rangle  \ \text{d}x
		\\		& = \int_{\mathbb{R}^M} \left \langle   \int_{\mathbb{R}^M} p_{0}(\tilde{x}) p_{t|0}(x | \tilde{x}) \nabla_{x} \log(p_{t|0}(x | \tilde{x}))  \ \text{d}\tilde{x},  \ s(t, \theta , x) \right \rangle  \ \text{d}x
		\\ & = \int_{\mathbb{R}^M}  \int_{\mathbb{R}^M}   p_{t,0}(x, \tilde{x})  \left \langle  \nabla_{x} \log(p_{t|0}(x | \tilde{x})) ,  \ s(t, \theta ,x) \right \rangle  \ \text{d}x  \ \text{d}\tilde{x},
	\end{split}
\end{equation}
where $p_{t|0}$ is the density of the transition kernel associated with \ref{OU_process_introduction} and $ p_{t,0}$ is the joint density of $X_t$ and $X_0$. Using \ref{Inner_product_explicit_score_matching_denoising_score_matching} in the objective function given in \ref{eq:obj_explicit_score_matching_general} and the OU representation \ref{eq:OUdistribtuion_distribution}, we have
\begin{equation} \label{switch_explicit_to_denosing_score_matching}
	\begin{split}
		U(\theta)  	 & =  \int_{\epsilon}^T \frac{\kappa(t) }{T-\epsilon}  \int_{\mathbb{R}^M}   \left(  | \nabla \log p_{t}(x)|^2   -  2   \langle \nabla \log p_{t}(x), s(t, \theta , x) \rangle  + | s(t, \theta , x)|^2 \right)  \ p_t(x)  \ \text{d}x  \ \text{d} t
		\\	& 	 =   \int_{\epsilon}^T \frac{\kappa(t) }{T-\epsilon }  \int_{\mathbb{R}^M} \int_{\mathbb{R}^M} \Bigg( | \nabla \log p_{t}(x)|^2 - 2 \left \langle  \nabla_{x} \log(p_{t|0}(x |\tilde{x})) ,  \ s(t, \theta , x) \right \rangle
		\\ & \qquad \qquad \qquad \qquad \qquad   + | s(t, \theta , x)|^2 \Bigg)  p_{t,0}(x,\tilde{x}) \   \text{d}x  \ \text{d}\tilde{x} \ \text{d} t
	\\	& = \int_{\epsilon}^T \frac{\kappa(t) }{T-\epsilon }  \int_{\mathbb{R}^M} \int_{\mathbb{R}^M}   | \nabla_{x} \log p_{t|0}(x | \tilde{x}) - s(t,\theta,x)|^2   p_{t,0}(x, \tilde{x})   \  \text{d}x  \ \text{d}\tilde{x}  \ \text{d} t
		\\ & 	\quad  +  \int_{\epsilon}^T \frac{\kappa(t)}{T-\epsilon }      \int_{\mathbb{R}^M}    | \nabla \log p_{t}(x)|^2 p_{t}(x)  \  \text{d}x   \ \text{d} t
		\\ &  \quad -  \int_{\epsilon}^T \frac{\kappa(t) }{T-\epsilon}      \int_{\mathbb{R}^M} \int_{\mathbb{R}^M}   | \nabla_{x} \log p_{t|0}(x | \tilde{x})|^2   p_{t,0}(x, \tilde{x})  \  \text{d}x  \ \text{d}\tilde{x}  \ \text{d} t
		\\	 & = \mathbb{E}  [\kappa(\tau) | \sigma_{\tau}^{-1}Z+ s(\tau, \theta , m_{\tau} X_0+\sigma_{\tau} Z) |^2 ]   	+  \int_{\epsilon}^T \frac{\kappa(t)}{T-\epsilon}    \int_{\mathbb{R}^M}    | \nabla \log p_{t}(x)|^2 p_{t}(x)  \  \text{d}x   \ \text{d} t
		\\ &  \quad -  \int_{\epsilon}^T \frac{\kappa(t) }{T-\epsilon}  \int_{\mathbb{R}^M} \int_{\mathbb{R}^M}   | \nabla_{x} \log p_{t|0}(x | \tilde{x})|^2   p_{t,0}(x, \tilde{x})  \  \text{d}x  \ \text{d}\tilde{x}  \ \text{d} t.
	\end{split}
\end{equation}	
We emphasize that we may choose to restrict $t \in [\epsilon,T]$ with $\epsilon \in (0,1)$  to prevent the divergence of the integrals on the right-hand side of \ref{switch_explicit_to_denosing_score_matching}.

\section{Additional Discussions about Assumption \ref{assumption_equivalence_global_minimiser_epsilon}} \label{Additional_discussions_about_Assumption_4_Appendix}
\noindent In this section, we provide additional details about Remark \ref{Remark_8_Assumption_4}, in which we used the multivariate logistic distribution \citep{malik1973multivariate}
	\begin{equation*}
		\begin{split}
			g(x) & = M! \ \exp \left( - \sum_{k=1}^M x_k  \right) \left( 1+ \sum_{k=1}^M \exp(-x_k) \right)^{-(M+1)},  \qquad k=1,\dots, M, \quad -\infty < x_k < \infty,
		\end{split}
	\end{equation*}
and derived its score function in \ref{score_function_logistic}. The Hessian of  $\log(	g(x))$ is
\begin{equation} \label{bounded_Hession}
	\begin{split}
		\frac{\partial^2}{\partial x_k^2}	\log(	g(x)) & = - (M+1) \left( \frac{ \exp(-x_k)  }{( 1+ \sum_{k=1}^M \exp(-x_k) )} - \frac{ \exp(-2 x_k)}{( 1+ \sum_{k=1}^M \exp(-x_k) )^2}  \right), \\
		\frac{\partial^2}{\partial x_k \partial x_j}	\log(	g(x)) & =  - (M+1) \left( \frac{  \exp(-x_k) \exp(-x_j)}{( 1+ \sum_{k=1}^M \exp(-x_k) )^2} \right), \qquad j, k =1, \dots ,M , \quad  j\neq k,
	\end{split}
\end{equation}
Since the Hessian \ref{bounded_Hession} is bounded, we can conclude that the score function of the multivariate logistic distribution defined in \ref{score_function_logistic} is Lipschitz continuous. By using the same arguments as in Remark \ref{Remark_8_Assumption_4}, we can conclude that $\nabla \log p_t(x)$
 defined in \ref{score_strongly_log} is Lipschitz. 
 
Using Assumption \ref{Assumption_2}
        and  Assumption \ref{general_assumption_algorithm}, we have 
	\begin{equation*}
			\begin{split}
				& \mathbb{E}  \int_0^{T-\epsilon}   | \nabla \log p_{T - r}(Y_r^{\text{aux}}) - s(T-r, \hat{\theta} , Y_r^{\text{aux}})  |^2  \  \text{d} r
				\\ & \le 2 \mathbb{E}  \int_0^{T-\epsilon} |\nabla \log p_{T - r}(Y_r^{\text{aux}}) - s(T-r, \theta^* , Y_r^{\text{aux}})  |^2 \ \rmd r  + 2 \mathsf{K}_2^2 \int_0^{T-\epsilon}  (1+ 2| T-r |^{\alpha} )^2  \   \mathbb{E} [| \hat{\theta} - \theta^*|^2]  \ \rmd r 
				\\ & \le  
				2	\mathbb{E}  \int_0^{T-\epsilon} |\nabla \log p_{T - r}(Y_r^{\text{aux}}) - s(T-r, \theta^* , Y_r^{\text{aux}})  |^2 \ \rmd r	+ 4 \mathsf{K}_2^2  (T- \epsilon)(1 + 4 T^{2 \alpha } ) \widetilde{\varepsilon}_{\text{AL}} < \varepsilon_{\text{SN}},
			\end{split}
		\end{equation*}
where the error of the first term on the right-hand side above is expected to be small since the Lipschitz $\nabla \log p_t(x)$ is approximated by the function $s$ satisfying Assumption \ref{general_assumption_algorithm}. This satisfies  Assumption \ref{assumption_equivalence_global_minimiser_epsilon}. 

In addition, we remark that the motivating example in section \ref{Motivating_example} is a special case of the general setting. Indeed, Assumption \ref{general_assumption_algorithm} is satisfied due to Lemma \ref{lem:sgld2ndmmtexample} with $\hat{\theta}= \theta_n^{\lambda}$ being the $n$th-iterate of the SGLD algorithm \ref{stochastic_gradient_motivating_example} and $\theta^*:=\mu$.
	Assumption \ref{assumption_score_strong_monotonicity} is satisfied by the score function  \ref{eqn:scoreexample}.  Assumption \ref{Assumption_2}  is satisfied by the approximating function \ref{eqn:appscoreexample} with $d=M$, $\alpha=1$, $D_1(\theta, \bar{\theta}):=  |\bar{\theta} | $, $D_2(t, \bar{t}):=e^{- t}$, $D_3(t,\bar{t}):=1$, for any $\theta$, $\bar{\theta} \in \mathbb{R}^d$, $t$, $\bar{t} \in [0,T]$ and $ \mathsf{K}_1= \mathsf{K}_2= \mathsf{K}_3= \mathsf{K}_4=1$.
	Furthermore, both functions satisfy Assumption \ref{assumption_equivalence_global_minimiser_epsilon}  with $d=M$. Indeed, by Assumption \ref{general_assumption_algorithm} with $\theta^*= \mu$, we have
	\begin{equation*}
		\begin{split}
			& \mathbb{E}  \int_0^{T-\epsilon}   | \nabla \log p_{T - r}(Y_r^{\text{aux}}) - s(T-r, \hat{\theta} , Y_r^{\text{aux}})  |^2  \  \text{d} r
			\\ & \le 2 \ \mathbb{E}  \int_0^{T-\epsilon}   | \nabla \log p_{T - r}(Y_r^{\text{aux}}) - s(T-r, \theta^* , Y_r^{\text{aux}})  |^2  \  \text{d} r
			\\ & \quad + 2 \  \mathbb{E}  \int_0^{T-\epsilon}   |  s(T-r, \theta^* , Y_r^{\text{aux}}) -  s(T-r, \hat{\theta} , Y_r^{\text{aux}})   |^2 \  \text{d} r
			\\ & =  (e^{-2\epsilon} - e^{-2T}) \  \mathbb{E} [| \hat{\theta} - \theta^*|^2]  <  (e^{-2\epsilon} - e^{-2T}) \widetilde{\varepsilon}_{\text{AL}}  < \varepsilon_{\text{SN}}.
		\end{split}
	\end{equation*}

\section{Proofs of the Results for the Multivariate Gaussian Initial Data with Unknown Mean} \label{proofs_appendix_motivating_example}

\noindent In this section, we provide the proof of Theorem \ref{main_theorem_toy_example}. We start by introducing the results which will be used in the proof of Theorem \ref{main_theorem_toy_example}.

\subsection{Preliminary Estimates} \label{preliminary_estimates_appendix_motivating_example}
\noindent We provide the results for the SGLD algorithm \ref{eqn:SGLDexample} with $\lambda$ given in \ref{lambda_value_SGLD}, $\beta >0$,  $C_{\mathsf{SGLD},1}$ and $C_{\mathsf{SGLD},2}$ given in  Table \ref{tab:convconst}, as well as for the auxiliary process \ref{Y_auxiliary_theta_hat}, the discrete process \ref{eq:backwardprocemdisc_old}, and the continuous-time interpolation \ref{continuous_time_EM_version} with $s$ given in \ref{eqn:appscoreexample} and $\gamma \in (0,1/2]$.

\begin{prop} \label{Lipschitz_stochastic_gradient_motivating_example}
	For any $\theta, \bar{\theta} \in \R^{\text{d}}$ and $\mathbf{x} \in \R^m$,
	\begin{equation*}\label{eqn:exampleHlip}
		\begin{split}
			|H(\theta, \mathbf{x}) - H(\bar{\theta} , \mathbf{x})| & =2 \sigma_t^2 m_t ^2|\theta- \bar{\theta}  |,
			\\	\langle H(\theta, \mathbf{x}) - H(\bar{\theta} , \mathbf{x}), \theta-\bar{\theta} \rangle & = 2 \sigma_t^2 m_t^2|\theta- \bar{\theta}  |^2.
		\end{split}
	\end{equation*}
\end{prop}
\noindent Proposition \ref{Lipschitz_stochastic_gradient_motivating_example} is derived from the definition of the stochastic gradient \ref{stochastic_gradient_motivating_example}.

The proof of the following lemmas are postponed to Section \ref{proof_preliminary_estimates_appendix_motivating_example}.

\begin{lem}\label{lem:sgld2ndmmtexample} For any $n\in \N_0$,
	\[
	\E\left [|\theta_{n}^{\lambda}-\theta^*|^2 \right]\leq  (1-2\lambda\E[\sigma_{\tau}^2 m_{\tau} ^2 ] )^{n}\E\left[|\theta_0-\theta^* |^2\right]+\text{d} C_{\mathsf{SGLD},1} /\beta+\lambda C_{\mathsf{SGLD},2}.
	\]
\end{lem}

\noindent As a consequence of Lemma \ref{lem:sgld2ndmmtexample}, we derive the following corollary.
\begin{cor}\label{lem:sgld2ndmmtcorexample} For any $n\in \N_0$,
	\[
	\E\left[|\theta_{n}^{\lambda} |^2\right]\leq  2e^{-2n\lambda\E[\sigma_{\tau}^2m_{\tau}^2 ] }\E\left[|\theta_0-\theta^* |^2\right]+2 \text{d} C_{\mathsf{SGLD},1} /\beta+2\lambda C_{\mathsf{SGLD},2}+2|\theta^*|^2.
	\]
\end{cor}
\begin{lem}\label{lem:auxproc2ndbdexample} It holds that
	\[
	\sup_{t \ge 0}\E\left[|Y_t^{\text{aux}} |^2\right]\leq C_{\mathsf{aux}},
	\]
	where $C_{\mathsf{aux}}: = (8/3)(e^{-2n\lambda\E[\sigma_{\tau}^2 m_{\tau}^2 ] }\E\left[|\theta_0-\theta^* |^2\right]+\text{d}C_{\mathsf{SGLD},1} /\beta+\lambda C_{\mathsf{SGLD},2}+|\theta^*|^2)+2 \text{d}$.
\end{lem}
\begin{lem}\label{lem:EMal2ndbdexample}
	It holds that
	\[
	\sup_{k\in \mathbb{N}}\E \left[|Y_{k}^{\text{EM}}|^2 \right]\leq C_{\mathsf{EM}},
	\]
	where $C_{\mathsf{EM}}:= 3 \text{d}+20( e^{-2n\lambda\E[\sigma_{\tau}^2 m_{\tau} ^2 ] }\E\left[|\theta_0-\theta^* |^2\right]+ \text{d} C_{\mathsf{SGLD},1} /\beta+\lambda C_{\mathsf{SGLD},2}+ |\theta^*|^2)$.
\end{lem}
\begin{lem}\label{lem:EMprocoseexample} For any $0 < \gamma \le 1/2$, one obtains that
	\[
	\sup_{t \ge 0}\E\left[|\widehat{Y}_t^{\text{EM}}  - \widehat{Y}_{\lfloor t/\gamma \rfloor \gamma}^{\text{EM}}  |^2\right]\leq  \gamma C_{\mathsf{EMose}},
	\]
	where $C_{\mathsf{EMose}}:= 8d+56  ( e^{-2n\lambda\E\left[\sigma_{\tau}^2 m_{\tau} ^2 \right] }\E\left[|\theta_0-\theta^* |^2\right]+ dC_{\mathsf{SGLD},1} /\beta+\lambda C_{\mathsf{SGLD},2}+ |\theta^*|^2)$.
\end{lem}

\begin{lem}\label{lem:EMproc2ndbdexample} 	It holds that
	\[
	\sup_{t \ge 0}\E\left[|\widehat{Y}_t^{\text{EM}} |^2 \right]\leq \widehat{C}_{\mathsf{EM}},
	\]
	where $
	\widehat{C}_{\mathsf{EM}}
	:=18d+ 128 ( e^{-2n\lambda\E\left[\sigma_{\tau}^2 m_{\tau} ^2 \right] }\E\left[|\theta_0-\theta^* |^2\right]+dC_{\mathsf{SGLD},1} /\beta+\lambda C_{\mathsf{SGLD},2}+ |\theta^*|^2)$.
\end{lem}

\subsection{Proof of the Main Result in the Motivating Example} \label{appendix_proof_main_result_motivating_example}

\begin{proof}[Proof of Theorem \ref{main_theorem_toy_example}]
	We derive the non-asymptotic estimate for $W_2(\mathcal{L}(\widehat{Y}_{K+1}^{\text{EM}}),\pi_{\mathsf{D}})$ using the splitting
	\begin{equation} \label{upper_bound_wasserstein_proof}
		\begin{split}
			W_2(\mathcal{L}(Y_{K+1}^{\text{EM}}),\pi_{\mathsf{D}}) & \le
			W_2(\pi_{\mathsf{D}}, \mathcal{L}(Y_{t_{K+1}}))+W_2(\mathcal{L}(Y_{t_{K+1}}), \mathcal{L}(\widetilde{Y}_{t_{K+1}})) \\ & \quad + W_2( \mathcal{L}(\widetilde{Y}_{t_{K+1}}), \mathcal{L}(Y_{t_{K+1}}^{\text{aux}})) +W_2(\mathcal{L}(Y_{t_{K+1}}^{\text{aux}}), \mathcal{L}(Y_{K+1}^{\text{EM}})).
		\end{split}
	\end{equation}
	Since $t_{K+1} = T$, we have  $	W_2(\pi_{\mathsf{D}}, \mathcal{L}(Y_{t_{K+1}}))=0$. We provide upper bounds on the error made by approximating the initial condition of the backward process $Y_0 \sim \mathcal{L}(X_T)$ with $\widetilde{Y}_0 \sim \pi_{\infty} $, i.e.   $W_2(\mathcal{L}(Y_{t_{K+1}}), \mathcal{L}(\widetilde{Y}_{t_{K+1}}))$,  the error made by approximating the score function with $s$, i.e. $W_2( \mathcal{L}(\widetilde{Y}_{t_{K+1}}), \mathcal{L}(Y_{t_{K+1}}^{\text{aux}}))$,  and the discretisation error, i.e.  $W_2(\mathcal{L}(Y_{t_{K+1}}^{\text{aux}}), \mathcal{L}(Y_{K+1}^{\text{EM}}))$, separately.
	
	\paragraph{Upper bound on $W_2(\mathcal{L}(Y_{t_{K+1}}), \mathcal{L}(\widetilde{Y}_{t_{K+1}})).$} Applying It{\^o}'s formula and using \ref{eq:backwardproc_real_initial_condition_introduction} and \ref{Y_hat_auxiliary} with the score function given in  \ref{eqn:scoreexample}, we have, for any $t\in [0,T]$,
	\begin{equation} \label{Ito_formula_first_bound_example}
		\begin{split}
			\rmd |Y_t- \widetilde{Y}_t|^2
			&= 2 \langle Y_t- \widetilde{Y}_t, Y_t- \widetilde{Y}_t+2( \nabla \log p_{T-t} ( Y_t ) -  \nabla \log p_{T-t} ( \widetilde{Y}_t ) )\rangle\rmd t\\
			& = -2|Y_t- \widetilde{Y}_t|^2 \rmd  t.
		\end{split}
	\end{equation}
	Integrating and taking expectation both sides in \ref{Ito_formula_first_bound_example} yields
	\begin{equation} \label{first_bound_motivating_example_ytilde}
		\begin{split}
			\E[ |Y_t- \widetilde{Y}_t|^2 ]
			& = \E[ |Y_0- \widetilde{Y}_0|^2 ]-2\int_0^t\E[|Y_s- \widetilde{Y}_s|^2] \rmd s
			\\ & \leq e^{-2t}\E[ |Y_0- \widetilde{Y}_0|^2 ].
		\end{split}
	\end{equation}
	Using \ref{first_bound_motivating_example_ytilde}, the representation \ref{eq:OUdistribtuion} with $Z_T \overset{\text{d}}{=} \widetilde{Y}_0$ and $1-\sigma_t \le m_t$, we have
	\begin{equation} \label{first_bound_before_wasserstein_example}
		\begin{split}
			\E[ |Y_{t_{K+1}}- \widetilde{Y}_{t_{K+1}}|^2 ] & \leq e^{-2t_{K+1}}\E[ |Y_0- \widetilde{Y}_0|^2 ]
			\\ & = e^{-2t_{K+1}}  \mathbb{E}[|m_T X_0+ ( \sigma_T - 1) \widetilde{Y}_0 |^2]
			\\ & \leq 2 e^{-4t_{K+1}} \left(\E[ |X_0|^2] +  d  \right).
		\end{split}
	\end{equation}
	Using \ref{first_bound_before_wasserstein_example}, we have
	\begin{equation}\label{second_upper_bound}
		\begin{split}
			W_2(\mathcal{L}(Y_{t_{K+1}}), \mathcal{L}(\widetilde{Y}_{t_{K+1}})) & \leq	\sqrt{\E[ |Y_{t_{K+1}}- \widetilde{Y}_{t_{K+1}}|^2 ]}
			\\ &  \leq   \sqrt{2} e^{- 2 T} (\sqrt{\E\left[ |X_0|^2 \right]}+\sqrt{  d} ).
		\end{split}
	\end{equation}
	\paragraph{Upper bound on $W_2( \mathcal{L}(\widetilde{Y}_{t_{K+1}}), \mathcal{L}(Y_{t_{K+1}}^{\text{aux}})).$}
	Applying It{\^o}'s formula and using the process \ref{Y_hat_auxiliary} with the score function \ref{eqn:scoreexample} and the process \ref{Y_auxiliary_theta_hat} with the approximating function \ref{eqn:appscoreexample}, we have, for any $t \in [0,T]$,
	\begin{equation} \label{Ito_formula_second_bound_example}
		\begin{split}
			\rmd | \widetilde{Y}_t - Y_t^{\text{aux}}|^2
			& = 2\langle  \widetilde{Y}_t-Y_t^{\text{aux}},  \widetilde{Y}_t-Y_t^{\text{aux}}+2(   \nabla \log p_{T-t} ( \widetilde{Y}_t )-s(T-t, \hat{\theta}, Y_t^{\text{aux}} ) )\rangle\rmd t
			\\ & = - 2 | \widetilde{Y}_t-Y_t^{\text{aux}}|^2 \rmd t + 4 \langle \widetilde{Y}_t-Y_t^{\text{aux}} , m_{T-t} (\mu - \hat{\theta}) \rangle\rmd t.
		\end{split}
	\end{equation}
	Integrating and taking expectation on both sides in \ref{Ito_formula_second_bound_example} and using that the minimiser $\theta^{\star}=\mu$, we have
	\begin{equation} \label{Ito_formula_third_bound_example}
		\begin{split}
			\E[ | \widetilde{Y}_t - Y_t^{\text{aux}}|^2 ]
			&=  -2\int_0^t\E[|\widetilde{Y}_s - Y_s^{\text{aux}}|^2] \rmd s +4\int_0^t\E[\langle\widetilde{Y}_s - Y_s^{\text{aux}}, m_{T-s}(\theta^*-\hat{\theta})\rangle] \rmd s.
		\end{split}
	\end{equation}
	Differentiating both sides of \ref{Ito_formula_third_bound_example} and using Young's inequality, we obtain
	\[
	\begin{aligned}
		\frac{\rmd}{\rmd t}\E[ | \widetilde{Y}_t - Y_t^{\text{aux}}|^2 ]
		&=  -2\E[|\widetilde{Y}_t - Y_t^{\text{aux}}|^2] +4\E[ \langle \widetilde{Y}_t - Y_t^{\text{aux}}, m_{T-t}(\theta^*-\hat{\theta}) \rangle]\\
		&\leq -\E[|\widetilde{Y}_t - Y_t^{\text{aux}}|^2] +4e^{-2T}e^{2t}\E[|\theta^*-\hat{\theta}|^2],
	\end{aligned}
	\]
	which implies that
	\[
	\frac{\rmd}{\rmd t}(e^t\E[ | \widetilde{Y}_t - Y_t^{\text{aux}}|^2 ])\leq 4e^{-2T}e^{3t}\E[|\theta^*-\hat{\theta}|^2].
	\]
	Integrating both sides and using Lemma \ref{lem:sgld2ndmmtexample} yields
	\begin{equation} \label{second_bound_before_Wasserstein_example}
		\begin{split}
			&\E[ | \widetilde{Y}_t - Y_t^{\text{aux}}|^2 ]\\
			&\leq (4/3)e^{-2T}(e^{2t}-e^{-t})\E[|\theta^*-\hat{\theta}|^2]\\
			&\leq (4/3)e^{-2(T-t)} (e^{-2n\lambda\E[\sigma_{\tau}^2 m_{\tau} ^2 ] }\E[|\theta_0-\theta^* |^2]+ dC_{\mathsf{SGLD},1} /\beta+\lambda C_{\mathsf{SGLD},2}).
		\end{split}
	\end{equation}
	Using \ref{second_bound_before_Wasserstein_example}, we have
	\begin{equation}\label{third_upper_bound}
		\begin{split}
			W_2( \mathcal{L}(\widetilde{Y}_{t_{K+1}}), \mathcal{L}(Y_{t_{K+1}}^{\text{aux}})) & \leq \sqrt{\E[ | \widetilde{Y}_{t_{K+1}} - Y_{t_{K+1}}^{\text{aux}}|^2 ]} \\
			& \leq \sqrt{4/3} (e^{-2n\lambda\E[\sigma_{\tau}^2 m_{\tau} ^2 ] }\E[|\theta_0-\theta^* |^2]+dC_{\mathsf{SGLD},1} /\beta+\lambda C_{\mathsf{SGLD},2})^{1/2}.
		\end{split}
	\end{equation}
	\paragraph{Upper bound on $W_2(\mathcal{L}(Y_{t_{K+1}}^{\text{aux}}),  \mathcal{L}(\widehat{Y}_{K+1}^{\text{EM}})).$}	Applying It{\^o}'s formula  and using the processes \ref{Y_auxiliary_theta_hat} and \ref{continuous_time_EM_version} with the approximating function $s$ given in \ref{eqn:appscoreexample}, we have, for any $t\in [0,T]$,
	\begin{equation}  \label{Ito_formula_fourth_bound_example} 
		\begin{split}
			\rmd |   Y_t^{\text{aux}} -\widehat{Y}_t^{\text{EM}}|^2
			&= 2 \langle Y_t^{\text{aux}} -\widehat{Y}_t^{\text{EM}}, Y_t^{\text{aux}}- \widehat{Y}_{\lfloor t/\gamma \rfloor \gamma }^{\text{EM}} \rangle\rmd t\\
			&\quad +4 \langle Y_t^{\text{aux}} -\widehat{Y}_t^{\text{EM}},  s(T-t, \hat{\theta}, Y_t^{\text{aux}} ) - s(T-\lfloor t/\gamma \rfloor \gamma, \hat{\theta} , \widehat{Y}_{\lfloor t/\gamma \rfloor \gamma}^{\text{EM}}) \rangle\rmd t\\
			&=-2 \langle Y_t^{\text{aux}} -\widehat{Y}_t^{\text{EM}}, Y_t^{\text{aux}}- \widehat{Y}_{\lfloor t/\gamma \rfloor \gamma }^{\text{EM}}  \rangle\rmd t\\
			&\quad +4 \langle Y_t^{\text{aux}} -\widehat{Y}_t^{\text{EM}},  (m_{T-t}- m_{T-\lfloor t/\gamma \rfloor \gamma})\hat{\theta}  \rangle\rmd t.
		\end{split}
	\end{equation}
	Integrating both sides and taking expectation in \ref{Ito_formula_fourth_bound_example} yields
	\begin{equation} \label{Ito_formula_fourth_bound_example_second_part}
		\begin{split}
			\E\left[ |Y_t^{\text{aux}} -\widehat{Y}_t^{\text{EM}}|^2 \right]
			&=- 2\int_0^t\E\left[|Y_s^{\text{aux}} -\widehat{Y}_s^{\text{EM}}|^2\right] \rmd s - 2\int_0^t\E\left[\left\langle Y_s^{\text{aux}} -\widehat{Y}_s^{\text{EM}},  \widehat{Y}_s^{\text{EM}}- \widehat{Y}_{\lfloor s/\gamma \rfloor \gamma }^{\text{EM}} \right\rangle\right] \rmd s \\
			&\quad +4\int_0^t\E\left[\left\langle Y_s^{\text{aux}} -\widehat{Y}_s^{\text{EM}}, (m_{T-s}- m_{T-\lfloor s/\gamma \rfloor \gamma})\hat{\theta} \right\rangle\right] \rmd s.
		\end{split}
	\end{equation}
	Differentiating both sides in \ref{Ito_formula_fourth_bound_example_second_part}, using Young's inequality and $m_{T-t}- m_{T-\lfloor t/\gamma \rfloor \gamma}\leq \gamma m_{T-t}$, we have
\begin{equation}\label{eqn:w1con4thub}
		\begin{split}
			\frac{\rmd}{\rmd t}\E[ |Y_t^{\text{aux}} -\widehat{Y}_t^{\text{EM}}|^2 ]
			&=- 2\E[|Y_t^{\text{aux}} -\widehat{Y}_t^{\text{EM}}|^2] -  2\E\left[\left\langle Y_t^{\text{aux}} -\widehat{Y}_t^{\text{EM}},  \sqrt{2 } \int_{\lfloor t/\gamma \rfloor \gamma}^t\rmd\overline{B}_s \right\rangle\right]\\
			&\quad-  2\E\left[\left\langle Y_t^{\text{aux}} -\widehat{Y}_t^{\text{EM}},  \int_{\lfloor t/\gamma \rfloor \gamma}^t  (-\widehat{Y}_{\lfloor s/\gamma \rfloor \gamma }^{\text{EM}} + 2  m_{T-\lfloor s/\gamma \rfloor \gamma} \hat{\theta}) \rmd s \right\rangle\right] \\
			&\quad +4\E[\langle Y_t^{\text{aux}} -\widehat{Y}_t^{\text{EM}}, (m_{T-t}- m_{T-\lfloor t/\gamma \rfloor \gamma})\hat{\theta} \rangle] \\
			&\leq -\E[|Y_t^{\text{aux}} -\widehat{Y}_t^{\text{EM}}|^2] -  2\E\left[\left\langle Y_t^{\text{aux}} -\widehat{Y}_t^{\text{EM}},  \sqrt{2 } \int_{\lfloor t/\gamma \rfloor \gamma}^t\rmd\overline{B}_s \right\rangle\right]\\
			&\quad +2\gamma^2\E[ (-\widehat{Y}_{\lfloor t/\gamma \rfloor \gamma }^{\text{EM}} + 2  m_{T-\lfloor t/\gamma \rfloor \gamma} \hat{\theta})^2] +8\gamma^2e^{-2(T-t)} \E[|\hat{\theta} |^2]\\
			&\leq -\E[|Y_t^{\text{aux}} -\widehat{Y}_t^{\text{EM}}|^2] -  2\E\left[\left\langle Y_t^{\text{aux}} -\widehat{Y}_t^{\text{EM}},  \sqrt{2 } \int_{\lfloor t/\gamma \rfloor \gamma}^t\rmd\overline{B}_s \right\rangle\right]\\
			&\quad +4\gamma^2\E[ |\widehat{Y}_{\lfloor t/\gamma \rfloor \gamma }^{\text{EM}} |^2] +24\gamma^2  \E[|\hat{\theta} |^2].
		\end{split}
	\end{equation}
	We derive an upper bound for the second term on the right-hand side in \ref{eqn:w1con4thub}. Using Cauchy-Schwarz inequality, It\^o formula applied to $tB_t$ and Young's inequality, we have
	\begin{equation*} 
    \begin{split}
		&-  2\E\left[\left\langle Y_t^{\text{aux}} -\widehat{Y}_t^{\text{EM}},  \sqrt{2 } \int_{\lfloor t/\gamma \rfloor \gamma}^t\rmd\overline{B}_s \right\rangle\right]\\
		&=-  2\E\left[\left\langle (Y_t^{\text{aux}}-Y_{\lfloor t/\gamma \rfloor \gamma}^{\text{aux}}) -(\widehat{Y}_t^{\text{EM}}-\widehat{Y}_{\lfloor t/\gamma \rfloor \gamma}^{\text{EM}}),  \sqrt{2 } \int_{\lfloor t/\gamma \rfloor \gamma}^t\rmd\overline{B}_s \right\rangle\right] \\
		&=-  2\E\left[\left\langle \int_{\lfloor t/\gamma \rfloor \gamma}^t (-Y_s^{\text{aux}} +2 m_{T-s} \hat{\theta})\rmd s,  \sqrt{2 } \int_{\lfloor t/\gamma \rfloor \gamma}^t\rmd\overline{B}_s \right\rangle\right]
        \\
		&\quad + 2\E\left[\left\langle  \int_{\lfloor t/\gamma \rfloor \gamma}^t  (-\widehat{Y}_{\lfloor s/\gamma \rfloor \gamma }^{\text{EM}} + 2  m_{T-\lfloor s/\gamma \rfloor \gamma} \hat{\theta}) \rmd s ,  \sqrt{2 } \int_{\lfloor t/\gamma \rfloor \gamma}^t\rmd\overline{B}_s \right\rangle\right]
		 \\ &=2\E\left[\left\langle \int_{\lfloor t/\gamma \rfloor \gamma}^t ( Y_s^{\text{aux}} -\widehat{Y}_{\lfloor s/\gamma \rfloor \gamma }^{\text{EM}})\rmd s,  \sqrt{2 } \int_{\lfloor t/\gamma \rfloor \gamma}^t\rmd\overline{B}_s \right\rangle\right]
         \\
		&=2\E\left[\left\langle \int_{\lfloor t/\gamma \rfloor \gamma}^t ( Y_s^{\text{aux}} - Y_{\lfloor s/\gamma \rfloor \gamma }^{\text{aux}})\rmd s,  \sqrt{2 } \int_{\lfloor t/\gamma \rfloor \gamma}^t\rmd\overline{B}_s \right\rangle\right]	
		\\	&=2\E\left[\left\langle \int_{\lfloor t/\gamma \rfloor \gamma}^t\int_{\lfloor s/\gamma \rfloor \gamma}^s (- Y_{\nu}^{\text{aux}} )\rmd\nu \rmd s,  \sqrt{2 } \int_{\lfloor t/\gamma \rfloor \gamma}^t\rmd\overline{B}_s \right\rangle\right]
        \\
		&\quad +2\E\left[\left\langle \int_{\lfloor t/\gamma \rfloor \gamma}^t \sqrt{2 } \int_{\lfloor s/\gamma \rfloor \gamma}^s \rmd\overline{B}_\nu \rmd s,  \sqrt{2 } \int_{\lfloor t/\gamma \rfloor \gamma}^t\rmd\overline{B}_s \right\rangle\right]
        \end{split}
	\end{equation*}
	\begin{equation} \label{eqn:w1con4thub2}
    \begin{split}
		&\leq 2\sqrt{2}\left(\E\left[\left|\int_{\lfloor t/\gamma \rfloor \gamma}^t\int_{\lfloor s/\gamma \rfloor \gamma}^s (- Y_{\nu}^{\text{aux}} )\rmd\nu \rmd s\right|^2\right]\right)^{1/2}\left(\E\left[\left|\int_{\lfloor t/\gamma \rfloor \gamma}^t\rmd\overline{B}_s \right|^2\right]\right)^{1/2}
		\\	& +4\E\left[\left\langle  t\int_{\lfloor t/\gamma \rfloor \gamma}^t \rmd\overline{B}_s -\int_{\lfloor t/\gamma \rfloor \gamma}^t s \rmd\overline{B}_s,    \int_{\lfloor t/\gamma \rfloor \gamma}^t\rmd\overline{B}_s \right\rangle\right]
		\\	&\leq  \sqrt{2 } \gamma^{5/2}\left(\sup_{t\geq 0}\E\left[| Y_t^{\text{aux}}|^2\right]+d\right)+2d\gamma^2.
        \end{split}
	\end{equation}
	Substituting \ref{eqn:w1con4thub2} into \ref{eqn:w1con4thub} yields
	\[
	\begin{aligned}
		\frac{\rmd}{\rmd t}\E[ |Y_t^{\text{aux}} -\widehat{Y}_t^{\text{EM}}|^2 ]
		&\leq -\E[|Y_t^{\text{aux}} -\widehat{Y}_t^{\text{EM}}|^2]+\sqrt{2} \gamma^{5/2}\sup_{t\geq 0}\E[| Y_t^{\text{aux}}|^2]\\
		&\quad + \sqrt{2 }d \gamma^{5/2} +2d\gamma^2 +4\gamma^2\E[ |\widehat{Y}_{\lfloor t/\gamma \rfloor \gamma }^{\text{EM}} |^2] +24\gamma^2  \E[|\hat{\theta} |^2].
	\end{aligned}
	\]
	Thus,
	\begin{equation} \label{fourth_bound_integrating_factor_example}
		\frac{\rmd}{\rmd t}(e^t\E[ |Y_t^{\text{aux}} -\widehat{Y}_t^{\text{EM}}|^2 ])
		\leq  e^t\gamma^2\left( \sqrt{2} \sup_{t\geq 0}\E[| Y_t^{\text{aux}}|^2]+3d +4\E[ |\widehat{Y}_{\lfloor t/\gamma \rfloor \gamma }^{\text{EM}} |^2] +24 \E[|\hat{\theta} |^2]\right).
	\end{equation}
	Integrating both sides in \ref{fourth_bound_integrating_factor_example} and using Lemma \ref{lem:auxproc2ndbdexample}, Lemma \ref{lem:EMal2ndbdexample} 	  and Corollary \ref{lem:sgld2ndmmtcorexample}, yield
	\begin{equation} \label{third_bound_before_wasserstein_example}
		\begin{split}
			&\E[ |Y_t^{\text{aux}} -\widehat{Y}_t^{\text{EM}}|^2 ]\\
			&\leq  \sqrt{2}\gamma^2 (2d+(8/3)(e^{-2n\lambda\E[\sigma_{\tau}^2 m_{\tau} ^2 ] }\E[|\theta_0-\theta^* |^2]+dC_{\mathsf{SGLD},1} /\beta+\lambda C_{\mathsf{SGLD},2}+|\theta^*|^2))\\
			&\quad +3d\gamma^2 +4\gamma^2(3d+ 20 ( e^{-2n\lambda\E[\sigma_{\tau}^2 m_{\tau} ^2 ] }\E[|\theta_0-\theta^* |^2]+ dC_{\mathsf{SGLD},1} /\beta+\lambda C_{\mathsf{SGLD},2}+ |\theta^*|^2))\\
			&\quad+ 48\gamma^2( e^{-2n\lambda\E[\sigma_{\tau}^2 m_{\tau} ^2 ] }\E[|\theta_0-\theta^* |^2]+ dC_{\mathsf{SGLD},1} /\beta+ \lambda C_{\mathsf{SGLD},2}+ |\theta^*|^2)\\
			&\leq 18d\gamma^2+132\gamma^2  ( e^{-2n\lambda\E[\sigma_{\tau}^2 m_{\tau} ^2 ] }\E[|\theta_0-\theta^* |^2]+ dC_{\mathsf{SGLD},1} /\beta+\lambda C_{\mathsf{SGLD},2}+ |\theta^*|^2).
		\end{split}
	\end{equation}
	By using \ref{third_bound_before_wasserstein_example}, we have
	\begin{equation}\label{fourth_upper_bound}
		\begin{aligned}
			&W_2(\mathcal{L}(Y_{t_{K+1}}^{\text{aux}}), \mathcal{L}(\widehat{Y}_{K+1}^{\text{EM}}))\\
			& \leq \sqrt{ \E[ |Y_{t_{K+1}}^{\text{aux}} -\widehat{Y}_{t_{K+1}}^{\text{EM}}|^2 ] } \\
			& \leq \gamma (18 d+132  ( e^{-2n\lambda\E[\sigma_{\tau}^2 m_{\tau} ^2 ] }\E[|\theta_0-\theta^* |^2]+ dC_{\mathsf{SGLD},1} /\beta+\lambda C_{\mathsf{SGLD},2}+ |\theta^*|^2))^{1/2}.
		\end{aligned}
	\end{equation}
	\paragraph{Final upper bound on $W_2(\mathcal{L}(\widehat{Y}_{K+1}^{\text{EM}}),\pi_{\mathsf{D}}).$} Substituting \ref{second_upper_bound}, \ref{third_upper_bound}, \ref{fourth_upper_bound} into \ref{upper_bound_wasserstein_proof} yields
	\begin{equation} \label{final_upper_bound_example_before_choosing_parameters}
		\begin{split}
			&W_2(\mathcal{L}(\widehat{Y}_{K+1}^{\text{EM}}),\pi_{\mathsf{D}})
			\\ & \leq  \sqrt{2} e^{- 2 T} (\sqrt{\E\left[ |X_0|^2 \right]}+\sqrt{  d} ) \\
			&\quad + \sqrt{4/3} (e^{-2n\lambda\E[\sigma_{\tau}^2 m_{\tau} ^2 ] }\E[|\theta_0-\theta^* |^2]+dC_{\mathsf{SGLD},1} /\beta+\lambda C_{\mathsf{SGLD},2})^{1/2} \\
			&\quad +\gamma (18 d+132 ( e^{-2n\lambda\E[\sigma_{\tau}^2 m_{\tau} ^2 ] }\E[|\theta_0-\theta^* |^2]+ dC_{\mathsf{SGLD},1} /\beta+\lambda C_{\mathsf{SGLD},2}+ |\theta^*|^2))^{1/2}\\
			& \le  \sqrt{2} e^{- 2 T} (\sqrt{\E\left[ |X_0|^2 \right]}+\sqrt{  d} ) \\
			&\quad +(\sqrt{4/3}+2 \sqrt{33}) (e^{-2n\lambda\E[\sigma_{\tau}^2 m_{\tau} ^2 ] }\E[|\theta_0-\theta^* |^2]+ dC_{\mathsf{SGLD},1} /\beta+\lambda C_{\mathsf{SGLD},2})^{1/2}\\
			&\quad +\gamma (18 d+132|\theta^*|^2 )^{1/2}.
		\end{split}
	\end{equation}
	The bound for $W_2(\mathcal{L}(\widehat{Y}_{K+1}^{\text{EM}}),\pi_{\mathsf{D}}) $ in  \ref{final_upper_bound_example_before_choosing_parameters} can be made arbitrarily small by appropriately choosing parameters including $T,\beta,\lambda,n$ and $\gamma$. More precisely, for any $\delta>0$,  we first choose $T > T_{\delta}$ with $T_{\delta}$ given explicitly in Table \ref{tab:convconst}  such that
	\begin{equation} \label{second_delta_4_example}
		\sqrt{2} e^{- 2 T} (\sqrt{\E\left[ |X_0|^2 \right]}+\sqrt{  d} )  <\delta/4.
	\end{equation}
	Next, we choose $ \beta \geq \beta_{\delta}$ and $	0 <\lambda \le \lambda_{\delta}$ with $\beta_{\delta}$ and $ \lambda_{\delta}$ given explicitly in Table \ref{tab:convconst},
	and, in the case where $\lambda= \lambda_{\delta}$, we choose
	$n\geq n_{\delta}$ with $n_{\delta}$ given explicitly in Table \ref{tab:convconst}
	such that
	\begin{equation} \label{third_delta_4_example}
		\begin{split}
			&(\sqrt{4/3}+ 2 \sqrt{33})
			(e^{-2n\lambda\E[\sigma_{\tau}^2 m_{\tau}^2 ] }\E[|\theta_0-\theta^* |^2]+ dC_{\mathsf{SGLD},1} /\beta+\lambda C_{\mathsf{SGLD},2})^{1/2}\\
			&\leq (\sqrt{4/3} +2 \sqrt{33}) \sqrt{d /(\beta \E[\sigma_{\tau}^2 m_{\tau}^2])}\\
			&\quad +(\sqrt{4/3}+ 2 \sqrt{33}) (4 \lambda \E[ \sigma_{\tau}^4 m_{\tau}^2(\sigma_{\tau}^{-1}|Z|+ m_{\tau}  |X_0| +\sigma_{\tau} |Z|+m_{\tau}  |\theta^*| )^2 ]/ (\E[\sigma_{\tau}^2 m_{\tau}^2] ))^{1/2}\\
			&\quad +(\sqrt{4/3}+2 \sqrt{33})  e^{-n\lambda\E[\sigma_{\tau}^2m_{\tau} ^2]} \sqrt{ \E[|\theta_0-\theta^* |^2]}\\
			&\leq \delta/12+\delta/12+\delta/12 = \delta/4.
		\end{split}
	\end{equation}
	Finally, we choose $ 	0 < \gamma < \gamma_{\delta}$  with $\gamma_{\delta}$ given explicitly in Table \ref{tab:convconst} such that
	\begin{equation} \label{fourth_delta_4_example}
		\gamma (18 d+132|\theta^*|^2 )^{1/2}<\delta/4.
	\end{equation}
	Using \ref{second_delta_4_example}, \ref{third_delta_4_example} and \ref{fourth_delta_4_example} in \ref{final_upper_bound_example_before_choosing_parameters}, we obtain $ W_2(\mathcal{L}(\widehat{Y}_{K+1}^{\text{EM}}),\pi_{\mathsf{D}}) <\delta$.
\end{proof}

\subsection{Proof of the Preliminary Results} \label{proof_preliminary_estimates_appendix_motivating_example}

\noindent We provide the proofs of the results of Appendix \ref{preliminary_estimates_appendix_motivating_example}.

\begin{proof} [Proof of Lemma \ref{lem:sgld2ndmmtexample}] \label{proof:sgld2ndmmtexample}
	By \ref{eqn:SGLDexample}, we have for any  $n\in\N_0$,
	\begin{equation} \label{distance_SGLD_minimiser_first_proof_preliminary_results_appendix}
		\begin{split}
			|\theta_{n+1}^{\lambda}-\theta^*|^2
			&= \left| \theta_{n}^{\lambda} -\theta^*- \lambda H(\theta_n^{\lambda},\mathbf{X}_{n+1}) + \sqrt{2 \lambda /\beta} \ \xi_{n+1}\right|^2
			\\ &= \left| \theta_{n}^{\lambda} -\theta^*\right|^2+2\left\langle \theta_{n}^{\lambda} -\theta^*, - \lambda H(\theta_n^{\lambda},\mathbf{X}_{n+1}) + \sqrt{2 \lambda /\beta} \ \xi_{n+1} \right\rangle
			\\	& \quad  + \left|   - \lambda H(\theta_n^{\lambda},\mathbf{X}_{n+1}) + \sqrt{2 \lambda /\beta}  \xi_{n+1} \right|^2
			\\	&= \left| \theta_{n}^{\lambda} -\theta^*\right|^2-2\lambda \left\langle \theta_{n}^{\lambda} -\theta^*, H(\theta_n^{\lambda},\mathbf{X}_{n+1}) -H(\theta^*,\mathbf{X}_{n+1}) \right\rangle
			\\	&  \quad -2\lambda \left\langle \theta_{n}^{\lambda} -\theta^*, H(\theta^*,\mathbf{X}_{n+1}) \right\rangle+2\sqrt{2 \lambda /\beta} \left\langle \theta_{n}^{\lambda} -\theta^*,   \xi_{n+1} \right\rangle + \lambda^2\left| H(\theta_n^{\lambda},\mathbf{X}_{n+1}) \right|^2
			\\	& \quad -2\lambda\sqrt{2 \lambda /\beta}\left\langle H(\theta_n^{\lambda},\mathbf{X}_{n+1}) ,\xi_{n+1} \right\rangle+ (2 \lambda /\beta)\left| \xi_{n+1} \right|^2.
		\end{split}
	\end{equation}
	Taking conditional expectation on both sides in \ref{distance_SGLD_minimiser_first_proof_preliminary_results_appendix} yields
	\begin{equation*}
		\begin{split}
			\E\left[|\theta_{n+1}^{\lambda}-\theta^*|^2\mid \theta_n^{\lambda}\right]
			&= \left| \theta_{n}^{\lambda} -\theta^*\right|^2-2\lambda\E\left[ \left\langle \theta_{n}^{\lambda} -\theta^*, H(\theta_n^{\lambda},\mathbf{X}_{n+1}) -H(\theta^*,\mathbf{X}_{n+1}) \right\rangle\mid \theta_n^{\lambda}\right] \\
			&\quad  +2 \lambda^2\E\left[\left| H(\theta_n^{\lambda},\mathbf{X}_{n+1}) - H(\theta^*,\mathbf{X}_{n+1}) \right|^2\mid \theta_n^{\lambda}\right]\\
			&\quad +2 \lambda^2\E\left[\left|   H(\theta^*,\mathbf{X}_{n+1}) \right|^2 \mid \theta_n^{\lambda} \right]+2 \lambda d /\beta.
		\end{split}
	\end{equation*}
	Recall that $0< \lambda \leq \min\{\E[\sigma_{\tau}^2 m_{\tau} ^2]/(4\E[\sigma_{\tau}^4m_{\tau}^4]), 1/(2\E[\sigma_{\tau}^2m_{\tau}^2])\}$. Using
	Proposition \ref{Lipschitz_stochastic_gradient_motivating_example} and the stochastic gradient \ref{stochastic_gradient_motivating_example}, we have
	\begin{align*}
		\E\left[|\theta_{n+1}^{\lambda}-\theta^*|^2\mid \theta_n^{\lambda}\right]
		&\leq \left(1-2\lambda\E\left[\sigma_{\tau}^2 m_{\tau} ^2 \right] \right)|\theta_{n}^{\lambda} -\theta^* |^2+2 \lambda d /\beta\\
		&\quad  -2\lambda\E\left[\sigma_{\tau}^2 m_{\tau} ^2 \right] |\theta_{n}^{\lambda} -\theta^* |^2+8 \lambda^2\E\left[\sigma_{\tau}^4 m_{\tau} ^4 \right]\left| \theta_n^{\lambda}- \theta^* \right|^2\\
		&\quad +8 \lambda^2\E\left[ \sigma_{\tau}^4 m_{\tau}^2\left(\sigma_{\tau}^{-1}|Z|+ m_{\tau}  |X_0| +\sigma_{\tau} |Z|+m_{\tau}  |\theta^*| \right)^2 \right]
		\\	&\leq  \left(1-2\lambda\E\left[\sigma_{\tau}^2 m_{\tau} ^2 \right] \right)|\theta_{n}^{\lambda} -\theta^* |^2+2 \lambda d /\beta\\
		&\quad +8 \lambda^2\E\left[ \sigma_{\tau}^4 m_{\tau}^2\left(\sigma_{\tau}^{-1}|Z|+ m_{\tau}  |X_0| +\sigma_{\tau} |Z|+m_{\tau}  |\theta^*| \right)^2 \right].
	\end{align*}
	This implies that
	\[
	\E\left[|\theta_{n+1}^{\lambda}-\theta^*|^2 \right]\leq \left(1-2\lambda\E\left[\sigma_{\tau}^2 m_{\tau} ^2 \right] \right)^{n+1}\E\left[|\theta_0-\theta^* |^2\right]+dC_{\mathsf{SGLD},1} /\beta+\lambda C_{\mathsf{SGLD},2}.
	\]
\end{proof}

\begin{proof} [Proof of Lemma \ref{lem:auxproc2ndbdexample}] \label{proof:auxproc2ndbdexample}
	Applying It{\^o}'s formula and using the process \ref{Y_auxiliary_theta_hat} with $s$ given in \ref{eqn:appscoreexample}, we have, for any $t\in[0,T]$,
	\[
	\begin{aligned}
		\rmd |Y_t^{\text{aux}} |^2
		&=-2|  Y_t^{\text{aux}}|^2\rmd t+2 \langle  Y_t^{\text{aux}} , 2m_{T-t} \hat{\theta} \rangle \rmd t+2 \langle  Y_t^{\text{aux}} ,  \sqrt{2 }  \rmd \overline{B}_t \rangle +2d  \rmd t.
	\end{aligned}
	\]
	Integrating both sides and taking expectation, we have
	\[
	\begin{aligned}
		\E\left[|Y_t^{\text{aux}}|^2\right]
		&=  -2\int_0^t \E\left[|Y_s^{\text{aux}} |^2\right]\rmd s+2\int_0^t \E[\langle  Y_s^{\text{aux}} , 2m_{T-s} \hat{\theta} \rangle ]\rmd s  + \E\left[|Y_0^{\text{aux}}|^2\right]+2dt.
	\end{aligned}
	\]
	Then, differentiating both sides, we have
	\[
	\begin{aligned}
		\frac{\rmd}{\rmd t}\E\left[|Y_t^{\text{aux}}|^2\right]
		&=  -2  \E\left[|Y_t^{\text{aux}} |^2\right]  +2 \E[ \langle  Y_t^{\text{aux}} , 2m_{T-t} \hat{\theta} \rangle ]  + 2d\\
		&\leq - \E\left[|Y_t^{\text{aux}} |^2\right]+4m_{T-t}^2\E [| \hat{\theta} |^2 ]  + 2d,
	\end{aligned}
	\]
	which, by rearranging the terms, yields
	\[
	\begin{aligned}
		\frac{\rmd}{\rmd t} (e^t\E\left[|Y_t^{\text{aux}}|^2\right]  )
		&\leq  4e^{-2T}e^{3t}\E[| \hat{\theta} |^2 ]  + 2de^t.
	\end{aligned}
	\]
	Integrating both side and using Corollary \ref{lem:sgld2ndmmtcorexample}, we obtain
	\[
	\begin{aligned}
		\E\left[|Y_t^{\text{aux}}|^2\right]
		&\leq e^{-t} ( \E\left[|Y_0^{\text{aux}}|^2\right]+ (4/3)e^{-2T}(e^{3t}-1)\E [| \hat{\theta} |^2 ]  + 2d(e^t-1))\\
		&\leq d(2-e^{-t})+(4/3)e^{-2T}(e^{2t}-e^{-t})\E [| \hat{\theta} |^2 ]  \\
		&\leq 2d+e^{-2(T-t)}(8/3)(e^{-2n\lambda\E\left[\sigma_{\tau}^2 m_{\tau} ^2 \right] }\E [|\theta_0-\theta^* |^2 ]+dC_{\mathsf{SGLD},1} /\beta+\lambda C_{\mathsf{SGLD},2}+|\theta^*|^2).
	\end{aligned}
	\]
\end{proof}

\begin{proof}[Proof of Lemma \ref{lem:EMal2ndbdexample}]\label{proof:EMal2ndbdexample}
	Using the process \ref{eq:backwardprocemdisc_old} with the approximating function $s$ given in  \ref{eqn:appscoreexample}, we have
	\begin{equation} \label{Y_K_motivating_exaple_first_line_proof}
		\begin{split}
			|Y_{k+1}^{\text{EM}}|^2
			&=|	Y_{k}^{\text{EM}} +  \gamma (-Y_{k}^{\text{EM}} + 2m_{T-k} \hat{\theta})|^2+ 2  \gamma| \bar{Z}_{k+1}|^2\\
			&\quad +2\langle Y_{k}^{\text{EM}} +  \gamma (-Y_{k}^{\text{EM}} + 2m_{T-k} \hat{\theta}),  \sqrt{2  \gamma   }   \bar{Z}_{k+1} \rangle .
		\end{split}
	\end{equation}
	Taking conditional expectation on both sides in \ref{Y_K_motivating_exaple_first_line_proof}, using the independence of $Y_{k}^{\text{EM}}$ and $\bar{Z}_{k+1}$, Young's inequality and  $0< \gamma \leq 1/2$, we obtain
	\begin{equation*}
		\begin{split}
			\E \left[|Y_{k+1}^{\text{EM}}|^2\mid Y_{k}^{\text{EM}}\right]
			&=| Y_{k}^{\text{EM}}|^2+2\langle Y_{k}^{\text{EM}},  \gamma (-Y_{k}^{\text{EM}} + 2m_{T-t_{k}} \hat{\theta}) \rangle  \\
			&\quad +\gamma^2| Y_{k}^{\text{EM}}|^2+4\gamma^2m_{T-k}^2|\hat{\theta}|^2-2\gamma^2 \langle Y_{k}^{\text{EM}}, 2m_{T-t_{k}} \hat{\theta} \rangle+2  \gamma d\\
			&=(1-2\gamma)| Y_{k}^{\text{EM}}|^2+4\gamma^2m_{T-k}^2|\hat{\theta}|^2+2  \gamma d\\
			&\quad +\gamma^2| Y_{k}^{\text{EM}}|^2+2\gamma(1-\gamma) \langle Y_{k}^{\text{EM}},   2m_{T-k} \hat{\theta} \rangle  \\
			&\leq (1- \gamma)| Y_{k}^{\text{EM}}|^2+4\gamma(\gamma+2)m_{T-k}^2|\hat{\theta}|^2+2  \gamma d\\
			&\quad - \gamma| Y_{k}^{\text{EM}}|^2+\gamma^2| Y_{k}^{\text{EM}}|^2+ (\gamma/2)| Y_{k}^{\text{EM}}|^2  \\
			&\leq (1- \gamma)| Y_{k}^{\text{EM}}|^2+4\gamma(\gamma+2)|\hat{\theta}|^2+2  \gamma d.
		\end{split}
	\end{equation*}
	Thus,
	\begin{equation} \label{second_moment_EM_second_probability_space}
		\E \left[|Y_{k+1}^{\text{EM}}|^2 \right]\leq (1- \gamma)^{k+1}d+10|\hat{\theta}|^2+2   d.
	\end{equation}
	Using  Corollary \ref{lem:sgld2ndmmtcorexample} in \ref{second_moment_EM_second_probability_space} yields
	\[
	\E \left[|Y_{k}^{\text{EM}}|^2 \right]\leq (1- \gamma)^{k}d+2   d+20( e^{-2n\lambda\E\left[\sigma_{\tau}^2 m_{\tau} ^2 \right] }\E\left[|\theta_0-\theta^* |^2\right]+ dC_{\mathsf{SGLD},1} /\beta+\lambda C_{\mathsf{SGLD},2}+ |\theta^*|^2).
	\]
\end{proof}

\begin{proof}[Proof of Lemma \ref{lem:EMprocoseexample}] \label{proof:EMprocoseexample}
	Using the process \ref{continuous_time_EM_version} with the approximating function $s$ given in  \ref{eqn:appscoreexample}, Lemma \ref{lem:EMal2ndbdexample}, Corollary \ref{lem:sgld2ndmmtcorexample} and $\gamma \in (0,1/2]$, we have, for any $t\in[0,T]$,
	\[
	\begin{aligned}
		&\E\left[ |\widehat{Y}_t^{\text{EM}}  - \widehat{Y}_{\lfloor t/\gamma \rfloor \gamma}^{\text{EM}} |^2\right]  \\
		& =\E\left[ \left|\int_{\lfloor t/\gamma \rfloor \gamma}^t  (-\widehat{Y}_{\lfloor s/\gamma \rfloor \gamma }^{\text{EM}} + 2  m_{T-\lfloor s/\gamma \rfloor \gamma} \hat{\theta}) \rmd s + \sqrt{2 } \int_{\lfloor t/\gamma \rfloor \gamma}^t\rmd\overline{B}_s\right|^2\right]  \\
		&\leq  \gamma^2 \E\left[ \left|-\widehat{Y}_{\lfloor t/\gamma \rfloor \gamma }^{\text{EM}} + 2  m_{T-\lfloor t/\gamma \rfloor \gamma} \hat{\theta}\right|^2\right]+2d\gamma\\
		& \leq 2 \gamma^2 \E \left[ | \widehat{Y}_{\lfloor t/\gamma \rfloor \gamma}^{\text{EM}} |^2\right]  +8\gamma^2 \E\left[ | \hat{\theta}|^2\right]  +2d\gamma  \\
		&\leq 6\gamma^2d+40\gamma^2( e^{-2n\lambda\E\left[\sigma_{\tau}^2 m_{\tau} ^2 \right] }\E\left[|\theta_0-\theta^* |^2\right]+dC_{\mathsf{SGLD},1} /\beta+\lambda C_{\mathsf{SGLD},2}+ |\theta^*|^2)\\
		&\quad +16\gamma^2( e^{-2n\lambda\E\left[\sigma_{\tau}^2 m_{\tau} ^2 \right] }\E\left[|\theta_0-\theta^* |^2\right]+ dC_{\mathsf{SGLD},1} /\beta+\lambda C_{\mathsf{SGLD},2}+ |\theta^*|^2)+2d\gamma\\
		& \leq \gamma C_{\mathsf{EMose}},
	\end{aligned}
	\]
	where $C_{\mathsf{EMose}}= 8d+56  ( e^{-2n\lambda\E\left[\sigma_{\tau}^2 m_{\tau} ^2 \right] }\E\left[|\theta_0-\theta^* |^2\right]+ dC_{\mathsf{SGLD},1} /\beta+\lambda C_{\mathsf{SGLD},2}+ |\theta^*|^2)$.
\end{proof}

\begin{proof}[Proof of Lemma \ref{lem:EMproc2ndbdexample}] \label{proof:EMproc2ndbdexample}
	Applying It{\^o}'s formula to the process \ref{continuous_time_EM_version} with the approximating function $s$ given in \ref{eqn:appscoreexample}, we have, for any $t\in[0,T]$,
	\[
	\begin{aligned}
		\rmd |\widehat{Y}_t^{\text{EM}}|^2
		& =  2 \langle \widehat{Y}_t^{\text{EM}}, -\widehat{Y}_{\lfloor t/\gamma \rfloor \gamma }^{\text{EM}}   + 2  m_{T-\lfloor t/\gamma \rfloor \gamma} \hat{\theta} \rangle\rmd t   +2 \langle \widehat{Y}_t^{\text{EM}} ,  \sqrt{2 }  \rmd \overline{B}_t \rangle +2d  \rmd t
		\\ & = - 2 |\widehat{Y}_t^{\text{EM}}|^2 \rmd t +  2 \langle \widehat{Y}_t^{\text{EM}} , \widehat{Y}_t^{\text{EM}} -\widehat{Y}_{\lfloor t/\gamma \rfloor \gamma }^{\text{EM}}   \rangle\rmd t +  2 \langle \widehat{Y}_t^{\text{EM}} ,  2  m_{T-\lfloor t/\gamma \rfloor \gamma} \hat{\theta} \rangle\rmd t
		\\ & \quad +2 \langle \widehat{Y}_t^{\text{EM}} ,  \sqrt{2 }  \rmd \overline{B}_t \rangle +2d  \rmd t.
	\end{aligned}
	\]
	Integrating both sides and taking expectation, we have
	\[
	\begin{aligned}
		\E\left[ |\widehat{Y}_t^{\text{EM}}|^2\right]
		&=-2\int_0^t \E \left[| \widehat{Y}_s^{\text{EM}}|^2 \right]\rmd s +2\int_0^t \E \left[\langle \widehat{Y}_s^{\text{EM}}, \widehat{Y}_s^{\text{EM}}-\widehat{Y}_{\lfloor s/\gamma \rfloor \gamma }^{\text{EM}}   \rangle \right]\rmd s\\
		&\quad +2\int_0^t \E \left [ \langle \widehat{Y}_s^{\text{EM}},   2  m_{T-\lfloor s/\gamma \rfloor \gamma} \hat{\theta}\rangle \right]\rmd s+\E \left[ |\widehat{Y}_0^{\text{EM}}|^2 \right] +2d  t.
	\end{aligned}
	\]
	Then, differentiating both sides and using Young's inequality, yield
	\[
	\begin{aligned}
		\frac{\rmd}{\rmd t}\E \left[ |\widehat{Y}_t^{\text{EM}}|^2 \right]
		&=-2 \E \left[| \widehat{Y}_t^{\text{EM}}|^2 \right] +2\E \left[\langle \widehat{Y}_t^{\text{EM}}, \widehat{Y}_t^{\text{EM}}-\widehat{Y}_{\lfloor t/\gamma \rfloor \gamma }^{\text{EM}}   \rangle \right]\\
		&\quad +2\E \left[\langle \widehat{Y}_t^{\text{EM}},   2  m_{T-\lfloor t/\gamma \rfloor \gamma} \hat{\theta} \rangle \right] +2d \\
		&\leq -  \E \left[| \widehat{Y}_t^{\text{EM}}|^2  \right] +2\E \left[|\widehat{Y}_t^{\text{EM}}-\widehat{Y}_{\lfloor t/\gamma \rfloor \gamma }^{\text{EM}}   |^2 \right]+8\E [| \hat{\theta}|^2] +2d.
	\end{aligned}
	\]
	Using Lemma \ref{lem:EMprocoseexample}, we have
	\[
	\frac{\rmd}{\rmd t}\left(e^t\E \left[ |\widehat{Y}_t^{\text{EM}}|^2 \right] \right)\leq e^t (2 \gamma C_{\mathsf{EMose}}+8\E [| \hat{\theta}|^2 ]+2d ).
	\]
	Integrating both sides, using $\gamma \in (0,1/2]$ and Corollary \ref{lem:sgld2ndmmtcorexample}  yields
	\[
	\begin{aligned}
		\E \left[ |\widehat{Y}_t^{\text{EM}}|^2 \right]
		&\leq e^{-t}d+(1-e^{-t}) (2 \gamma C_{\mathsf{EMose}}+8 \E[| \hat{\theta}|^2]+2d )\\
		& \leq 18d+ 128 ( e^{-2n\lambda\E\left[\sigma_{\tau}^2 m_{\tau} ^2 \right] }\E\left[|\theta_0-\theta^* |^2\right]+dC_{\mathsf{SGLD},1} /\beta+\lambda C_{\mathsf{SGLD},2}+ |\theta^*|^2).
	\end{aligned}
	\]
\end{proof}

\section{Proofs of the Results in the General Case} \label{proofs_appendix_preliminary_results}

\noindent In this section, we provide the proof of Theorem \ref{main_theorem_general}. We start by introducing the results which will be used in the proof of Theorem \ref{main_theorem_general}.

\subsection{Preliminary Estimates for the General Case} \label{proofs_appendix_preliminary_results_general_case_statements}

\noindent Throughout this section, we fix $\epsilon \in (0,1)$. The following auxiliary results will be used in the proof of Theorem \ref{main_theorem_general} and their proofs are postponed to Appendix \ref{proof_preliminary_estimates_appendix_general_case}.

We provide an upper bound for the moments of $(\widehat{Y}_t^{\text{EM}})_{t \in [0,T]}$ defined in \ref{continuous_time_EM_version}.
\begin{lem}\label{lem:EMproc2ndbdgeneral}
	Let Assumptions \ref{general_assumption_algorithm} and \ref{Assumption_2} hold.
	For any $p \in [2, 4]$ and $t\in[0,T-\epsilon]$,
	\begin{equation*}
		\sup_{0\leq s\leq t} \mathbb{E} \left[|\widehat{Y}_s^{\text{EM}} |^p \right]\leq C_{\mathsf{EM},p}(t),
	\end{equation*}
	where
	\begin{equation*}
		\begin{split}
			C_{\mathsf{EM},p}(t)
			&	: =   e^{t(3p-1 - \frac{2}{p} + 2^{2p-1} \mathsf{K}^p_{\text{Total}} (1+T^{\alpha p}))}
			\\ & \quad \times   \left( \mathbb{E} \left[ |\widehat{Y}_0^{\text{EM}}|^p \right] +   2^{3p-2}\mathsf{K}^p_{\text{Total}} t (1+ \mathbb{E} [  |  \hat{\theta} |^p  ]) (1+T^{\alpha p})  + \frac{2}{p} (pM+ p(p-2))^{\frac{p}{2}} t  \right) .
		\end{split}
	\end{equation*}
\end{lem}
\noindent The following result provides an estimate for the one step error associated with $(\widehat{Y}_t^{\text{EM}})_{t \in [0,T]}$ defined in \ref{continuous_time_EM_version}.
\begin{lem}\label{lem:distance_EM_scheme}
	Let Assumptions \ref{general_assumption_algorithm} and \ref{Assumption_2} hold.
	For any $p \in [2, 4]$ and $t\in[0,T-\epsilon]$,
	\begin{equation*}
		\E\left[|\widehat{Y}_t^{\text{EM}}  - \widehat{Y}_{\lfloor t/\gamma \rfloor \gamma}^{\text{EM}}  |^p \right]\leq  \gamma^{\frac{p}{2}}  C_{\mathsf{EMose},p},
	\end{equation*}
	where
	\begin{equation*}
		\begin{split}
			C_{\mathsf{EMose},p} & :=   2^{p-1}  (C_{\mathsf{EM},p}(T) + \mathsf{K}^p_{\text{Total}} (1+  T^{\alpha p} )  (2^{3p-2}    C_{\mathsf{EM},p}(T) + 2^{4p-3} (1+ \E [|\hat{\theta}|^p])  ))
			\\ & \quad +   (M p(p-1))^{\frac{p}{2}}.
		\end{split}
	\end{equation*}
\end{lem}
\noindent The following result is a modification of \citet[Lemma 4.1]{kumar2019milstein}.
\begin{lem} \label{Improved_rate_drift_growth}
	Let Assumption \ref{Assumption_2} hold and let $b : [0,T] \times \mathbb{R}^d \times \mathbb{R}^M \rightarrow \mathbb{R}^M$ such that
	\begin{equation} \label{expression_b_drift_corrollary_lemma_appendix}
		b(t,\theta,x) := x+ 2 s(t,\theta,x).
	\end{equation}
	Then, for any  $x, \bar{x} \in \mathbb{R}^M$, $t \in [0,T]$, $\theta \in \mathbb{R}^d$, $ \alpha \in [\frac{1}{2},1]$, and $k=1, \dots M$,
	\begin{equation*}
		\left|  b^{(k)}(t,\theta,x) - b^{(k)}(t,\theta,\bar{x}) - \sum_{i=1}^M \frac{\partial b^{(k)}(t,\theta,\bar{x})}{\partial y^i} (x^i - \bar{x}^{i})   \right| \le  \mathsf{K}_4  (1+2 |t|^{\alpha})  | x - \bar{x}|^2.
	\end{equation*}
\end{lem}
\begin{lem} \label{growth_estimate_gradient_theta_neural_network}
	Let Assumption \ref{Assumption_2} hold and let $b$ be as in \ref{expression_b_drift_corrollary_lemma_appendix}. Then, one obtains that, for any $(t,\theta,x) \in [0,T] \times \mathbb{R}^{d} \times \mathbb{R}^{M}$ and $k=1, \dots M$,
	\begin{equation} \label{bound_gradient_s}
		| \nabla_{x}  b^{(k)}(t,\theta,x)|  \le 1+ 2 \mathsf{K}_3 (1+ 2 | t |^{\alpha} ).
	\end{equation}
	
\end{lem}

\subsection{Proof of the Main Result for the General Case} \label{appendix_proof_main_result_general_case}
\begin{proof}[Proof of Theorem \ref{main_theorem_general}]
	We proceed as in the proof of Theorem \ref{main_theorem_toy_example} using the splitting
	\begin{equation} \label{upper_bound_wasserstein_general}
		\begin{split}
			W_2(\mathcal{L}(Y_{K}^{\text{EM}}),\pi_{\mathsf{D}}) & \le
			W_2(\pi_{\mathsf{D}}, \mathcal{L}(Y_{t_{K}}))+W_2(\mathcal{L}(Y_{t_{K}}), \mathcal{L}(\widetilde{Y}_{t_{K}})) \\ & \quad + W_2( \mathcal{L}(\widetilde{Y}_{t_{K}}), \mathcal{L}(Y_{t_{K}}^{\text{aux}})) +W_2(\mathcal{L}(Y_{t_{K}}^{\text{aux}}), \mathcal{L}(Y_{K}^{\text{EM}})),
		\end{split}
	\end{equation}
	where the first term on the right-hand side in \ref{upper_bound_wasserstein_general} corresponds to the error made by the early stopping and the remaining terms have the same interpretation of the corresponding ones in \ref{upper_bound_wasserstein_proof}.
	
	\paragraph{Upper bound on $W_2(\pi_{\mathsf{D}}, \mathcal{L}(Y_{t_K})).$}
	For any $t \in [0,T]$, note that
	\begin{equation} \label{mean_standard_deviation_OU_inequality}
		1-m_t \leq \sigma_t,  \qquad 	\sigma_{t}^2 = 1-e^{-2  t}  \leq 2   t.
	\end{equation}
	Recall that $t_K = T- \epsilon$. Using the representation of the OU process \ref{eq:OUdistribtuion} and the inequalities  \ref{mean_standard_deviation_OU_inequality}, we have
	\begin{equation} \label{first_upper_bound_general}
		\begin{split}
			W_2(\pi_{\mathsf{D}}, \mathcal{L}(Y_{t_K})) & =		W_2(\pi_{\mathsf{D}}, \mathcal{L}(X_{T-t_K})) \\
			&\leq \sqrt{ \mathbb{E}\left[|X_0 - X_{T-t_K} |^2 \right]}
			\\ & \leq \sqrt{ 2 } \left[(1-m_{T-t_K}) \sqrt{  \mathbb{E}[|X_0  |^2]} +\sigma_{T-t_K} \sqrt{M} \right]
			\\ & \leq \sqrt{ 2 } \sigma_{T-t_K} (\sqrt{  \mathbb{E}[|X_0  |^2]} + \sqrt{M} )
			\\ & \leq 2 \sqrt{  \epsilon}  (\sqrt{  \mathbb{E}[|X_0  |^2]} + \sqrt{M} ).
		\end{split}
	\end{equation}
	
	\paragraph{Upper bound on $W_2(\mathcal{L}(Y_{t_K}), \mathcal{L}(\widetilde{Y}_{t_K})).$}
	Using It\^o's formula,  we have, for any $t \in [0,T-\epsilon]$,
	\begin{equation} \label{Ito_formula_inequality_second_bound_new_second}
		\begin{split}
			\text{d} |Y_t -  \widetilde{Y}_t|^2
			&= 2\langle Y_t -  \widetilde{Y}_t,  Y_t +2  \nabla \log p_{T-t}(Y_t) -  \widetilde{Y}_t -2  \nabla \log p_{T-t}( \widetilde{Y}_t) \rangle \ \text{d}  t\\
			& = 2  | Y_t -  \widetilde{Y}_t|^2 \ \text{d}  t +4   \langle Y_t -  \widetilde{Y}_t,  \nabla \log p_{T-t}(Y_t)   -\nabla \log p_{T-t}( \widetilde{Y}_t)\rangle \ \text{d}  t.
		\end{split}
	\end{equation}
	By integrating, taking expectations on both sides in \ref{Ito_formula_inequality_second_bound_new_second}, and using Remark \ref{remark_contraction_score} with the lower bound $\widehat{	L}_{\text{MO}}$ in the estimate \ref{contraction_score_assumption}, we have
	\begin{equation}\label{eq:backwardsdecontr_Lipschitz_new_second}
		\begin{split}
			\mathbb{E}[|Y_{t_K} - \widetilde{Y}_{t_K} |^2]  &\leq 	\mathbb{E}[|Y_0-\widetilde{Y}_0| ^2]
			+ \int_0^{t_K} 2 (1   - 2  \widehat{	L}_{\text{MO}}  )	 \ \mathbb{E}\left[ | Y_t - \widetilde{Y}_t |^2\right]\text{d}  t
			\\ &\leq 	\mathbb{E}[|Y_0- \widetilde{Y}_0 |^2]e^{    2 (1   - 2 \widehat{	L}_{\text{MO}}  )t_K	 }.
		\end{split}
	\end{equation}
	Using \ref{eq:backwardsdecontr_Lipschitz_new_second},  the representation \ref{eq:OUdistribtuion} with $Z_T \overset{\text{d}}{=} \widetilde{Y}_0$ and \ref{mean_standard_deviation_OU_inequality}, we have
	\begin{equation} \label{second_bound_difference_y_ytilde_new_second}
		\begin{split}
			\mathbb{E}[|Y_{t_K} - \widetilde{Y}_{t_K} |^2]
			& \leq		\mathbb{E}[|Y_0- \widetilde{Y}_0 |^2]e^{  2 (1   - 2 \widehat{	L}_{\text{MO}}  )t_K }	
			\\ & = 	\mathbb{E}[|m_T X_0+ ( \sigma_T - 1) \widetilde{Y}_0 |^2] \ e^{    2 (1   - 2 \widehat{	L}_{\text{MO}}  )t_K	 }
			\\ & \leq 2  \left( \mathbb{E}[|X_0|^2] + M \right) e^{     2 (1   - 2 \widehat{	L}_{\text{MO}}  )t_K	  -2T}.
		\end{split}
	\end{equation}
	Using \ref{second_bound_difference_y_ytilde_new_second}, we have
	\begin{equation} \label{final_second_bound_Wasserstein_general_new_second}
		\begin{split}
			W_2(\mathcal{L}(Y_{t_K}), \mathcal{L}(\widetilde{Y}_{t_K})) & \le \sqrt{	\mathbb{E} [ | Y_{t_K} - \widetilde{Y}_{t_K}  |^2] }
			\\ & \leq \sqrt{2}  ( 	\sqrt{\mathbb{E}[|X_0|^2]}+ \sqrt{M} ) e^{      - 2 \widehat{	L}_{\text{MO}} (T-\epsilon) -\epsilon}.
		\end{split}
	\end{equation}

	\paragraph{Upper bound on $W_2( \mathcal{L}(\widetilde{Y}_{t_K}), \mathcal{L}(Y_{t_K}^{\text{aux}})).$ }
	Using It{\^o}'s formula, we have, for $t \in [0,T-\epsilon]$,
	\begin{equation}\label{difference_aux_backward_process}
		\begin{split}
			\text{d} | \widetilde{Y}_t - Y_{t}^{\text{aux}} |^2 & = 2 \langle \widetilde{Y}_t - Y^{\text{aux}}_{t}, \widetilde{Y}_t + 2 \  \nabla \log p_{T - t}(\widetilde{Y}_t)  - Y^{\text{aux}}_{t}  -  2 \ s(T-t, \hat{\theta} , Y_t^{\text{aux}}) \rangle \  \text{d} t
			\\ & = 2 | \widetilde{Y}_t - Y^{\text{aux}}_{t}|^2 \ \text{d} t + 4 \ \langle \widetilde{Y}_t - Y^{\text{aux}}_{t},  \nabla \log p_{T - t}(\widetilde{Y}_t) -  \nabla \log p_{T - t}(Y_t^{\text{aux}})   \rangle \  \text{d} t
			\\ & \quad + 4 \ \langle \widetilde{Y}_t - Y^{\text{aux}}_{t},    \nabla \log p_{T - t}(Y_t^{\text{aux}}) -s(T-t, \hat{\theta} , Y_t^{\text{aux}})  \rangle \  \text{d} t.
		\end{split}
	\end{equation}
	By integrating and taking the expectation on both sides in \ref{difference_aux_backward_process}, and using the Remark \ref{remark_contraction_score} with the lower bound $\widehat{	L}_{\text{MO}}$ in the estimate \ref{contraction_score_assumption}, Young's inequality with $\zeta \in (0,1)$ and Assumption \ref{assumption_equivalence_global_minimiser_epsilon}, we have
	\begin{equation} \label{toward_third_error_bound_inequality}
		\begin{split}
			\mathbb{E} [  | \widetilde{Y}_{T-\epsilon} - Y_{T-\epsilon}^{\text{aux}} |^2 ] & = 2 \int_0^{T-\epsilon} 	\mathbb{E} [  | \widetilde{Y}_s - Y^{\text{aux}}_{s}|^2 ] \ \text{d} s
			\\ & \quad + 4 \int_0^{T-\epsilon} 	\mathbb{E} [   \langle \widetilde{Y}_s - Y^{\text{aux}}_{s},  \nabla \log p_{T - s}(\widetilde{Y}_s) -  \nabla \log p_{T - s}(Y_s^{\text{aux}})   \rangle ]  \  \text{d} s
			\\ & \quad + 4 \int_0^{T-\epsilon} 	\mathbb{E} [   \langle \widetilde{Y}_s - Y^{\text{aux}}_{s},    \nabla \log p_{T - s}(Y_s^{\text{aux}}) - s(T-s, \hat{\theta} , Y_s^{\text{aux}})      \rangle ]  \  \text{d} s
			\\	 & \le  \int_0^{T-\epsilon}  2(1+  \zeta  - 2 \widehat{	L}_{\text{MO}}) 	\mathbb{E} [  | \widetilde{Y}_s - Y^{\text{aux}}_{s}|^2 ] \ \text{d} s
			\\ & \quad  + 2  \zeta^{-1} \   \int_0^{T-\epsilon} 	\mathbb{E} [  | \nabla \log p_{T - s}(Y_s^{\text{aux}}) - s(T-s, \hat{\theta} , Y_s^{\text{aux}})  |^2 ]  \  \text{d} s
			\\ & \le  \int_0^{T-\epsilon} 2(1+  \zeta  - 2 \widehat{	L}_{\text{MO}})   \	\mathbb{E} [  | \widetilde{Y}_s - Y^{\text{aux}}_{s}|^2 ] \ \text{d} s +  2 \zeta^{-1} \varepsilon_{\text{SN}}
			\\ & \le 2e^{2(1+  \zeta - 2 \widehat{ L}_{\text{MO}} )(T-\epsilon) }   \zeta^{-1} \varepsilon_{\text{SN}}.
		\end{split}
	\end{equation}
	Using \ref{toward_third_error_bound_inequality} and $t_K = T- \epsilon$, we have
	\begin{equation} \label{final_third_upper_bound}
		\begin{split}
			W_2(\mathcal{L}(\widetilde{Y}_{t_K}),\mathcal{L}(Y^{\text{aux}}_{t_K}))  & \le \sqrt{	\mathbb{E} [ |\widetilde{Y}_{t_K} - Y_{t_K}^{\text{aux}} |^2] }
			\\ &  \le   \sqrt{2  \zeta^{-1}} e^{(1+  \zeta - 2 \widehat{ L}_{\text{MO}} )(T-\epsilon) }  \sqrt{   \varepsilon_{\text{SN}} }.
		\end{split}
	\end{equation}
	\paragraph{Upper bound on $W_2(\mathcal{L}(Y_{t_K}^{\text{aux}}), \mathcal{L}(Y_{K}^{\text{EM}})).$}
	The following bound is derived by modifying \citet[Lemma 4.7]{kumar2019milstein}. For the sake of presentation, let $b : [0,T] \times \mathbb{R}^d \times \mathbb{R}^M \rightarrow \mathbb{R}^M$ such that
	\begin{equation} \label{b_drift_continuous_interpolation_proof_general_case}
		b(t,\theta,x) = x+ 2 s(t,\theta,x).
	\end{equation}
	Consequently, $(\widehat{Y}_t^{\text{EM}})_{t \in [0,T]}$ can be expressed using \ref{b_drift_continuous_interpolation_proof_general_case} as
	\begin{equation} \label{Y_EM_continuous_interpolation_proof_general_case}
		\rmd \widehat{Y}_t^{\text{EM}} =  b(T-\lfloor t/\gamma \rfloor \gamma, \hat{\theta} , \widehat{Y}_{\lfloor t/\gamma \rfloor \gamma}^{\text{EM}}) \ \text{d}t + \sqrt{2 } \ \text{d}\bar{B}_t, \qquad \widehat{Y}_0^{\text{EM}} \sim \pi_{\infty} = \mathcal{N}(0,I_{M}).
	\end{equation}
	Using It{\^o}'s formula, we have, for $t \in [0,T-\epsilon]$,
	\begin{equation} \label{Ito_formula_improved_rate}
		\begin{split}
			&	\text{d}	|Y_t^{\text{aux}} - \widehat{Y}_t^{\text{EM}} |^2
			\\ &= 2 \langle Y_t^{\text{aux}} - \widehat{Y}_t^{\text{EM}} , Y_t^{\text{aux}} + 2 \ s(T-t, \hat{\theta},Y_t^{\text{aux}}  ) - \widehat{Y}^{\text{EM}}_{\lfloor t/ \gamma\rfloor \gamma } - 2 \ s(T- \lfloor t/ \gamma\rfloor \gamma , \hat{\theta}, \widehat{Y}^{\text{EM}}_{\lfloor t/ \gamma\rfloor \gamma }  ) \rangle \ 	\text{d}t
			\\ & = 2 | Y_t^{\text{aux}} - \widehat{Y}_t^{\text{EM}}|^2  \ 	\text{d}t  +  4 \langle Y_t^{\text{aux}} - \widehat{Y}_t^{\text{EM}} ,   s(T-t, \hat{\theta},Y_t^{\text{aux}}  ) -  s(T-t, \hat{\theta},\widehat{Y}_t^{\text{EM}}  ) \rangle \ 	\text{d}t
			\\ &  \quad + 4 \langle Y_t^{\text{aux}} - \widehat{Y}_t^{\text{EM}} ,   s(T-t, \hat{\theta}, \widehat{Y}^{\text{EM}}_{\lfloor t/ \gamma\rfloor \gamma } ) -  s(T- \lfloor t/ \gamma\rfloor \gamma, \hat{\theta},\widehat{Y}^{\text{EM}}_{\lfloor t/ \gamma\rfloor \gamma }  ) \rangle \ 	\text{d}t
			\\ & \quad + 2 \langle Y_t^{\text{aux}} - \widehat{Y}_t^{\text{EM}} ,   b(T-  t, \hat{\theta},\widehat{Y}_t^{\text{EM}} ) -  b(T- t, \hat{\theta},\widehat{Y}^{\text{EM}}_{\lfloor t/ \gamma\rfloor \gamma }  )  ) \rangle \ 	\text{d}t.
		\end{split}
	\end{equation}
	Integrating and taking the expectation on both sides in \ref{Ito_formula_improved_rate}, using Cauchy--Schwarz inequality, Young's inequality  with $\zeta \in (0,1)$, Assumption \ref{Assumption_2} and Remark \ref{Control_algorithm}, yield
	\begin{equation} \label{after_Ito_formula_improved_rate} 
		\begin{split}
			&	\E \left[|  Y_t^{\text{aux}} - \widehat{Y}_t^{\text{EM}} |^2\right]
			\\ & = 2 \int_0^t \E \left[ | Y_s^{\text{aux}} - \widehat{Y}_s^{\text{EM}}|^2 \right]  \ 	\text{d}s  +  4 \int_0^t \E \left[ \langle Y_s^{\text{aux}} - \widehat{Y}_s^{\text{EM}} ,   s(T-s, \hat{\theta},Y_s^{\text{aux}}  ) -  s(T-s, \hat{\theta},\widehat{Y}_s^{\text{EM}}  ) \rangle \right] \	\text{d}s
			\\ & \quad + 4 \int_0^t \E  \left[  \langle Y_s^{\text{aux}} - \widehat{Y}_s^{\text{EM}} ,   s(T-s, \hat{\theta}, \widehat{Y}^{\text{EM}}_{\lfloor s/ \gamma\rfloor \gamma } ) -  s(T- \lfloor s/ \gamma\rfloor \gamma, \hat{\theta},\widehat{Y}^{\text{EM}}_{\lfloor s/ \gamma\rfloor \gamma }  ) \rangle \right] \ 	\text{d}s
			\\ &  \quad + 2 \int_0^t \E  \left[  \langle Y_s^{\text{aux}} - \widehat{Y}_s^{\text{EM}} ,   b(T-s, \hat{\theta},\widehat{Y}_s^{\text{EM}} ) -  b(T-s, \hat{\theta},\widehat{Y}^{\text{EM}}_{\lfloor s/ \gamma\rfloor \gamma }  )  \rangle  \right] \  	\text{d}s
			\\ & \le 	 \int_0^t 2 (1+\zeta + 2  \mathsf{K}_3 (1+2 |T-s|^{\alpha})) \  \E\left[| Y_s^{\text{aux}} - \widehat{Y}_s^{\text{EM}}|^2\right]  \ \text{d}s
			\\ & \quad + 2 \zeta^{-1} \int_0^t \E \left[ | s(T- s , \hat{\theta} , \widehat{Y}^{\text{EM}}_{\lfloor s/ \gamma\rfloor \gamma } ) -  \ s(T- \lfloor s/ \gamma\rfloor \gamma  , \hat{\theta}, \widehat{Y}^{\text{EM}}_{\lfloor s/ \gamma\rfloor \gamma }  ) |^2 \right]    \ \text{d}s
			\\ & \quad + 2 \int_0^t \E  \left[  \langle Y_s^{\text{aux}} - \widehat{Y}_s^{\text{EM}} ,   b(T-s, \hat{\theta},\widehat{Y}_s^{\text{EM}} ) -  b(T- s, \hat{\theta},\widehat{Y}^{\text{EM}}_{\lfloor s/ \gamma\rfloor \gamma }  )  ) \rangle  \right] \  	\text{d}s
			\\	 & \le 	 \int_0^t 2(1+\zeta + 2 \mathsf{K}_3 (1+2 T^{\alpha})) \  \E\left[| Y_s^{\text{aux}} - \widehat{Y}_s^{\text{EM}}|^2\right]  \ \text{d}s
			\\ & \quad + 2 \zeta^{-1} \mathsf{K}_1^2 \  \E \left[(1+ 2  | \hat{\theta}|)^2\right] \int_0^t  |s - \lfloor s/ \gamma\rfloor \gamma |^{2 \alpha}   \ \text{d}s
			\\ & \quad + 2 \int_0^t \E  \left[  \langle Y_s^{\text{aux}} - \widehat{Y}_s^{\text{EM}} ,   b(T-s, \hat{\theta},\widehat{Y}_s^{\text{EM}} ) -  b(T-s, \hat{\theta},\widehat{Y}^{\text{EM}}_{\lfloor s/ \gamma\rfloor \gamma }  )  ) \rangle  \right] \  	\text{d}s
			\\ & \le 	 \int_0^t 2(1+\zeta + 2 \mathsf{K}_3 (1+2 T^{\alpha})) \  \E\left[| Y_s^{\text{aux}} - \widehat{Y}_s^{\text{EM}}|^2 \right]  \ \text{d}s + 4 \zeta^{-1} \mathsf{K}_1^2 (1+ 8(   \widetilde{\varepsilon}_{\text{AL}} +  | \theta^{*}|^2) ) \gamma^{2\alpha} t
			\\ & \quad + 2 \int_0^t \E  \left[  \langle Y_s^{\text{aux}} - \widehat{Y}_s^{\text{EM}} ,   b(T- s, \hat{\theta},\widehat{Y}_s^{\text{EM}} ) -  b(T- s, \hat{\theta},\widehat{Y}^{\text{EM}}_{\lfloor s/ \gamma\rfloor \gamma }  )   \rangle  \right] \  	\text{d}s.
		\end{split}
	\end{equation}
	We proceed estimating the third term on the right-hand side of \ref{after_Ito_formula_improved_rate}. Using \ref{Y_EM_continuous_interpolation_proof_general_case}, Young's inequality  with $\zeta \in (0,1)$,  Lemma \ref{Improved_rate_drift_growth}  and Lemma \ref{lem:distance_EM_scheme}, yields, for any $t \in [0,T-\epsilon]$
	\begin{equation*}
		\begin{split}
			&  \int_0^t \E \left[ \langle Y_s^{\text{aux}} - \widehat{Y}_s^{\text{EM}} , b(T- s , \hat{\theta} , \widehat{Y}_s^{\text{EM}} ) -  \ b(T- s , \hat{\theta}, \widehat{Y}^{\text{EM}}_{\lfloor s/ \gamma\rfloor \gamma }  ) \rangle \right]   \rmd s
			\\ & = \sum_{k=1}^M \E \Bigg[ \int_0^t \left( Y_s^{\text{aux}, (k)} - \widehat{Y}_s^{\text{EM}, (k)} \right) \Bigg( b^{(k)}(T- s , \hat{\theta} , \widehat{Y}_s^{\text{EM}} ) - b^{(k)}(T- s, \hat{\theta}, \widehat{Y}^{\text{EM}}_{\lfloor s/ \gamma\rfloor \gamma }  )
			\\ &  \qquad  \qquad \qquad \qquad \qquad  \qquad \qquad \quad  - \sum_{j=1}^M \frac{\partial b^{(k)}(T- s , \hat{\theta}, \widehat{Y}^{\text{EM}}_{\lfloor s/ \gamma\rfloor \gamma }  )}{\partial x^{j}} ( \widehat{Y}^{\text{EM} , (j)}_{s} - \widehat{Y}^{\text{EM} , (j)}_{\lfloor s/ \gamma\rfloor \gamma } ) \Bigg)  \ \rmd s \Bigg]
			\\ & \quad + \sum_{k=1}^M \E  \left[	\int_0^t  \left( Y_s^{\text{aux}, (k)} - \widehat{Y}_s^{\text{EM}, (k)} \right) \left(  \sum_{j=1}^M  \frac{\partial b^{(k)}(T- s , \hat{\theta}, \widehat{Y}^{\text{EM}}_{\lfloor s/ \gamma\rfloor \gamma }  )}{\partial x^{j}}  ( \widehat{Y}^{\text{EM} , (j)}_{s} - \widehat{Y}^{\text{EM} , (j)}_{\lfloor s/ \gamma\rfloor \gamma } ) \right)  \ \rmd s \right]
						\\ & \le \frac{\zeta}{2}	 \int_0^t \E \left[|Y_s^{\text{aux}} - \widehat{Y}_s^{\text{EM}}  |^2 \right] \ \rmd s
			\\ & \quad  + \frac{1}{2 \zeta} \sum_{k=1}^M \E \Bigg[ \int_0^t  | b^{(k)}(T- s , \hat{\theta} , \widehat{Y}_s^{\text{EM}} ) - b^{(k)}(T- s , \hat{\theta}, \widehat{Y}^{\text{EM}}_{\lfloor s/ \gamma\rfloor \gamma }  )
		\end{split}
	\end{equation*}
	\begin{equation} \label{T1_T2_T3_estimate}
		\begin{split}	
			 &  \qquad  \qquad \qquad \qquad \qquad  \qquad \qquad \qquad \  - \sum_{j=1}^M \frac{\partial b^{(k)}(T-s, \hat{\theta}, \widehat{Y}^{\text{EM}}_{\lfloor s/ \gamma\rfloor \gamma }  )}{\partial x^{j}}  \left( \widehat{Y}^{\text{EM} , (j)}_{s} - \widehat{Y}^{\text{EM} , (j)}_{\lfloor s/ \gamma\rfloor \gamma } \right)  |^2  \ \rmd s \Bigg]
			\\ & \quad + \sum_{k=1}^M \E  \Bigg[	\int_0^t  \left( Y_s^{\text{aux}, (k)} - \widehat{Y}_s^{\text{EM}, (k)} \right)
			\\ & \qquad \qquad \qquad \quad  \times \left( \sum_{j=1}^M  \frac{\partial b^{(k)}(T-s , \hat{\theta}, \widehat{Y}^{\text{EM}}_{\lfloor s/ \gamma\rfloor \gamma }  )}{\partial x^{j}}
			\int_{\lfloor s/\gamma \rfloor \gamma}^s   b^{(j)}(T- \lfloor r/\gamma \rfloor \gamma , \hat{\theta}, \widehat{Y}^{\text{EM}}_{\lfloor r/ \gamma\rfloor \gamma }  ) \ \rmd r \right)    \rmd s \Bigg]
			\\ & \quad +  \sum_{k=1}^M \E  \left[ 	\int_0^t  \left( Y_s^{\text{aux}, (k)} - \widehat{Y}_s^{\text{EM}, (k)} \right) \left( \sum_{j=1}^M \ \frac{\partial b^{(k)}(T- s , \hat{\theta}, \widehat{Y}^{\text{EM}}_{\lfloor s/ \gamma\rfloor \gamma }  )}{\partial x^{j}}     \int_{\lfloor s/\gamma \rfloor \gamma}^s  \sqrt{2} \ \rmd \bar{B}_r^{(j)} \right)   \rmd s \right]
			\\ &  \le \frac{\zeta}{2}	 \int_0^t \E\left[ |Y_s^{\text{aux}} - \widehat{Y}_s^{\text{EM}}  |^2 \right] \ \rmd s + \frac{\mathsf{K}_4^2 }{2 \zeta}      \int_0^t   (1+2 |T-s|^{\alpha})^2 \ \E \left[|\widehat{Y}^{\text{EM}}_{s} - \widehat{Y}^{\text{EM}}_{\lfloor s/ \gamma\rfloor \gamma } |^4 \right]  \ \rmd s
			\\ & \quad + \sum_{k=1}^M \E  \Bigg[	\int_0^t  \left( Y_s^{\text{aux}, (k)} - \widehat{Y}_s^{\text{EM}, (k)} \right)
			\\ & \qquad \qquad \qquad \quad \times  \left(  \sum_{j=1}^M  \frac{\partial b^{(k)}(T-s , \hat{\theta}, \widehat{Y}^{\text{EM}}_{\lfloor s/ \gamma\rfloor \gamma }  )}{\partial x^{j}}
			\int_{\lfloor s/\gamma \rfloor \gamma}^s   b^{(j)}(T- \lfloor r/\gamma \rfloor \gamma , \hat{\theta}, \widehat{Y}^{\text{EM}}_{\lfloor r/ \gamma\rfloor \gamma }  ) \ \rmd r \right) \   \rmd s \Bigg]
			\\ & \quad +  \sum_{k=1}^M \E \left[ 	\int_0^t  \left( Y_s^{\text{aux}, (k)} - \widehat{Y}_s^{\text{EM}, (k)} \right) \left(  \sum_{j=1}^M \ \frac{\partial b^{(k)}(T- s , \hat{\theta}, \widehat{Y}^{\text{EM}}_{\lfloor s/ \gamma\rfloor \gamma }  )}{\partial x^{j}}     \int_{\lfloor s/\gamma \rfloor \gamma}^s  \sqrt{2} \ \rmd \bar{B}^{(j)}_r \right)  \ \rmd s \right]
			\\		& \le \frac{\zeta}{2}	 \int_0^t \E \left[|Y_s^{\text{aux}} - \widehat{Y}_s^{\text{EM}}  |^2 \right] \ \rmd s +  \gamma^2 \frac{ \mathsf{K}_4^2}{\zeta} t(1+4 T^{2\alpha})  C_{\mathsf{EMose},4}
			\\ & \quad + \sum_{k=1}^M \E  \Bigg[	\int_0^t  \left( Y_s^{\text{aux}, (k)} - \widehat{Y}_s^{\text{EM}, (k)} \right) 
			\\ &  \qquad \qquad \qquad \quad \times \left(  \sum_{j=1}^M  \frac{\partial b^{(k)}(T-s , \hat{\theta}, \widehat{Y}^{\text{EM}}_{\lfloor s/ \gamma\rfloor \gamma }  )}{\partial x^{j}}
			\int_{\lfloor s/\gamma \rfloor \gamma}^s   b^{(j)}(T- \lfloor r/\gamma \rfloor \gamma , \hat{\theta}, \widehat{Y}^{\text{EM}}_{\lfloor r/ \gamma\rfloor \gamma }  ) \ \rmd r \right)    \rmd s \Bigg]
			\\ & \quad +  \sum_{k=1}^M \E  \left[	\int_0^t  \left( Y_s^{\text{aux}, (k)} - \widehat{Y}_s^{\text{EM}, (k)} \right) \left(  \sum_{j=1}^M \ \frac{\partial b^{(k)}(T- s , \hat{\theta}, \widehat{Y}^{\text{EM}}_{\lfloor s/ \gamma\rfloor \gamma }  )}{\partial x^{j}}     \int_{\lfloor s/\gamma \rfloor \gamma}^s  \sqrt{2} \ \rmd \bar{B}_r^{(j)} \right)  \rmd s \right].
		\end{split}
	\end{equation}
	We proceed estimating the third and fourth term on the right-hand side of \ref{T1_T2_T3_estimate}, separately.

	The third term is estimated using Young's inequality  with $\zeta \in (0,1)$, Lemma \ref{growth_estimate_gradient_theta_neural_network}, Remark  \ref{remark_growth_estimate_neural_network}, Remark \ref{Control_algorithm} and Lemma \ref{lem:EMproc2ndbdgeneral}, and for any $t \in [0,T-\epsilon]$, we obtain that
	\begin{equation*}
		\begin{split}
			& \sum_{k=1}^M \E \Bigg[ 	\int_0^t  \left( Y_s^{\text{aux}, (k)} - \widehat{Y}_s^{\text{EM}, (k)} \right) \\ &  \qquad \qquad \qquad   \times \left(  \sum_{j=1}^M  \frac{\partial b^{(k)}(T-s , \hat{\theta}, \widehat{Y}^{\text{EM}}_{\lfloor s/ \gamma\rfloor \gamma }  )}{\partial x^{j}}
			\int_{\lfloor s/\gamma \rfloor \gamma}^s b^{(j)}(T- \lfloor r/\gamma \rfloor \gamma , \hat{\theta}, \widehat{Y}^{\text{EM}}_{\lfloor r/ \gamma\rfloor \gamma }  )  \rmd r \right)   \rmd s \Bigg]
			\\ & \le \frac{\zeta}{2}     	\int_0^t \E \left[ | Y_s^{\text{aux}} - \widehat{Y}_s^{\text{EM}} |^2 \right] \ \rmd s
			\\ & \quad +  \frac{2 \gamma}{ \zeta }  M (1 +8 \mathsf{K}^2_3 (1+ 4  T^{2\alpha} ) )   \  \E \left[	\int_0^t  \int_{\lfloor s/\gamma \rfloor \gamma}^s  \sum_{j =1}^M | b^{(j)}(T- \lfloor r/\gamma \rfloor \gamma , \hat{\theta}, \widehat{Y}^{\text{EM}}_{\lfloor r/ \gamma\rfloor \gamma }  )|^2   \rmd r  \ \rmd s \right]
					\end{split}
		\end{equation*}
		\begin{equation} \label{estimate_T2}
		\begin{split}
			 & \le \frac{\zeta}{2}     	\int_0^t \E \left[ | Y_s^{\text{aux}} - \widehat{Y}_s^{\text{EM}} |^2 \right] \  \rmd s
			\\	 & \quad 	+  \frac{4 \gamma}{ \zeta}      	M (1 +8 \mathsf{K}^2_3 (1+ 4  T^{2\alpha} ) ) 	\\ &\qquad \qquad \qquad   \times  \int_0^t \int_{\lfloor s/\gamma \rfloor \gamma}^s
			\left[ (1+ 16 \mathsf{K}^2_{\text{Total}} (1+T^{2 \alpha}))   \E[ |  \widehat{Y}^{\text{EM}}_{\lfloor r/ \gamma\rfloor \gamma } |^2] + 32 \mathsf{K}^2_{\text{Total}}  (1+T^{2 \alpha})(1+ \E [|\hat{\theta}|^2]) \right]	 \rmd r  \ \rmd s
			\\	&  \le \frac{\zeta}{2}     	\int_0^t \E \left[ | Y_s^{\text{aux}} - \widehat{Y}_s^{\text{EM}} |^2 \right]   \ \rmd s
			\\ & \quad  + \frac{\gamma^2}{ \zeta}     4 M (1 +8 \mathsf{K}^2_3 (1+ 4  T^{2\alpha} ) )
			\\ & \qquad \times t \left[ (1+ 16 \mathsf{K}^2_{\text{Total}} (1+T^{2 \alpha}))  \sup_{0 \le s \le t}  \E[ |  \widehat{Y}^{\text{EM}}_{s} |^2] + 32 \mathsf{K}^2_{\text{Total}}  (1+T^{2 \alpha})(1+  2  \widetilde{\varepsilon}_{\text{AL}} + 2 | \theta^{*}|^2) \right]
			\\ &  \le \frac{\zeta}{2}     	\int_0^t \E \left[ | Y_s^{\text{aux}} - \widehat{Y}_s^{\text{EM}} |^2 \right]  \ \rmd s
			\\ & \quad + t \frac{\gamma^2}{ \zeta}    4 M(1 +8 \mathsf{K}^2_3 (1+ 4  T^{2\alpha} ) )
			\\ & \qquad \times \left[ (1+ 16 \mathsf{K}^2_{\text{Total}} (1+T^{2 \alpha})) C_{\mathsf{EM},2}(T) + 32 \mathsf{K}^2_{\text{Total}}  (1+T^{2 \alpha})(1+  2  \widetilde{\varepsilon}_{\text{AL}} + 2 | \theta^{*}|^2) \right].
		\end{split}
	\end{equation}
	We proceed with the estimate of the fourth term on the right-hand side of \ref{T1_T2_T3_estimate}. For $k=1, \dots, M$, we note
	\begin{equation} \label{split_T3}
		\begin{split}
			Y_s^{\text{aux}, (k)} - \widehat{Y}_s^{\text{EM}, (k)} & = Y_{\lfloor s/ \gamma\rfloor \gamma}^{\text{aux}, (k)} - \widehat{Y}_{\lfloor s/ \gamma\rfloor \gamma}^{\text{EM}, (k)}
			\\ & + \int_{\lfloor s/ \gamma\rfloor \gamma}^s (b^{(k)}(T-r , \hat{\theta}, Y_{r}^{\text{aux}}) - b^{(k)}(T-r , \hat{\theta}, \widehat{Y}_{r}^{\text{EM}}) \  \rmd r
			\\ & + \int_{\lfloor s/ \gamma\rfloor \gamma}^s (b^{(k)}(T-r , \hat{\theta}, \widehat{Y}_{r}^{\text{EM}}) - b^{(k)}(T-r , \hat{\theta}, \widehat{Y}_{\lfloor r/ \gamma\rfloor \gamma}^{\text{EM}})) \  \rmd r
			\\ & + \int_{\lfloor s/ \gamma\rfloor \gamma}^s 2 (s^{(k)}(T-r , \hat{\theta}, \widehat{Y}_{\lfloor r/ \gamma\rfloor \gamma}^{\text{EM}}) - s^{(k)}(T-\lfloor r/ \gamma\rfloor \gamma , \hat{\theta}, \widehat{Y}_{\lfloor r/ \gamma\rfloor \gamma}^{\text{EM}})) \  \rmd r .
		\end{split}
	\end{equation}
	By using \ref{split_T3},  Cauchy--Schwarz inequality, Young's inequality, Assumption \ref{Assumption_2} and Lemma \ref{growth_estimate_gradient_theta_neural_network}, we have
	\begin{equation*}
		\begin{split}
			&  \sum_{k=1}^M \E  \left[	\int_0^t  \left( Y_s^{\text{aux}, (k)} - \widehat{Y}_s^{\text{EM}, (k)} \right) \left(  \sum_{j=1}^M \ \frac{\partial b^{(k)}(T- s , \hat{\theta}, \widehat{Y}^{\text{EM}}_{\lfloor s/ \gamma\rfloor \gamma }  )}{\partial x^{j}}     \int_{\lfloor s/\gamma \rfloor \gamma}^s  \sqrt{2} \ \rmd \bar{B}^{(j)}_r \right)   \rmd s \right]
			\\ & =  \sum_{k=1}^M \E   \left[	\int_0^t  \left( Y_{\lfloor s/ \gamma\rfloor \gamma}^{\text{aux}, (k)} - \widehat{Y}_{\lfloor s/ \gamma\rfloor \gamma}^{\text{EM}, (k)} \right) \left(  \sum_{j=1}^M \ \frac{\partial b^{(k)}(T- s , \hat{\theta}, \widehat{Y}^{\text{EM}}_{\lfloor s/ \gamma\rfloor \gamma }  )}{\partial x^{j}}     \int_{\lfloor s/\gamma \rfloor \gamma}^s  \sqrt{2} \ \rmd \bar{B}^{(j)}_r \right)  \rmd s \right]
			\\ & \qquad +  \sum_{k=1}^M \E   \Bigg[	\int_0^t \left(   \int_{\lfloor s/ \gamma\rfloor \gamma}^s (b^{(k)}(T-r , \hat{\theta}, Y_{r}^{\text{aux}}) - b^{(k)}(T-r , \hat{\theta}, \widehat{Y}_{r}^{\text{EM}})) \  \rmd r \right)
			\\ &  \qquad \qquad \qquad \qquad \times  \left( \sum_{j=1}^M \ \frac{\partial b^{(k)}(T- s , \hat{\theta}, \widehat{Y}^{\text{EM}}_{\lfloor s/ \gamma\rfloor \gamma }  )}{\partial x^{j}}     \int_{\lfloor s/\gamma \rfloor \gamma}^s  \sqrt{2} \ \rmd \bar{B}^{(j)}_r \right)  \ \rmd s  \Bigg]
										\end{split}	
					\end{equation*}
				\begin{equation*}
			\begin{split}
			& \qquad +  \sum_{k=1}^M \E  \Bigg[		\int_0^t  \left(  \int_{\lfloor s/ \gamma\rfloor \gamma}^s (b^{(k)}(T-r , \hat{\theta}, \widehat{Y}_{r}^{\text{EM}}) - b^{(k)}(T-r , \hat{\theta}, \widehat{Y}_{\lfloor r/ \gamma\rfloor \gamma}^{\text{EM}})) \  \rmd r \right)
			\\ &  \qquad \qquad \qquad \qquad \times  \left( \sum_{j=1}^M \ \frac{\partial b^{(k)}(T- s , \hat{\theta}, \widehat{Y}^{\text{EM}}_{\lfloor s/ \gamma\rfloor \gamma }  )}{\partial x^{j}}     \int_{\lfloor s/\gamma \rfloor \gamma}^s  \sqrt{2} \ \rmd \bar{B}^{(j)}_r  \right)  \ \rmd s \Bigg]
			\\ & \qquad +  \sum_{k=1}^M \E  \Bigg[	\int_0^t  \left(  \int_{\lfloor s/ \gamma\rfloor \gamma}^s 2 (s^{(k)}(T-r , \hat{\theta}, \widehat{Y}_{\lfloor r/ \gamma\rfloor \gamma}^{\text{EM}}) - s^{(k)}(T-\lfloor r/ \gamma\rfloor \gamma , \hat{\theta}, \widehat{Y}_{\lfloor r/ \gamma\rfloor \gamma}^{\text{EM}})) \  \rmd r \right)
			\\ &  \qquad \qquad \qquad \qquad \times \left( \sum_{j=1}^M \ \frac{\partial b^{(k)}(T- s , \hat{\theta}, \widehat{Y}^{\text{EM}}_{\lfloor s/ \gamma\rfloor \gamma }  )}{\partial x^{j}}     \int_{\lfloor s/\gamma \rfloor \gamma}^s  \sqrt{2} \ \rmd \bar{B}^{(j)}_r \right)  \ \rmd s \Bigg]
			\\	& \le  \sum_{k=1}^M \E  \Bigg[	\int_0^t \int_{\lfloor s/ \gamma\rfloor \gamma}^s \gamma^{- 1/2}  |  b^{(k)}(T- r , \hat{\theta}, Y_{r}^{\text{aux}}) - b^{(k)}(T- r , \hat{\theta}, \widehat{Y}_{ r}^{\text{EM}})  |   \rmd r
			\\ & \qquad \qquad \qquad \qquad  \times \gamma^{ 1/2}   \left | \int_{\lfloor s/\gamma \rfloor \gamma}^s  \sum_{j=1}^M   \frac{\partial b^{(k)}(T- s , \hat{\theta}, \widehat{Y}^{\text{EM}}_{\lfloor s/ \gamma\rfloor \gamma }  )}{\partial x^{j}}    \sqrt{2} \ \rmd \bar{B}_r^{(j)}      \right |    \rmd s \Bigg]
			\\		& \qquad +  \sum_{k=1}^M \E  \Bigg[	\int_0^t \int_{\lfloor s/ \gamma\rfloor \gamma}^s   |  b^{(k)}(T- r , \hat{\theta}, \widehat{Y}_{ r}^{\text{EM}}) -  b^{(k)}(T- r , \hat{\theta},\widehat{Y}_{\lfloor r/ \gamma\rfloor \gamma}^{\text{EM}})   |   \rmd r
			\\ & \qquad \qquad \qquad \quad   \times   \left | \int_{\lfloor s/\gamma \rfloor \gamma}^s  \sum_{j=1}^M   \frac{\partial b^{(k)}(T- s , \hat{\theta}, \widehat{Y}^{\text{EM}}_{\lfloor s/ \gamma\rfloor \gamma }  )}{\partial x^{j}}    \sqrt{2} \ \rmd \bar{B}_r^{(j)}      \right |    \rmd s \Bigg]
			\\	 & \qquad +  \sum_{k=1}^M \E  \Bigg[	\int_0^t \int_{\lfloor s/ \gamma\rfloor \gamma}^s 2   |  s^{(k)}(T- r , \hat{\theta},\widehat{Y}_{\lfloor r/ \gamma\rfloor \gamma}^{\text{EM}}) -  s^{(k)}(T- \lfloor r/ \gamma\rfloor \gamma , \hat{\theta}, \widehat{Y}_{\lfloor r/ \gamma\rfloor \gamma}^{\text{EM}})   |    \rmd r
			\\ & \qquad \qquad \qquad \quad \times  \left | \int_{\lfloor s/\gamma \rfloor \gamma}^s  \sum_{j=1}^M   \frac{\partial b^{(k)}(T- s , \hat{\theta}, \widehat{Y}^{\text{EM}}_{\lfloor s/ \gamma\rfloor \gamma }  )}{\partial x^{j}}    \sqrt{2} \ \rmd \bar{B}_r^{(j)}      \right |    \rmd s \Bigg]
			\\ & \le  \frac{\gamma^{- 1}}{2}  \sum_{k=1}^M \E  \left[	\int_0^t \int_{\lfloor s/ \gamma\rfloor \gamma}^s   |  b^{(k)}(T- r , \hat{\theta}, Y_{r}^{\text{aux}}) - b^{(k)}(T- r , \hat{\theta}, \widehat{Y}_{ r}^{\text{EM}})  |^2   \rmd r \   \rmd s \right]
			\\ & \quad + \frac{\gamma}{2}  \sum_{k=1}^M  \int_0^t \E \left[ \left | \int_{\lfloor s/\gamma \rfloor \gamma}^s  \sum_{j=1}^M   \frac{\partial b^{(k)}(T- s , \hat{\theta}, \widehat{Y}^{\text{EM}}_{\lfloor s/ \gamma\rfloor \gamma }  )}{\partial x^{j}}    \sqrt{2} \ \rmd \bar{B}_r^{(j)}      \right |^2    \rmd s \right]
			\\ & \quad +   \sum_{k=1}^M 	\int_0^t \left[ \E  \left(   \int_{\lfloor s/ \gamma\rfloor \gamma}^s  | b^{(k)}(T- r , \hat{\theta}, \widehat{Y}_{ r}^{\text{EM}}) -  b^{(k)}(T-  r, \hat{\theta},\widehat{Y}_{\lfloor r/ \gamma\rfloor \gamma}^{\text{EM}})  |  \rmd r  \right)^2 \right]^{1/2}
			\\ 	& \qquad \qquad \quad \times  \left[ \E \left | \int_{\lfloor s/\gamma \rfloor \gamma}^s  \sum_{j=1}^M   \frac{\partial b^{(k)}(T- s , \hat{\theta}, \widehat{Y}^{\text{EM}}_{\lfloor s/ \gamma\rfloor \gamma }  )}{\partial x^{j}}    \sqrt{2} \ \rmd \bar{B}_r^{(j)}      \right |^2 \right]^{1/2}    \rmd s
			\\ & \quad +  \sum_{k=1}^M  	\int_0^t \left[ \E  \left(   \int_{\lfloor s/ \gamma\rfloor \gamma}^s  2 | s^{(k)}(T- r , \hat{\theta}, \widehat{Y}^{\text{EM}}_{\lfloor r/ \gamma\rfloor \gamma }  ) -  s^{(k)}(T-  \lfloor r/ \gamma\rfloor \gamma, \hat{\theta},\widehat{Y}_{\lfloor r/ \gamma\rfloor \gamma}^{\text{EM}})  | \ \rmd r  \right)^2 \right]^{1/2}
			\\ 	& \qquad \qquad  \quad \times  \left[ \E \left | \int_{\lfloor s/\gamma \rfloor \gamma}^s  \sum_{j=1}^M   \frac{\partial b^{(k)}(T- s , \hat{\theta}, \widehat{Y}^{\text{EM}}_{\lfloor s/ \gamma\rfloor \gamma }  )}{\partial x^{j}}    \sqrt{2} \ \rmd \bar{B}_r^{(j)}      \right |^2 \right]^{1/2}    \rmd s
										\end{split}	
					\end{equation*}
				\begin{equation}   \label{estimate_T3}
				\begin{split}
				& 	 \le   \E \left[ 	\int_0^t \int_{\lfloor s/ \gamma\rfloor \gamma}^s \gamma^{- 1}  (1+8 \mathsf{K}_3^2 (1+ 4 T^{2\alpha}))  |  Y_{r}^{\text{aux}} - \widehat{Y}_{ r}^{\text{EM}}  |^2   \rmd r \   \rmd s \right]
			\\ &  \quad + 2 \gamma^2 \sum_{k=1}^M \E \left[  \int_0^t    \sum_{j=1}^M \left |   \frac{\partial b^{(k)}(T- s , \hat{\theta}, \widehat{Y}^{\text{EM}}_{\lfloor s/ \gamma\rfloor \gamma }  )}{\partial x^{j}}      \right |^2     \rmd s \right]
			\\ & \quad + \gamma^{1/2} 2  (1+ 8 \mathsf{K}_3^2(1+ 4T^{2 \alpha}))^{1/2} 
			\\ & \qquad \times \sum_{k=1}^M  	\int_0^t \left[      \int_{\lfloor s/ \gamma\rfloor \gamma}^s \E[  |  \widehat{Y}_{ r}^{\text{EM}} -  \widehat{Y}_{\lfloor r/ \gamma\rfloor \gamma}^{\text{EM}}   |^2 ] \   \rmd r \right]^{1/2}  \left[ \E  \int_{\lfloor s/\gamma \rfloor \gamma}^s  \left|  \sum_{j=1}^M   \frac{\partial b^{(k)}(T- s , \hat{\theta}, \widehat{Y}^{\text{EM}}_{\lfloor s/ \gamma\rfloor \gamma }  )}{\partial x^{j}}    \right|^2    \rmd r       \right]^{1/2}    \rmd s	
			\\ & \quad +   \gamma^{1+ \alpha} 4  \mathsf{K}_1  (1+ 8  \widetilde{\varepsilon}_{\text{AL}} + 8 | \theta^{*}|^2)^{1/2}   \sum_{k=1}^M  	\int_0^t  \left[ \E  \int_{\lfloor s/\gamma \rfloor \gamma}^s \left |  \sum_{j=1}^M   \frac{\partial b^{(k)}(T- s , \hat{\theta}, \widehat{Y}^{\text{EM}}_{\lfloor s/ \gamma\rfloor \gamma }  )}{\partial x^{j}}   \right |^2    \rmd r       \right]^{1/2}    \rmd s	
			\\ & 	 \le    (1+8 \mathsf{K}_3^2 (1+ 4 T^{2\alpha})) 	\int_0^t   \sup_{0 \le r \le s}  \E \left[  |  Y_{r}^{\text{aux}} - \widehat{Y}_{ r}^{\text{EM}}  |^2 \right]        \rmd s
			+  4 \gamma^2 t  M (1 +8 \mathsf{K}^2_3 (1+ 4  T^{2\alpha} ) )
			\\ & \quad  +  \gamma^{3/2} 2 (1+ 8 \mathsf{K}_3^2(1+ 4T^{2 \alpha}))^{1/2} C^{1/2}_{\mathsf{EMose},2}
			\sum_{k=1}^M  	\int_0^t        \left[ \E  \int_{\lfloor s/\gamma \rfloor \gamma}^s \left |  \sum_{j=1}^M   \frac{\partial b^{(k)}(T- s , \hat{\theta}, \widehat{Y}^{\text{EM}}_{\lfloor s/ \gamma\rfloor \gamma }  )}{\partial x^{j}}   \right |^2   \rmd r       \right]^{1/2}    \rmd s	
			\\ & \quad +   \gamma^{1+ \alpha} 4  \mathsf{K}_1  (1+ 8  \widetilde{\varepsilon}_{\text{AL}} + 8 | \theta^{*}|^2)^{1/2}   \sum_{k=1}^M  	\int_0^t  \left[ \E  \int_{\lfloor s/\gamma \rfloor \gamma}^s \left |  \sum_{j=1}^M   \frac{\partial b^{(k)}(T- s , \hat{\theta}, \widehat{Y}^{\text{EM}}_{\lfloor s/ \gamma\rfloor \gamma }  )}{\partial x^{j}}   \right |^2    \rmd r       \right]^{1/2}    \rmd s	
		\\	& 	 \le     (1+8 \mathsf{K}_3^2 (1+ 4 T^{2\alpha})) 	\int_0^t   \sup_{0 \le r \le s}  \E \left[  |  Y_{r}^{\text{aux}} - \widehat{Y}_{ r}^{\text{EM}}  |^2 \right]   \   \rmd s
			+  4 \gamma^2 t  M (1 +8 \mathsf{K}^2_3 (1+ 4  T^{2\alpha} ) )
			\\ & \quad +  [\gamma^{2}  2   (1+ 8 \mathsf{K}_3^2(1+ 4T^{2 \alpha}))^{1/2} C^{1/2}_{\mathsf{EMose},2} + \gamma^{3/2+ \alpha}  4 \mathsf{K}_1       (1+ 8  \widetilde{\varepsilon}_{\text{AL}} + 8 | \theta^{*}|^2)^{1/2}]
			\\ & \qquad \times [ t M \sqrt{2} (  1+ 8  \mathsf{K}_3^2 (1+ 4  T^{2\alpha} )            )^{1/2}].
		\end{split}	
	\end{equation}	
	Using \ref{T1_T2_T3_estimate}, \ref{estimate_T2} and \ref{estimate_T3} in \ref{after_Ito_formula_improved_rate}, we have
	\begin{equation*}
		\begin{split}
			&	\E \left[|  Y_t^{\text{aux}} - \widehat{Y}_t^{\text{EM}} |^2 \right]
			\\ & \le \int_0^t 4(1+  \zeta +  \mathsf{K}_3  (1+2 T^{\alpha} + 4  \mathsf{K}_3  (1+ 4T^{2 \alpha}) ))  \sup_{0 \le r \le s}  \E\left[| Y_r^{\text{aux}} - \widehat{Y}_r^{\text{EM}}|^2 \right]  \ \text{d}s
			\\ & \quad + 8 \gamma^2 t  \zeta^{-1}   M  (1 + 8 \mathsf{K}^2_3 (1+ 4  T^{2\alpha} ) )
			\\ & \qquad  \times  \left[ (1+ 16 \mathsf{K}^2_{\text{Total}} (1+T^{2 \alpha})) C_{\mathsf{EM},2}(T) + 32 \mathsf{K}^2_{\text{Total}}  (1+T^{2 \alpha})(1+  2  \widetilde{\varepsilon}_{\text{AL}} + 2 | \theta^{*}|^2) \right]
			\\ & \quad + 2 \gamma^2 t  \mathsf{K}_4^2 \zeta^{-1} (1+4 T^{2\alpha})  C_{\mathsf{EMose},4}
			+  8 \gamma^2 t M  (1 +8 \mathsf{K}^2_3 (1+ 4  T^{2\alpha} ) )
			\\ & \quad 	+  \gamma^{2\alpha} t 4 \zeta^{-1} \mathsf{K}_1^2 (1+ 8(   \widetilde{\varepsilon}_{\text{AL}} +  | \theta^{*}|^2) )
			\\			 &  \quad + 4  [\gamma^{2}     (1+ 8 \mathsf{K}_3^2(1+ 4T^{2 \alpha}))^{1/2} C^{1/2}_{\mathsf{EMose},2} + \gamma^{3/2+ \alpha}  2 \mathsf{K}_1       (1+ 8  \widetilde{\varepsilon}_{\text{AL}} + 8 | \theta^{*}|^2)^{1/2}]
			\\ & \qquad \times  [ t M \sqrt{2} (  1+ 8  \mathsf{K}_3^2 (1+ 4  T^{2\alpha} )            )^{1/2}].
		\end{split}
	\end{equation*}
	Thus,
	\begin{equation*} 
		\begin{split}
			& \sup_{0 \le s \le t}	\E \left[|  Y_s^{\text{aux}} - \widehat{Y}_s^{\text{EM}} |^2 \right]
			\\ & \le \int_0^t 4(1+  \zeta +  \mathsf{K}_3  (1+2 T^{\alpha} + 4  \mathsf{K}_3  (1+ 4T^{2 \alpha}) ))  \sup_{0 \le r \le s}  \E \left[| Y_r^{\text{aux}} - \widehat{Y}_r^{\text{EM}}|^2 \right]  \ \text{d}s
				\end{split}
	\end{equation*}	
	\begin{equation} \label{after_Ito_formula_improved_rate_final_estimates}
	\begin{split}
			\\ & \quad + \gamma^{2\alpha} t 2 \Bigg(    \mathsf{K}_4^2 \zeta^{-1} (1+4 T^{2\alpha})  C_{\mathsf{EMose},4}
			+  4  M (1 +8 \mathsf{K}^2_3 (1+ 4  T^{2\alpha} ) )
			+    2 \zeta^{-1} \mathsf{K}_1^2 (1+ 8(   \widetilde{\varepsilon}_{\text{AL}} +  | \theta^{*}|^2) )
			\\ & \qquad \qquad  \quad + 4   \zeta^{-1}   M  (1 +8 \mathsf{K}^2_3 (1+ 4  T^{2\alpha} ) )
			\\ & \qquad \qquad  \qquad  \times [ (1+ 16 \mathsf{K}^2_{\text{Total}} (1+T^{2 \alpha})) C_{\mathsf{EM},2}(T) + 32 \mathsf{K}^2_{\text{Total}}  (1+T^{2 \alpha})(1+  2  \widetilde{\varepsilon}_{\text{AL}} + 2 | \theta^{*}|^2) ]
			\\ & \qquad \qquad  \quad   + 2  [    (1+ 8 \mathsf{K}_3^2(1+ 4T^{2 \alpha}))^{1/2} C^{1/2}_{\mathsf{EMose},2} +   2 \mathsf{K}_1       (1+ 8  \widetilde{\varepsilon}_{\text{AL}} + 8 | \theta^{*}|^2)^{1/2}]
			\\ & \qquad \qquad  \qquad  \times   [  M \sqrt{2} (  1+ 8  \mathsf{K}_3^2 (1+ 4  T^{2\alpha} )            )^{1/2}] \Bigg)
			\\ &  \le 2 e^{ 4(1+  \zeta +  \mathsf{K}_3  (1+2 T^{\alpha} + 4  \mathsf{K}_3  (1+ 4T^{2 \alpha}) ))t } \gamma^{2\alpha} t
			\\	& \times  \Bigg(    \mathsf{K}_4^2 \zeta^{-1} (1+4 T^{2\alpha})  C_{\mathsf{EMose},4}  + 4  M (1 +8 \mathsf{K}^2_3 (1+ 4  T^{2\alpha} ) ) +    2 \zeta^{-1} \mathsf{K}_1^2 (1+ 8(   \widetilde{\varepsilon}_{\text{AL}} +  | \theta^{*}|^2) )
			\\ & \qquad + 4   \zeta^{-1}  M  (1 +8 \mathsf{K}^2_3 (1+ 4  T^{2\alpha} ) )
			\\ & \qquad \quad \times [ (1+ 16 \mathsf{K}^2_{\text{Total}} (1+T^{2 \alpha})) C_{\mathsf{EM},2}(T) + 32 \mathsf{K}^2_{\text{Total}}  (1+T^{2 \alpha})(1+  2  \widetilde{\varepsilon}_{\text{AL}} + 2 | \theta^{*}|^2) ]
			\\ &  \qquad + 2  [    (1+ 8 \mathsf{K}_3^2(1+ 4T^{2 \alpha}))^{1/2} C^{1/2}_{\mathsf{EMose},2} +   2 \mathsf{K}_1       (1+ 8  \widetilde{\varepsilon}_{\text{AL}} + 8 | \theta^{*}|^2)^{1/2}]
			\\ & \qquad \quad  \times   [  M \sqrt{2} (  1+ 8  \mathsf{K}_3^2 (1+ 4  T^{2\alpha} )            )^{1/2}] \Bigg).
		\end{split}
	\end{equation}
	Using \ref{after_Ito_formula_improved_rate_final_estimates} and $t_K=T-\epsilon$, we have
	\begin{equation} \label{fourth_upper_bound_general_improved_rate_convergence} 
		\begin{split}
			& W_2(\mathcal{L}(Y_{t_K}^{\text{aux}}), \mathcal{L}(Y_{K}^{\text{EM}}))
			\\ & \le \sqrt{	\mathbb{E}\left[ |Y_{t_K}^{\text{aux}} -\widehat{Y}_{t_K}^{\text{EM}}|^2 \right]}
			\\ & \le  \sqrt{2} e^{ 2(1+  \zeta +  \mathsf{K}_3  (1+2 T^{\alpha} + 4  \mathsf{K}_3  (1+ 4T^{2 \alpha}) )) (T-\epsilon)} \gamma^{\alpha} \sqrt{T-\epsilon}
			\\ & \quad \times    \Bigg(    \mathsf{K}_4^2 \zeta^{-1} (1+4 T^{2\alpha})  C_{\mathsf{EMose},4}  +  4  M (1 +8 \mathsf{K}^2_3 (1+ 4  T^{2\alpha} ) )  +    2 \zeta^{-1} \mathsf{K}_1^2 (1+ 8(   \widetilde{\varepsilon}_{\text{AL}} +  | \theta^{*}|^2) )
			\\ & \qquad \quad +4  \zeta^{-1}  M   (1 +8 \mathsf{K}^2_3 (1+ 4  T^{2\alpha} ) )
			\\ & \qquad \qquad   \times [ (1+ 16 \mathsf{K}^2_{\text{Total}} (1+T^{2 \alpha})) C_{\mathsf{EM},2}(T) + 32 \mathsf{K}^2_{\text{Total}}  (1+T^{2 \alpha})(1+  2  \widetilde{\varepsilon}_{\text{AL}} + 2 | \theta^{*}|^2) ]
		\\	& \qquad \quad   + 2  [    (1+ 8 \mathsf{K}_3^2(1+ 4T^{2 \alpha}))^{1/2} C^{1/2}_{\mathsf{EMose},2} +   2 \mathsf{K}_1       (1+ 8  \widetilde{\varepsilon}_{\text{AL}} + 8 | \theta^{*}|^2)^{1/2}]
			\\ & \qquad \qquad  \quad  \times   [  M \sqrt{2} (  1+ 8  \mathsf{K}_3^2 (1+ 4  T^{2\alpha} )            )^{1/2}] \Bigg)^{1/2}.
		\end{split}
	\end{equation}

	\paragraph{Final upper bound on $W_2(\mathcal{L}(Y_K^{\text{EM}}),\pi_{\mathsf{D}}).$} Substituting \ref{first_upper_bound_general}, \ref{final_second_bound_Wasserstein_general_new_second}, \ref{final_third_upper_bound}, and \ref{fourth_upper_bound_general_improved_rate_convergence}  into \ref{upper_bound_wasserstein_general}, we have
	\begin{equation*} 
		\begin{split}
			&	W_2(\mathcal{L}(Y_K^{\text{EM}}),\pi_{\mathsf{D}}) \\ & \le  (\sqrt{  \mathbb{E}[|X_0  |^2]} + \sqrt{M} ) 2 \sqrt{  \epsilon}
			\\ & \quad +  \sqrt{2}  ( 	\sqrt{\mathbb{E}[|X_0|^2]}+ \sqrt{M} ) e^{    - 2 \widehat{	L}_{\text{MO}} (T-\epsilon) -\epsilon}
			\\ & \quad + \sqrt{2  \zeta^{-1}} e^{(1+  \zeta - 2 \widehat{ L}_{\text{MO}} )(T-\epsilon)}  \sqrt{   \varepsilon_{\text{SN}} }
			\\ & \quad + \sqrt{2} e^{ 2(1+  \zeta +  \mathsf{K}_3  (1+2 T^{\alpha} + 4  \mathsf{K}_3  (1+ 4T^{2 \alpha}) )) (T-\epsilon) } \gamma^{\alpha} \sqrt{T-\epsilon}
					\end{split}
		\end{equation*}
	\begin{equation} \label{final_bound_first_inequality}
\begin{split}
			 & \quad \  \times    \Bigg(    \mathsf{K}_4^2 \zeta^{-1} (1+4 T^{2\alpha})  C_{\mathsf{EMose},4}  +  4  M (1 +8 \mathsf{K}^2_3 (1+ 4  T^{2\alpha} ) )+    2 \zeta^{-1} \mathsf{K}_1^2 (1+ 8(   \widetilde{\varepsilon}_{\text{AL}} +  | \theta^{*}|^2) )
			\\ & \qquad \quad + 4    \zeta^{-1}  M   (1 +8 \mathsf{K}^2_3 (1+ 4  T^{2\alpha} ) )
			\\ &  \qquad \quad  \quad \times [ (1+ 16 \mathsf{K}^2_{\text{Total}} (1+T^{2 \alpha})) C_{\mathsf{EM},2}(T) + 32 \mathsf{K}^2_{\text{Total}}  (1+T^{2 \alpha})(1+  2  \widetilde{\varepsilon}_{\text{AL}} + 2 | \theta^{*}|^2) ]
		\\	&  \qquad \quad   + 2  [    (1+ 8 \mathsf{K}_3^2(1+ 4T^{2 \alpha}))^{1/2} C^{1/2}_{\mathsf{EMose},2} +   2 \mathsf{K}_1       (1+ 8  \widetilde{\varepsilon}_{\text{AL}} + 8 | \theta^{*}|^2)^{1/2}]
			\\ &  \qquad \qquad \times  [  M \sqrt{2} (  1+ 8  \mathsf{K}_3^2 (1+ 4  T^{2\alpha} )            )^{1/2}] \Bigg)^{1/2}.
		\end{split}
	\end{equation}
	The bound for $W_2(\mathcal{L}(\widehat{Y}_{K}^{\text{EM}}),\pi_{\mathsf{D}}) $ in \ref{final_bound_first_inequality} can be made arbitrarily small by appropriately choosing parameters including $\epsilon,T, \varepsilon_{\text{SN}}$ and $\gamma$. More precisely, for any $\delta>0$, we first choose $ 0 \le   \epsilon < \epsilon_{\delta}$ with $\epsilon_{\delta}$ given in Table \ref{tab:convconst_general} such that the first term on the right-hand side of \ref{final_bound_first_inequality} is
	\begin{equation} \label{first_delta_4_general}
		(\sqrt{  \mathbb{E}[|X_0  |^2]} + \sqrt{M} ) 2 \sqrt{  \epsilon} <\delta/4.
	\end{equation}
	Next, we choose $T > T_{\delta}$ with $T_{\delta}$ given in Table \ref{tab:convconst_general} such that the second term on the right-hand side of \ref{final_bound_first_inequality} is
	\begin{equation} \label{second_delta_4_general}
		\sqrt{2}  ( 	\sqrt{\mathbb{E}[|X_0|^2]}+ \sqrt{M} ) e^{    - 2 \widehat{	L}_{\text{MO}} (T-\epsilon) -\epsilon} <\delta/4.
	\end{equation}
	Next, we turn to the third term on the right-hand side of \ref{final_bound_first_inequality}. We choose $ 0 < \varepsilon_{\text{SN}} < \varepsilon_{\text{SN}, \delta}$ with $\varepsilon_{\text{SN}, \delta}$ given in Table \ref{tab:convconst_general} such that
	\begin{equation} \label{third_delta_4_general_first}
		\begin{split}
			\sqrt{2  \zeta^{-1}} e^{(1+  \zeta - 2 \widehat{ L}_{\text{MO}} )(T-\epsilon)}  \sqrt{   \varepsilon_{\text{SN}} } < \delta / 4.
		\end{split}
	\end{equation}
	Finally, we choose $ 0 < \gamma < \gamma_{\delta}$ with $\gamma_{\delta}$ given in Table \ref{tab:convconst_general} such that  the fourth term on the right-hand side of \ref{final_bound_first_inequality} is
	\begin{equation*} 
		\begin{split}
			& \sqrt{2} e^{ 2(1+  \zeta +  \mathsf{K}_3  (1+2 T^{\alpha} + 4  \mathsf{K}_3  (1+ 4T^{2 \alpha}) )) (T-\epsilon) } \gamma^{\alpha} \sqrt{T-\epsilon}
			\\ 	 & \quad \  \times    \Bigg(    \mathsf{K}_4^2 \zeta^{-1} (1+4 T^{2\alpha})  C_{\mathsf{EMose},4}  +  4  M (1 +8 \mathsf{K}^2_3 (1+ 4  T^{2\alpha} ) )+    2 \zeta^{-1} \mathsf{K}_1^2 (1+ 8(   \widetilde{\varepsilon}_{\text{AL}} +  | \theta^{*}|^2) )
		\end{split}
	\end{equation*}
	\begin{equation} \label{fourth_delta_4_general}
		\begin{split}
			& \qquad \quad + 4    \zeta^{-1}    M (1 +8 \mathsf{K}^2_3 (1+ 4  T^{2\alpha} ) )
			\\ &  \qquad \quad  \quad \times [ (1+ 16 \mathsf{K}^2_{\text{Total}} (1+T^{2 \alpha})) C_{\mathsf{EM},2}(T) + 32 \mathsf{K}^2_{\text{Total}}  (1+T^{2 \alpha})(1+  2  \widetilde{\varepsilon}_{\text{AL}} + 2 | \theta^{*}|^2) ]
			\\ &  \qquad \quad   + 2  [    (1+ 8 \mathsf{K}_3^2(1+ 4T^{2 \alpha}))^{1/2} C^{1/2}_{\mathsf{EMose},2} +   2 \mathsf{K}_1       (1+ 8  \widetilde{\varepsilon}_{\text{AL}} + 8 | \theta^{*}|^2)^{1/2}]
			\\ &  \qquad \qquad \times  [  M \sqrt{2} (  1+ 8  \mathsf{K}_3^2 (1+ 4  T^{2\alpha} )            )^{1/2}] \Bigg)^{1/2} <\delta/4.
		\end{split}
	\end{equation}
	Using \ref{first_delta_4_general}, \ref{second_delta_4_general}, \ref{third_delta_4_general_first} and \ref{fourth_delta_4_general},  we obtain $ W_2(\mathcal{L}(\widehat{Y}_{K}^{\text{EM}}),\pi_{\mathsf{D}}) <\delta$.
\end{proof}

\subsection{Proof of the  Preliminary Results for the General Case}
\label{proof_preliminary_estimates_appendix_general_case}

\noindent We provide the proofs of Section \ref{section_general_case} and Appendix \ref{proofs_appendix_preliminary_results_general_case_statements}.
\begin{proof}[Proof of Remark \ref{remark_growth_estimate_neural_network}]
	Using Assumption \ref{Assumption_2}, we have
	\begin{equation*}
		\begin{split}
			|s(t, \theta,x)| & \le 	| s(t, \theta,x) - s(0, 0,0) | + |  s(0, 0,0) |
			\\ & \le \mathsf{K}_1 (1+ | \theta |  )  |t|^{\alpha} + \mathsf{K}_2 (1+ | t |^{\alpha} )  | \theta |  + \mathsf{K}_3 (1+ | t |^{\alpha} ) |x |  + |  s(0, 0,0) |
			\\ & \le \mathsf{K}_{\text{Total}} (1+ | t |^{\alpha} ) (1+ | \theta | + |x| ),
		\end{split}
	\end{equation*}
	where $\mathsf{K}_{\text{Total}} :=  \mathsf{K}_1+\mathsf{K}_2+\mathsf{K}_3+  |  s(0, 0,0) |$.
\end{proof}	


\begin{proof} [Proof of Lemma \ref{lem:EMproc2ndbdgeneral}]
	Using It{\^o}'s formula, we have, for any $t\in[0,T-\epsilon]$ and $p \in [2,4]$,
	\begin{equation} \label{Ito_formula_pth_moment_Euler_scheme_general}
		\begin{split}
			\rmd |\widehat{Y}_t^{\text{EM}}|^p
			& =  p\left\langle |\widehat{Y}_t^{\text{EM}}|^{p-2} \widehat{Y}_t^{\text{EM}}, \widehat{Y}_{\lfloor t/\gamma \rfloor \gamma }^{\text{EM}} \right\rangle\rmd t  + p \left\langle |\widehat{Y}_t^{\text{EM}}|^{p-2} \widehat{Y}_t^{\text{EM}},   2 \  s(T-\lfloor t/\gamma \rfloor \gamma, \hat{\theta} , \widehat{Y}_{\lfloor t/\gamma \rfloor \gamma}^{\text{EM}})\right\rangle\rmd t\\
			& \quad +p \left\langle  |\widehat{Y}_t^{\text{EM}}|^{p-2} \widehat{Y}_t^{\text{EM}} ,  \sqrt{2 } \   \rmd \overline{B}_t \right\rangle + \frac{p}{2}  |\widehat{Y}_t^{\text{EM}}|^{p-2}( 2 M)   \ \rmd t + \frac{p(p-2)}{2}    |\widehat{Y}_t^{\text{EM}}|^{p-4} \ 2 |\widehat{Y}_t^{\text{EM}}|^{2}   \rmd t.
		\end{split}
	\end{equation}
	Integrating and taking expectation on both sides in \ref{Ito_formula_pth_moment_Euler_scheme_general}, using Young's inequality and Remark \ref{remark_growth_estimate_neural_network}, we have
	
	\begin{equation*}
		\begin{split}
			\mathbb{E} \left[ |\widehat{Y}_t^{\text{EM}}|^p \right]
			&= 		\mathbb{E} \left[ |\widehat{Y}_0^{\text{EM}}|^p \right] + p \int_0^t 	\mathbb{E} \left[   \langle |\widehat{Y}_s^{\text{EM}}|^{p-2} \widehat{Y}_s^{\text{EM}}, \widehat{Y}_{\lfloor s/\gamma \rfloor \gamma }^{\text{EM}} \rangle   \right]\rmd s
			\\ & \quad + 2p \int_0^t 	\mathbb{E} \left[  \langle |\widehat{Y}_s^{\text{EM}}|^{p-2} \widehat{Y}_s^{\text{EM}},    s(T-\lfloor s/\gamma \rfloor \gamma, \hat{\theta} , \widehat{Y}_{\lfloor s/\gamma \rfloor \gamma}^{\text{EM}}) \rangle  \right] \rmd s  
			\\ & \quad + p(M + p-2)  \int_0^t  	\mathbb{E} \left[ |\widehat{Y}_s^{\text{EM}} |^{p-2} \right]  \ \rmd s
			\\  & \le 	\mathbb{E} \left[ |\widehat{Y}_0^{\text{EM}}|^p \right] + 3(p-1) \int_0^t 	\mathbb{E} \left[    |\widehat{Y}_s^{\text{EM}}|^{p} \right] \  \rmd s  +   \int_0^t 	\mathbb{E} \left[ |  \widehat{Y}_{\lfloor s/\gamma \rfloor \gamma }^{\text{EM}} |^p   \right] \ \rmd s
			\\ & \quad + 2  \int_0^t  	\mathbb{E} \left[|s(T-\lfloor s/\gamma \rfloor \gamma, \hat{\theta} , \widehat{Y}_{\lfloor s/\gamma \rfloor \gamma}^{\text{EM}})|^p     \right] \ \rmd s 
			\\ & \quad + \frac{2}{p} (pM+ p(p-2))^{\frac{p}{2}} t  + \frac{p-2}{p} \int_0^t \mathbb{E} \left[ |\widehat{Y}_s^{\text{EM}} |^{p} \right]  \ \rmd s
			\\	   & \quad  \le 	\mathbb{E} \left[ |\widehat{Y}_0^{\text{EM}}|^p \right] +  \left(3p-2 - \frac{2}{p} \right) \int_0^t 	\mathbb{E} \left[    |\widehat{Y}_s^{\text{EM}}|^{p} \right] \  \rmd s  +   \int_0^t 	\mathbb{E} \left[ |  \widehat{Y}_{\lfloor s/\gamma \rfloor \gamma }^{\text{EM}} |^p   \right] \ \rmd s
			\\ & \quad + 2^p  \mathsf{K}^p_{\text{Total}} \int_0^t (1+ | T-\lfloor s/\gamma \rfloor \gamma |^{\alpha} )^p  \ 	\mathbb{E} \left[   | \widehat{Y}_{\lfloor s/\gamma \rfloor \gamma }^{\text{EM}} |^p \right] \  \rmd s
			\\ 	  & \quad + 2^{2p-1}\mathsf{K}^p_{\text{Total}} (1+ \mathbb{E} [  |  \hat{\theta} |^p  ]) \int_0^t (1+ | T-\lfloor s/\gamma \rfloor \gamma |^{\alpha} )^p 	    \  \rmd s + \frac{2}{p} (pM+ p(p-2))^{\frac{p}{2}} t
			\\  & \le \mathbb{E} \left[ |\widehat{Y}_0^{\text{EM}}|^p \right] + ( 3p-1 - \frac{2}{p} + 2^{2p-1} \mathsf{K}^p_{\text{Total}} (1+T^{\alpha p})) \int_0^t 	\sup_{0 \le r \le s}\mathbb{E} \left[    |\widehat{Y}_r^{\text{EM}}|^{p} \right] \  \rmd s
			\\ & \quad + 2^{3p-2}\mathsf{K}^p_{\text{Total}} (1+ \mathbb{E} [  |  \hat{\theta} |^p  ]) (1+T^{\alpha p}) t  + \frac{2}{p} (pM+ p(p-2))^{\frac{p}{2}} t .
		\end{split}
	\end{equation*}
	Using Gr{\"o}nwall's inequality, we have
	\begin{equation*}
		\begin{split}
			\sup_{0 \le s \le t}	\mathbb{E} \left[ |\widehat{Y}_s^{\text{EM}}|^p \right]
			&  \le \mathbb{E} \left[ |\widehat{Y}_0^{\text{EM}}|^p \right]   + ( 3p-1 - \frac{2}{p}  + 2^{2p-1} \mathsf{K}^p_{\text{Total}} (1+T^{\alpha p})) \int_0^t 	\sup_{0 \le r \le s}\mathbb{E} \left[    |\widehat{Y}_r^{\text{EM}}|^{p} \right] \  \rmd s
			\\ & \quad + 2^{3p-2}\mathsf{K}^p_{\text{Total}} (1+ \mathbb{E} [  |  \hat{\theta} |^p  ]) (1+T^{\alpha p}) t + \frac{2}{p} (pM+ p(p-2))^{\frac{p}{2}} t
			\\ & \le e^{t(3p-1 - \frac{2}{p}   + 2^{2p-1} \mathsf{K}^p_{\text{Total}} (1+T^{\alpha p}))}
			\\ & \quad \times   ( \mathbb{E} \left[ |\widehat{Y}_0^{\text{EM}}|^p \right] +   2^{3p-2}\mathsf{K}^p_{\text{Total}} t (1+ \mathbb{E} [  |  \hat{\theta} |^p  ]) (1+T^{\alpha p})  + \frac{2}{p} (pM+ p(p-2))^{\frac{p}{2}} t  ).
		\end{split}
	\end{equation*}		
\end{proof}

\begin{proof}  [Proof of Lemma \ref{lem:distance_EM_scheme}]
	Using \ref{continuous_time_EM_version}, Lemma \ref{lem:EMproc2ndbdgeneral} and Remark \ref{remark_growth_estimate_neural_network}, we have, for any $t\in[0,T-\epsilon]$ and $p \in [2,4]$,
	\begin{equation*}
		\begin{split}
			&\E\left[ |\widehat{Y}_t^{\text{EM}}  - \widehat{Y}_{\lfloor t/\gamma \rfloor \gamma}^{\text{EM}} |^p\right]
			\\ & \le \gamma^{p} \E\left[ \left|  \widehat{Y}_{\lfloor t/\gamma \rfloor \gamma }^{\text{EM}} + 2 \ s(T-\lfloor t/\gamma \rfloor \gamma, \hat{\theta}, \widehat{Y}_{\lfloor t/\gamma \rfloor \gamma }^{\text{EM}})  \right|^p\right] +   \E\left[ \left|   \int_{\lfloor t/\gamma \rfloor \gamma}^t \sqrt{2} \  \rmd\overline{B}_s  \right|^p\right]
		\end{split}
	\end{equation*}
	\begin{equation*}
		\begin{split}
			& \le  2^{p-1}  \gamma^{p} \left(  \E \left[   | \widehat{Y}_{\lfloor t/\gamma \rfloor \gamma }^{\text{EM}} |^p \right] + 2^p \  \E\left[   | s(T-\lfloor t/\gamma \rfloor \gamma, \hat{\theta}, \widehat{Y}_{\lfloor t/\gamma \rfloor \gamma }^{\text{EM}}) |^p  \right] \right)  +  \gamma^{\frac{p}{2}} (M p(p-1))^{\frac{p}{2}}
			\\ & \le 2^{p-1}  \gamma^{p} ( C_{\mathsf{EM},p}(t) + 2^{3p-2} \mathsf{K}^p_{\text{Total}}  (1+  T^{\alpha p} )   C_{\mathsf{EM},p}(t) + 2^{4p-3} \mathsf{K}^p_{\text{Total}}  (1+  T^{\alpha p} ) (1+ \E [|\hat{\theta}|^p])  )
			\\ & \quad   +  \gamma^{\frac{p}{2}} (M p(p-1))^{\frac{p}{2}}
			\\ & \le \gamma^{\frac{p}{2}} 	C_{\mathsf{EMose},p},
		\end{split}
	\end{equation*}
	where
	\begin{equation*}
		C_{\mathsf{EMose},p}=  2^{p-1}  (C_{\mathsf{EM},p}(T) + \mathsf{K}^p_{\text{Total}} (1+  T^{\alpha p} )  (2^{3p-2}    C_{\mathsf{EM},p}(T) + 2^{4p-3} (1+ \E [|\hat{\theta}|^p])  ))
		+   (M p(p-1))^{\frac{p}{2}} .
	\end{equation*}
\end{proof}

\begin{proof}[Proof of Lemma \ref{Improved_rate_drift_growth}]
	By the mean value theorem,  for any $k=1, \dots, M$, we have,
	\begin{equation*}
		b^{(k)}(t,\theta,x) - b^{(k)}(t,\theta,\bar{x}) = \sum_{i=1}^M \frac{\partial b^{(k)}(t,\theta,qx +(1-q)\bar{x})}{\partial y^i} (x^i - \bar{x}^{(i)}),
	\end{equation*}
	for some $q \in (0,1)$. Hence, for a fixed $q \in (0,1)$, we have
	\begin{equation*}
		\begin{split}
			&	\left| b^{(k)}(t,\theta,x) - b^{(k)}(t,\theta,\bar{x})  - \sum_{i=1}^M \frac{\partial b^{(k)}(t,\theta,\bar{x})}{\partial y^i} (x^i - \bar{x}^{i})    \right|
			\\ & = \left| \sum_{i=1}^M \frac{\partial b^{(k)}(t,\theta,qx +(1-q)\bar{x})}{\partial y^i} (x^i - \bar{x}^{(i)})    - \sum_{i=1}^M \frac{\partial b^{(k)}(t,\theta,\bar{x})}{\partial y^i} (x^i - \bar{x}^{i})    \right|
			\\ & \le \sum_{i=1}^M  \left|  \frac{\partial b^{(k)}(t,\theta,qx +(1-q)\bar{x})}{\partial y^i}    - \frac{\partial b^{(k)}(t,\theta,\bar{x})}{\partial y^i}   \right| |x^i - \bar{x}^{i}|.
		\end{split}
	\end{equation*}
	The proof is completed using Assumption \ref{Assumption_2}.
\end{proof}	

\begin{proof}[Proof of Lemma \ref{growth_estimate_gradient_theta_neural_network}]
	At $x \in \mathbb{R}^M$ and for any $v \in \mathbb{R}^M$, we have, for any $k=1, \dots M$,
	\begin{equation*}
		\langle  \nabla_{x}   s^{(k)}(t,\theta,x), v \rangle = \lim_{h \rightarrow 0 }  \frac{ s^{(k)}(t,\theta ,x +vh) -  s^{(k)}(t,\theta,x)}{h}.
	\end{equation*}
	Using Assumption \ref{Assumption_2}, we have
	\begin{equation} \label{corollary_proof_general_case_gradient_drift_b_componentwise}
		\begin{split}
			|	\langle  \nabla_{x}   s^{(k)}(t,\theta, x), v \rangle | & \le  \lim_{h \rightarrow 0 }
			\left | \frac{ s^{(k)}(t,\theta ,x +vh) -  s^{(k)}(t,\theta,x)}{h} \right |
			\\ & \le   \lim_{h \rightarrow 0 }  \frac{| D_3(t,t) |   | x +vh - x | }{|h|}
			\\ &  \le  \mathsf{K}_3 (1+ 2 | t |^{\alpha} ) |v|.
		\end{split}
	\end{equation}
	Taking $v =  \frac{\nabla_{x} s^{(k)}(t,\theta, x)}{|\nabla_{x} s^{(k)}(t,\theta, x)|}$ in \ref{corollary_proof_general_case_gradient_drift_b_componentwise}, we have
	\begin{equation} \label{bound_gradient_s_proof}
		| \nabla_{x}  s^{(k)}(t,\theta,x)| \le \mathsf{K}_3 (1+ 2 | t |^{\alpha} ).
	\end{equation}
	Using \ref{expression_b_drift_corrollary_lemma_appendix} and \ref{bound_gradient_s_proof}, we obtain
	\begin{equation*}
		\begin{split}
			| \nabla_{x}  b^{(k)}(t,\theta,x)| & \le 1 + 2	| \nabla_{x}  s^{(k)}(t,\theta,x)|
			\\ & \le 1+ 2 \mathsf{K}_3 (1+ 2 | t |^{\alpha} ).
		\end{split}
	\end{equation*}
\end{proof}

\newpage
\section{Table of Constants} \label{Table_of_constants_appendix}
Table~\ref{tab:convconst_general} displays full expressions for constants which appear in Theorem \ref{main_theorem_general} and Remark \ref{main_theorem_general_relaxed_assumption}.
\begin{table}[h] 
		\caption{Explicit expressions for the constants in  Theorem \ref{main_theorem_general} and Remark \ref{main_theorem_general_relaxed_assumption}}.
	\renewcommand{\arraystretch}{2}
	\centering

	
	\scriptsize
	\begin{tabular}{@{}llllllllllllll@{}}
		\toprule
		
		\multicolumn{1}{c}{\begin{sc} Constant \end{sc}} &
		\multicolumn{2}{c}{\begin{sc} Dependency \end{sc}} &
		\multicolumn{3}{c}{\begin{sc}  Full Expression \qquad  \qquad\qquad\qquad\qquad\qquad\qquad \qquad\qquad\qquad\qquad\qquad \  \end{sc}}   \\
		
		
		
		\toprule
		$C_1$&  $O(\sqrt{M})$ &  & $ 2   ( \sqrt{\mathbb{E}[|X_0|^2]}+\sqrt{M}) $ \\\hline
		$C_2$ &  $O(\sqrt{M})$ & &  $\sqrt{2} \left(	\sqrt{\mathbb{E}[|X_0|^2]} + \sqrt{M} \right)  $  \\\hline
		$C_3(T,\epsilon)$ & $ O(e^{(1+  \zeta - 2 \widehat{ L}_{\text{MO}} )(T-\epsilon) } ) $& &  $  \sqrt{2  \zeta^{-1}} e^{(1+  \zeta - 2 \widehat{ L}_{\text{MO}} )(T-\epsilon) } $
		\\\hline $C_{\mathsf{EM},2}(T)$ & $O (M e^{T^{2\alpha+1}}  T^{2 \alpha +1} \widetilde{\varepsilon}_{\text{AL}}) $ & &
		$
		\begin{aligned}
			&	e^{T(4 + 8 \mathsf{K}^2_{\text{Total}} (1+T^{2 \alpha }))}
			\\ & \quad \times
			( \mathbb{E} [ |\widehat{Y}_0^{\text{EM}}|^2] +   16 \mathsf{K}^2_{\text{Total}} T (1+  2  \widetilde{\varepsilon}_{\text{AL}} + 2 | \theta^{*}|^2) (1+T^{2 \alpha })  +  2M T  )
		\end{aligned}
		$
		\\\hline
		$C_{\mathsf{EM},4}(T)$ & $O (M^2 e^{T^{4\alpha+1}}  T^{4 \alpha +1}) $ & &
		$
		\begin{aligned}
			& e^{T(  \frac{21}{2}   + 128 \mathsf{K}^4_{\text{Total}} (1+T^{4 \alpha }))}
			\\ & \quad \times   ( \mathbb{E} [ |\widehat{Y}_0^{\text{EM}}|^4] +  1024 \mathsf{K}^4_{\text{Total}} T (1+ \mathbb{E} [  |  \hat{\theta} |^4  ]) (1+T^{4 \alpha })  + 8(M^2+ 4M +4) T  )
		\end{aligned}
		$
		\\\hline
		$	C_{\mathsf{EMose},2} $ & $O (M e^{T^{2\alpha+1}}  T^{4 \alpha +1} \widetilde{\varepsilon}_{\text{AL}})  $ & &
		$
		\begin{aligned}
			2 (C_{\mathsf{EM},2}(T) + \mathsf{K}^2_{\text{Total}} (1+  T^{2 \alpha } )  (16    C_{\mathsf{EM},2}(T) + 32 (1+  2  \widetilde{\varepsilon}_{\text{AL}} + 2 | \theta^{*}|^2)  )) +   2 M
		\end{aligned}
		$
		\\\hline
		$	C_{\mathsf{EMose},4} $ & $O (M^2 e^{T^{4\alpha+1}}  T^{8 \alpha +1})  $ & &
		$
		\begin{aligned}
			8 (C_{\mathsf{EM},4}(T) + \mathsf{K}^4_{\text{Total}} (1+  T^{4 \alpha } )  (1024  C_{\mathsf{EM},4}(T) + 8192 (1+ \E [|\hat{\theta}|^4])  )) + 144   M^{2}
		\end{aligned}
		$
		\\\hline
		$C_4(T,\epsilon)$ & $O(M  e^{T^{4 \alpha+1}} T^{4 \alpha +1} \widetilde{\varepsilon}^{1/4}_{\text{AL}}) $& &   	$
		\begin{aligned} & \sqrt{2} e^{ 2(1+  \zeta +  \mathsf{K}_3  (1+2 T^{\alpha} + 4  \mathsf{K}_3  (1+ 4T^{2 \alpha}) )) (T-\epsilon) }  \sqrt{T-\epsilon}
			\\ & \quad \  \times    \Bigg(    \mathsf{K}_4^2 \zeta^{-1} (1+4 T^{2\alpha})  C_{\mathsf{EMose},4}  +  4  M (1 +8 \mathsf{K}^2_3 (1+ 4  T^{2\alpha} ) ) \\ & \qquad \quad +    2 \zeta^{-1} \mathsf{K}_1^2 (1+ 8(   \widetilde{\varepsilon}_{\text{AL}} +  | \theta^{*}|^2) )
			\\ & \qquad \quad + 4    \zeta^{-1}  M   (1 +8 \mathsf{K}^2_3 (1+ 4  T^{2\alpha} ) )
			\\ &  \qquad \quad  \quad \times [ (1+ 16 \mathsf{K}^2_{\text{Total}} (1+T^{2 \alpha})) C_{\mathsf{EM},2}(T) \\ &  \qquad \qquad  \qquad + 32 \mathsf{K}^2_{\text{Total}}  (1+T^{2 \alpha})(1+  2  \widetilde{\varepsilon}_{\text{AL}} + 2 | \theta^{*}|^2) ]
			\\ &  \qquad \quad   + 2  [    (1+ 8 \mathsf{K}_3^2(1+ 4T^{2 \alpha}))^{1/2} C^{1/2}_{\mathsf{EMose},2} +   2 \mathsf{K}_1       (1+ 8  \widetilde{\varepsilon}_{\text{AL}} + 8 | \theta^{*}|^2)^{1/2}]
			\\ &  \qquad \qquad \times  [  M \sqrt{2} (  1+ 8  \mathsf{K}_3^2 (1+ 4  T^{2\alpha} )            )^{1/2}] \Bigg)^{1/2} .
		\end{aligned}
		$  \\ 	\\ \hline
		$\widetilde{C}_4(T,\epsilon)$ & $O(\sqrt{M} e^{T^{2\alpha +1}} T^{2 \alpha +1} \widetilde{\varepsilon}^{1/2}_{\text{AL}}) $ &  &
        \begin{math}
        \begin{aligned}  
        & e^{(1+ (3/2)\zeta + 2  \mathsf{K}_3 (1+2T^{\alpha})  )(T-\epsilon)}  
        \\ & \ \times (\zeta^{-1/2}     (T-\epsilon)^{1/2}  C^{1/2}_{\mathsf{EMose},2}
		   + 8^{1/2} \zeta^{-1/2}  (T-\epsilon)^{1/2}   \mathsf{K}_1  (1+8  \widetilde{\varepsilon}_{\text{AL}} + 8 | \theta^{*}|^2)^{1/2}
		 \\ & \qquad  \ + 2 \zeta^{-1/2}  \mathsf{K}_3 (1+2T^{\alpha})    (T-\epsilon)^{1/2}  C_{\mathsf{EMose},2}^{1/2}) 
         \end{aligned} 
         \end{math}
         \\ \hline	$\epsilon_{\delta}$& - & & $\delta^2/(64( \sqrt{\mathbb{E}[|X_0|^2]}+\sqrt{M}) ^2)$ 		\\ \hline
		$T_{\delta}$& - & &  $(2\widehat{L}_{\text{MO}})^{-1}\left[\ln\left(4 \sqrt{2} \left(\sqrt{\E\left[ |X_0|^2 \right]}+\sqrt{M}\right)/\delta\right) - \epsilon  \right] +\epsilon $ 	\\ \hline
		$\varepsilon_{\text{SN}, \delta}$ & - & &   $( \zeta \delta^2/32) e^{-2(1+  \zeta  - 2 \widehat{L}_{\text{MO}}) (T-\epsilon)}$
		\\ \hline
		$\gamma_{\delta}$&- & &
		$
		\begin{aligned}
			\min\Bigg\{ &  (\delta / (4 \sqrt{2}))^{1/\alpha}  (T-\epsilon)^{- 1/(2\alpha)} e^{ -(2/\alpha)(1+  \zeta +  \mathsf{K}_3  (1+2 T^{\alpha} + 4 \mathsf{K}_3  (1+ 4T^{2 \alpha}) ))(T-\epsilon)}
			\\ & \times  \Bigg(  \mathsf{K}_4^2 \zeta^{-1} (1+4 T^{2\alpha})  C_{\mathsf{EMose},4}  + 4  M (1 +8 \mathsf{K}^2_3 (1+ 4  T^{2\alpha} ) ) \\ & \qquad  +    2 \zeta^{-1} \mathsf{K}_1^2 (1+ 8(   \widetilde{\varepsilon}_{\text{AL}} +  | \theta^{*}|^2) )
			\\ & \qquad  + 4    \zeta^{-1}  M   (1 +8 \mathsf{K}^2_3 (1+ 4  T^{2\alpha} ) )
			\\ &  \qquad \quad   \times [ (1+ 16 \mathsf{K}^2_{\text{Total}} (1+T^{2 \alpha})) C_{\mathsf{EM},2}(T) \\ & \qquad \qquad  + 32 \mathsf{K}^2_{\text{Total}}  (1+T^{2 \alpha})(1+  2  \widetilde{\varepsilon}_{\text{AL}} + 2 | \theta^{*}|^2) ]
			\\ &  \qquad    + 2  [    (1+ 8 \mathsf{K}_3^2(1+ 4T^{2 \alpha}))^{1/2} C^{1/2}_{\mathsf{EMose},2} +   2 \mathsf{K}_1       (1+ 8  \widetilde{\varepsilon}_{\text{AL}} + 8 | \theta^{*}|^2)^{1/2}]
			\\ &  \qquad \qquad \times  [  M \sqrt{2} (  1+ 8  \mathsf{K}_3^2 (1+ 4  T^{2\alpha} )            )^{1/2}] \Bigg)^{-1/(2\alpha)} ,  1\Bigg\}.
		\end{aligned}
		$
		
		\\
		\bottomrule
	\end{tabular}
	\label{tab:convconst_general}
	
	
\end{table}

\end{document}

%% file: math_commands.tex
\usepackage{amsmath,amsfonts,bm}

\def\eqref#1{equation~\ref{#1}}

\def\1{\bm{1}}

\DeclareMathAlphabet{\mathsfit}{\encodingdefault}{\sfdefault}{m}{sl}
\SetMathAlphabet{\mathsfit}{bold}{\encodingdefault}{\sfdefault}{bx}{n}

\newcommand{\E}{\mathbb{E}}

\newcommand{\R}{\mathbb{R}}

